\documentclass[acmsmall]{acmart}

\usepackage{algorithmic}
\usepackage{algorithm}
\usepackage{wrapfig}
\usepackage{array}
\usepackage{xcolor}

\newtheorem*{remark}{Remark}
\newtheorem{theorem}{Theorem}[section]
\newtheorem{assumption}{Assumption}
\newtheorem{corollary}{Corollary}[theorem]

\newtheorem{notation}{Notation}

\begin{document}

\title{Quantized Adam with Error Feedback}

\author{Congliang Chen}
\email{chcoli2007@163.com}
%\email{webmaster@marysville-ohio.com}
\affiliation{%
  \institution{The Chinese University of Hong Kong, Shenzhen}
  \city{Shenzhen}
  \state{Guangdong}
  \country{China}
  \postcode{518000}
}

\author{Li Shen}
\authornote{Li Shen is the corresponding author.}
\affiliation{%
  \institution{JD Explore Academy}
  %\streetaddress{1 Th{\o}rv{\"a}ld Circle}
  \city{Beijing}
  \country{China}}
\email{mathshenli@gmail.com}

\author{Haozhi Huang}
\affiliation{%
  \institution{Tencent AI Lab}
  %\streetaddress{1 Th{\o}rv{\"a}ld Circle}
  \city{Shenzhen}
  \state{Guangdong}
  \country{China}}
\email{matthzhuang@tencent.com}

\author{Wei Liu}
\affiliation{%
  \institution{Tencent}
  %\streetaddress{1 Th{\o}rv{\"a}ld Circle}
  \city{Shenzhen}
  \state{Guangdong}
  \country{China}}
\email{wl2223@columbia.edu}

\renewcommand{\shortauthors}{Chen and Shen, et al.}

\begin{abstract}
In this paper, we present a distributed variant of adaptive stochastic gradient method for training deep neural networks in the parameter-server model. To reduce the communication cost among the workers and server, we incorporate two types of quantization schemes, i.e., gradient quantization and weight quantization, into the proposed distributed Adam. Besides, to reduce the bias introduced by quantization operations, we propose an error-feedback technique to compensate for the quantized gradient. Theoretically, in the stochastic nonconvex setting, we show that the distributed adaptive gradient method with gradient quantization and error-feedback converges to the first-order stationary point, and that the distributed adaptive gradient method with weight quantization and error-feedback converges to the point related to the quantized level under both the single-worker and multi-worker modes. At last, we apply the proposed distributed adaptive gradient methods to train deep neural networks. Experimental results demonstrate the efficacy of our methods.
\end{abstract}

%%
%% Keywords. The author\left(s\right) should pick words that accurately describe

\maketitle

\section{Introduction}

Recently, deep neural networks \cite{lecun2015deep,goodfellow2016deep} achieve high performances in many applications, such as computer vision \cite{krizhevsky2012imagenet,he2016deep}, natural language processing \cite{devlin2018bert}, speech recognition \cite{amodei2016deep}, reinforcement learning \cite{mnih2015human,silver2016mastering}, etc. However, a huge deep neural network contains millions of parameters, so its training procedure requires a large amount of training data \cite{deng2009imagenet,wu2019tencent}, which may not be stored in a single machine. In addition, due to some privacy issues \cite{kairouz2019advances,yang2019federated}, all the training data cannot be sent to a single machine but can be stored in different devices. Therefore, how to accelerate the training process by using multiple machines over large-scale data or distributed data has already been a hot topic in both industrial and academic communities \cite{kraska2013mlbase,li2014scaling,xing2015petuum,liu2017distributed}.

An efficient approach to tackle this problem is to develop distributed training algorithms for the huge neural networks \cite{dean2012large}. Most of the distributed algorithms can be summarized into two categories: one is the parameter-server \cite{smola2010architecture} model (or called centralized model) shown in Fig.~\ref{cen}, and the other is the decentralized model \cite{lian2017can} shown in Fig.~\ref{decen}.  For the centralized model in Fig.~\ref{cen}, there are one parameter server and multiple workers. In an update iteration, all workers report the update vectors to the parameter server. After gathering all the update vectors, the parameter server will update the parameters and send the parameters to all workers. While for the decentralized model in Fig.~\ref{decen}, there are $n$ nodes working simultaneously. In each update iteration, each worker computes its update vector respectively and communicates with its neighbors, and then updates its own parameters. When we use a distributed training algorithm such as distributed stochastic gradient descent \cite{li2014scaling} in either the centralized model or the decentralized model, plenty of update vectors have to be communicated among different devices. Then, a communication issue emerges for huge networks.

To accelerate the distributed training process of huge deep learning models, we propose a new distributed adaptive stochastic gradient method with gradient quantization, weight quantization, and error-feedback in the parameter server model, as shown in Fig.~\ref{cen}. In what follows, we elaborate on each component used in the proposed method:

\begin{figure}
\begin{minipage}{0.45\columnwidth}

\centerline{\includegraphics[width=\columnwidth]{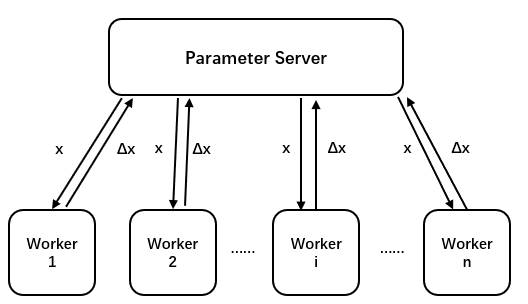}}
\vspace{-0.4cm}
\caption{The Centralized distributed model.}
\label{cen}
\end{minipage}
\begin{minipage}{0.4\columnwidth}
\centerline{\includegraphics[width=0.94\columnwidth]{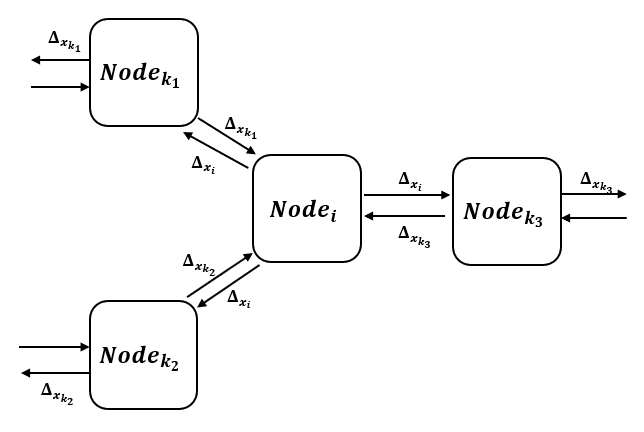}}
\vspace{-0.4cm}
\caption{The Decentralized distributed model.}
\label{decen}
\end{minipage}
\vspace{-0.6cm}
\end{figure}

{\bf Quantization}. Note that both gradient quantization and weight quantization are introduced in the proposed method to reduce the communication cost among the workers and the parameter server. Specifically, weight quantization is performed on the parameter server and the quantized weights are then broadcast to all the workers. Weight quantization is introduced because of the consideration of limited storage in edge devices. Meanwhile, gradient quantization is performed on each worker and then the quantized gradients are reported to the server. Thanks to the double quantization schemes, the communication cost can be largely reduced. In addition, for some resource-limited devices, storage is another issue. Weight quantization can also be used to reduce the deep neural network model size efficiently \cite{han2015deep,DBLP:journals/corr/ZhouNZWWZ16,rastegari2016xnor}. Especially, in federated learning, a distributed device may be smartphones or Internet of things devices, which may encounter both the storage issue and the communication issue. Thus, the weight quantization and gradient quantization schemes can jointly solve these two issues.

{\bf Adaptive learning rate.} To ease the labor of tuning learning rate, we also adopt the adaptive learning rate as \cite{kingma2014adam,hinton2012neural,duchi2011adaptive,zou2018weighted,reddi2019convergence,chen2018convergence} in the proposed method. Here, the adaptive learning rate is calculated by a similar definition to those in RMSProp \cite{hinton2012neural} and Adam \cite{kingma2014adam}, except that the noisy gradients are estimated with quantized weights. Moreover, to guarantee the convergence of the proposed method, we set the exponential moving average parameter in estimating the adaptive learning rate the same as that used in Zou et al.\cite{zou2019sufficient}.

{\bf Error-feedback.} In the proposed method, an error-feedback technique is leveraged to reduce the bias introduced by gradient quantization. The error-feedback technique is also performed on each worker. Actually, the error-feedback technique is motivated by Karimireddy et al. \cite{KarimireddyError} by introducing an additional term as the compensation term for the quantized gradient. However, due to the introduced adaptive learning rate and momentum, the compensation term is slightly different from that in Karimireddy et al. \cite{KarimireddyError}. To the best of our knowledge, this is the first work that simultaneously employs the adaptive learning rate and the error-feedback technique.

Besides, we establish the convergence rate of the proposed algorithm. In the stochastic nonconvex setting, we show that the distributed adaptive stochastic gradient method with gradient quantization and error-feedback converges to the first-order stationary point, and that the distributed adaptive stochastic gradient method with weight quantization and error-feedback converges to the point related to the quantized level under both the single-worker and multi-worker modes.
At last, we apply the proposed distributed adaptive method to train deep learning models, such as LeNet \cite{lecun1998gradient} on the MNIST dataset \cite{lecun1998gradient} and ResNet-101 \cite{he2016deep} on the CIFAR100 dataset \cite{krizhevsky2009learning}, respectively. The experimental results demonstrate the effectiveness of weight quantization, gradient quantization, and the error-feedback technique working in concert with distributed adaptive stochastic gradient method. Here, we summarize our contributions in three-fold:
\begin{itemize}
    \item We propose a distributed variant of the adaptive stochastic gradient method to train deep learning models. The proposed approach exploits gradient quantization, weight quantization, and the error-feedback technique to accelerate the training process.

    \item We establish the convergence rates of the proposed distributed adaptive stochastic gradient algorithms with weight quantization, gradient quantization, and error-feedback in the nonconvex stochastic setting under the single-worker and multi-worker environments, which are far different from the stochastic gradient setting because adaptive learning rate is introduced into the algorithms.

    \item We apply the proposed algorithms to train deep learning models including LeNet and ResNet-101. The experiments demonstrate the efficacy of the proposed algorithms.
\end{itemize}

\section{Related Works}

In this section, we enumerate several works that are most related to this work. We split the related works into two categories: distributed quantized algorithms and adaptive learning rate.

\subsection{Distributed Quantized Algorithms}

{
The quantization functions can be divided into two categories: unbiased quantization functions and biased quantization functions. For unbiased quantization functions, Wen et al. \cite{wen2017terngrad} showed that with an unbiased ternary quantization function, the distributed stochastic gradient descent algorithm can almost surely converge to a minimum point. Jiang et al. \cite{jiang2018linear} showed with an unbiased quantization function, the centralized distributed stochastic gradient descent algorithm can converge with convergence rate $\mathcal{O}(1/\sqrt{T})$. Besides, Hou et al.\cite{hou2018analysis} showed that in the stochastic convex setting, with gradient quantization solely, the algorithm they proposed will converge to the optimal solution, while with weight quantization the algorithm will converge to the point near the optimal solution which is related to the weight quantization level. However, they can only deal with the unbiased quantization function, which limits the use of both algorithms and theorems.}

{
For biased quantization functions, the main issue is to eliminate the biased error during optimization. A common technique to tackle this issue is error-feedback, where each worker stores the error of the quantization and adds the error term to the next communication before quantization.
Based on the decentralized model in Fig. \ref{decen}, Tang et al. \cite{tang2019deepsqueeze} and Koloskova et al. \cite{koloskova2019decentralized} showed that distributed stochastic gradient descent with quantized communication and error-feedback can converge to a stationary point in the nonconvex setting with convergence rate $\mathcal{O}(1/\sqrt{T})$.}
Based on the centralized model in Fig. \ref{cen}, Zhou et al. \cite{DBLP:journals/corr/ZhouNZWWZ16} and Wu et al. \cite{wu2018training} showed that few bits or integer networks can be trained empirically. {Zheng et al. \cite{zheng2019communication} showed the convergence of the algorithm with a block quantization function in the nonconvex setting.}

Among the above-mentioned algorithms,  Hou et al. \cite{hou2018analysis} is the most related work to our proposed algorithm.  However, their proposed algorithms do not adopt unbiased quantization on gradients. Moreover, they do not incorporate momentum acceleration terms into their algorithm to accelerate its piratical performance. In addition, the convergence analysis in Hou et al. \cite{hou2018analysis} is merely restricted to the stochastic convex setting, which makes their algorithm heuristic when it is applied to train deep learning models. By contrast, the convergence rates of our algorithms are established in the more difficult nonconvex setting. In this work, we first extend the error-feedback technique to adaptive stochastic gradient method (Adam) and then establish its convergence in the nonconvex setting, {and we compare the most related works in Table \ref{compare_table}.}

\begin{table}[H]
\vspace{-0.2cm}
{{
\small
    \centering
    \begin{tabular}{|c|c|c|c|c|c|}
    \hline
         Method & gradient & weight & convexity &communication & convergence \\
         &quantization&quantization&&&rate\\
         \hline
         Wen et al. \cite{wen2017terngrad}&  unbiased & no & nonconvex & centralize& $\mathcal{O}(1/\sqrt{T})$\\
         Zheng et al. \cite{zheng2019communication}&  biased & no & nonconvex& centralize&$\mathcal{O}(1/\sqrt{T})$\\
         Tang et al. \cite{tang2019deepsqueeze}&  biased & no & nonconvex& decentralize&$\mathcal{O}(1/\sqrt{T})$\\
         Koloskova et al. \cite{koloskova2019decentralized} &  biased & no & nonconvex& decentralize&$\mathcal{O}(1/\sqrt{T})$\\
         Hou et al. \cite{hou2018analysis} & unbiased & yes & convex & centralize&$\mathcal{O}(1/\sqrt{T})$\\
         Ours & biased & yes & nonconvex & centralize& $\mathcal{O}(1/\sqrt{T})$\\
         \hline
    \end{tabular}
    \caption{Comparison among different methods.}
    \label{compare_table}}}
    \vspace{-0.8cm}
\end{table}

\subsection{Adaptive Learning Rate}

Adaptive learning rate, as a popular optimization technique for training deep learning models, has attracted much attention. Numerous papers have studied the convergences of adaptive stochastic gradient methods, such as AdaGrad \cite{duchi2011adaptive,mcmahan2010adaptive}, RMSprop \cite{hinton2012neural}, Adam \cite{kingma2014adam}, and AMSGrad \cite{reddi2019convergence}. Besides the counterexample of divergence when using the Adam algorithm in the convex case in \cite{reddi2019convergence}, various works have proposed different conditions to make Adam-type methods converge to first-order stationary points. For example, \cite{ward2018adagrad,li2018convergence,zou2018weighted} establish global convergence of AdaGrad in the nonconvex setting;
Reddi et al. \cite{reddi2019convergence} check the difference between learning rates of two adjacent iterations and proposes a new variant called AMSGrad;
Chen et al. \cite{chen2018convergence} establish the convergence of AMSGrad in the nonconvex setting;
Basu et al. \cite{basu2018convergence} show that Adam converges when a full-batch gradient is used;
Zhou et al. \cite{zhou2018adashift} check the independence between gradient square and learning rate to ensure the convergence for the counterexamples in \cite{reddi2019convergence},
and Zou et al. \cite{zou2019sufficient} check the parameter setting to give a sufficient condition to guarantee the convergences of both Adam and RMSProp.
Also, Reddi et al. \cite{reddi2020adaptive} introduce distributed stochastic adaptive gradient methods in the centralized model and Nazari et al. \cite{nazari2019dadam} introduce a decentralized adaptive gradient method. In this paper, we propose a distributed variant of Adam method by incorporating quantization and error-feedback techniques. We show that the proposed method converges to a saddle point with quantized update vectors, and will be close to a saddle point when we quantize the weights of a certain network.

\section{Main Results}
Throughout this paper, we consider the following stochastic nonconvex optimization:
\begin{equation}\label{minimize}
\small
  \min_{x\in \mathbb{R}^{n}}\ f\left(x\right) = \mathbb{E}_{\xi\sim \mathbb{P}}\,\big[f\left(x,\xi\right)\big],
\end{equation}
where $\xi$ is a random variable with an unknown distribution $\small \mathbb{P}$,  and $\small f:\mathbb{R}^{n}\to \left(-\infty,+\infty\right)$ is a lower bounded nonconvex smooth function, i.e., $\small f^* = \min_{x\in R^d} f\left(x\right) >-\infty$.

Due to the absence of a probability distribution $\mathbb{P}$ of the random variable $\xi$, the access of exact gradient $\small \nabla{f}\left(x\right)$ may be impossible, which leads to constructing an unbiased noisy estimation $\small\nabla{f}\left(x_{t},\xi_{t}\right)$ for full gradient $\small\nabla{f}\left(x_{t}\right)$ at point $x_{t}$ with the given sampled sequence $\{\xi_{t}\}$. For convenience, we denote $\small g_{t}$ as the abbreviation, i.e., $\small g_{t}:=\nabla{f}\left(x_{t},\xi_{t}\right)$.  Moreover, throughout this paper, we assume that objective function $f$ is a gradient Lipschitz function. These two requirements on gradient estimation are summarized into the following assumption:
\begin{assumption}\label{ass1}
Gradient $\small\nabla{f}$ is $L$-Lipschitz continuous, i.e., $\small\|\nabla\!f\left(x\right)\!-\!\nabla\! f\left(y\right)\|\!\leq\! L\|x\!-\!y\|$. Moreover, the noisy gradient estimation $g_t$ is upper bounded and unbiased, i.e.,  $\small E[g_t] \!=\! \nabla\! f\left(x_t\right)$ and $\small \|g_t\| \!\leq\! G$.
\end{assumption}

%Below, we define a general quantization mapping $Q\left(\cdot\right):\mathbb{R}^{n}\to\mathbb{R}^{n}$ at point $x\in \mathbb{R}^{n}$ following \cite{tang2019deepsqueeze}:
%\begin{definition}
%Let $M \subset \mathbb{R}^{n}$ be a finite set. For $x \in R^n$, the quantization operation $Q\left(\cdot\right)$ on $x$ is defined as $Q\left(x\right) = \alpha\left(x\right)\beta\left(x\right) $, where $\alpha\left(x\right)\in \mathbb{R}$ and $\beta\left(x\right)\in M$.
%\end{definition}

%\begin{example}
%Let $M = \{-1,0,1\}$. We can define a quantization function that satisfies above %definition:
%\[Q\left(x\right)_i = \left\{\begin{array}{ll} \|x\|_\infty, &\text{ if } x_i > \|x\|_\infty/2\\ 0, %&\text{ if }  \|x\|_\infty/2\geq x_i \geq -\|x\|_\infty/2\\ -\|x\|_\infty, &\text{ if } %-\|x\|_\infty/2>x_i \end{array}\right.,\]
%where $\alpha\left(x\right) = \|x\|_\infty$, and \[\beta\left(x\right)_i = \left\{\begin{array}{ll} 1, &\text{ if } x_i > \|x\|_\infty/2\\ 0, &\text{ if }  \|x\|_\infty/2\geq x_i \geq -\|x\|_\infty/2\\ -1, &\text{ if } -\|x\|_\infty/2>x_i \end{array}\right..\]
%\end{example}

%Based on the above defined quantization operator, we introduce the gradient quantization $Q_g\left(\cdot\right)$ and weight quantization $Q_x\left(\cdot\right)$ operators, respectively.
Below, we introduce the gradient quantization operator $\small Q_g\left(\cdot\right)$ and weight quantization operator $\small Q_x\left(\cdot\right)$ that satisfy the following assumptions, respectively.
\begin{assumption}
\label{ass2}
{
Let $\small Q_g\left(\cdot\right)$ be the gradient quantization operator defined by Definition 1. We assume that there exists a constant $\small \delta_g \geq 0$ such that the inequality holds $\small\|g-Q_g\left(g\right)\| \leq (1-\delta_g) \|g\|$}.
\end{assumption}
\begin{assumption}
\label{ass3}
Let $\small Q_x\left(\cdot\right)$ be the weight quantization operator defined by Definition 1. First, we assume that the noisy gradient estimation at point $x_{t}$ is an unbiased estimation of $\small \nabla f\left(Q_x\left(x_t\right)\right)$, i.e., $\small E[g_t] = \nabla f\left(Q_x\left(x_t\right)\right)$. In addition, we assume that there exists  $\small \delta_x \geq 0$ such that $\small \|x-Q_x\left(x\right)\| \leq \delta_x$.
\end{assumption}

Assumption 1 is commonly used in analyzing adaptive stochastic type methods \cite{chen2018convergence,reddi2019convergence,duchi2011adaptive}. Especially, for the gradient quantization and weight quantization, the Lipschitz continuity conditions are used for bounding the error term introduced by quantization. For the weight quantization, the unbiased estimation condition $\small E[g_t] = \nabla\! f\!\left(Q_x\left(x_t\right)\right)$ is used, which has also been used in Hou et al. \cite{hou2018analysis}. All the detailed proof procedures are placed in Section \ref{proof_details}.

\subsection{Single-Machine Analysis}

\begin{wrapfigure}{r}{0.50\textwidth}
\vspace{-0.8cm}
\begin{minipage}{0.48\columnwidth}
\begin{algorithm}[H]
\small
\caption{\ Quantized Generic Adam}
\label{alg1}
\!\!{\bf Parameters:} Choose parameters $\{\alpha_t\}$, $\{\beta_t\}$, $\{\theta_t\}$,  and $x_1 \in \mathbb{R}^d$, quantization functions $Q_x\left(\cdot\right)$, and $Q_g\left(\cdot\right)$. Set initial values $m_0=0, v_0=0$, and $e_1=0$.
\vspace{-0.5cm}
\begin{algorithmic}[1]
\FOR{$t =1,2,\cdots,T$}
\STATE Sample a stochastic gradient of $f\left(Q_x\left(x_t\right)\right)$ as $g_t$;
\STATE $v_t = \theta_tv_{t-1} + \left(1-\theta_t\right)g_t^2$;
\STATE $m_t = \beta_tm_{t-1} + \left(1-\beta_t\right)g_t$;
\STATE $x_{t+1} = x_t - Q_g\left(\alpha_t\frac{m_t}{\sqrt{v_t+\epsilon}}+e_t\right)$;
\STATE $e_{t+1} = \alpha_t\frac{m_t}{\sqrt{v_t+\epsilon}}+e_t- Q_g\left(\alpha_t\frac{m_t}{\sqrt{v_t+\epsilon}}+e_t\right)$;
\ENDFOR
\end{algorithmic}
\end{algorithm}
\vspace{-0.7cm}
\end{minipage}
\end{wrapfigure}

In this subsection, we first present the quantized Generic Adam with weight quantization, gradient quantization, and the error-feedback technique working on a single machine. Then, {to show the influence on convergence related to gradient quantization or weight quantization,} we establish its convergence rate with either gradient quantization or weight quantization.

%\begin{algorithm}
%\caption{\ Quantized Generic Adam}
%\label{alg1}
%{\bf Parameters:} Choose $\{\alpha_t\},\{\beta_t\},\{\theta_t\},\text{ and } x_1 \in \mathbb{R}^d$, quantization functions $Q_x\left(\cdot\right), Q_g\left(\cdot\right)$. Set initial values $m_0 = 0$ , $v_0 = 0$, and $e_1 = 0$.
%\begin{algorithmic}[1]
%\FOR{$t =1,2,\cdots,T$}
%\STATE Sample a stochastic gradient of $f\left(Q_x\left(x_t\right)\right)$ as $g_t$;
%\STATE $v_t = \theta_tv_{t-1} + \left(1-\theta_t\right)g_t^2$;
%\STATE $m_t = \beta_tm_{t-1} + \left(1-\beta_t\right)g_t$;
%\STATE $x_{t+1} = x_t - Q_g\left(\alpha_t\frac{m_t}{\sqrt{v_t+\epsilon}}+e_t\right)$;
%\STATE $e_{t+1} = \alpha_t\frac{m_t}{\sqrt{v_t+\epsilon}}+e_t- Q_g\left(\alpha_t\frac{m_t}{\sqrt{v_t+\epsilon}}+e_t\right)$;
%\ENDFOR
%\end{algorithmic}
%\end{algorithm}

 Algorithm \ref{alg1} unifies weight quantization, gradient quantization, and the error-feedback technique into Adam, in which $\small Q_{x}\left(\cdot\right)$ denotes the weight quantization operator, $\small Q_{g}\left(\cdot\right)$ denotes the gradient quantization operator, and $e_{t}$ denotes the error-feedback term. In addition, to establish the convergence rate of Algorithm \ref{alg1} in the nonconvex setting, we make the following assumptions on momentum parameter $\beta_{t}$, exponential moving average parameter $\theta_{t}$, and base learning rate $\alpha_{t}$.
\begin{assumption}
\label{ass4}
 Assume that momentum parameter $\beta_{t}$, exponential moving average parameter $\theta_{t}$, and base learning rate $\alpha_{t}$ satisfy
 $\beta_{t}\!\in\! [0,\beta]$ with $0 \!<\! \beta \!<\! 1$, $\theta_t \!=\! 1\!-\!{\theta}/{t}$, and
 $\small \alpha_t \!=\! {\alpha}/{\sqrt{t}}$, respectively. Furthermore, we denote $\small \gamma \!=\! \beta/\theta'$ and $\small C_1 \!=\! \prod_{j=1}^N \!\frac{\theta_j}{\theta'}$ with $\small N \!=\! \max\{j|\theta_j\!<\!\theta'\}$ and $\theta'$ satisfies $\beta^2 \!<\! \theta' < 1$.
\end{assumption}

The above assumption on the hyperparameters is used to establish the convergence of adaptive stochastic type gradient method like Zou et al. \cite{zou2019sufficient}. In this paper, we use a simplified setting for momentum parameter $\beta_{t}$, exponential moving average parameter $\theta_{t}$, and the base learning rate $\alpha_{t}$ to simplify the convergence analysis, compared with the sufficient condition in Zou et al. \cite{zou2019sufficient}.

\subsubsection{Gradient Quantization}

Let $\small Q_x\left(x\right) = x$. The quantized generic Adam reduces to be {generic} Adam with gradient quantization and error-feedback. Below, we present the convergence rate of Algorithm \ref{alg1} in the single-machine mode.

\begin{theorem}
\label{T1}
 Let $\small \{x_{t}\}$ be the point generated by Algorithm \ref{alg1} with $\small Q_x\left(x\right) = x$.
 In addition, let $x_\tau^T$ represent random variable $x_\tau$ with $\tau$ taking from $\small \{1,2,\ldots,T\}$ with the same probability. If Assumptions \ref{ass1}, \ref{ass2}, \ref{ass4} further hold, the convergence result of Algorithm \ref{alg1} holds as follows:
\begin{equation}
\small
E\left[\|\nabla f\left(\left(x_\tau^T\right)\right)\|^2\right] \leq \frac{C+C'\sum_{t=1}^T\frac{1}{t}}{\sqrt{T}},
\end{equation}
where $\small C = \frac{2\sqrt{G^2+\epsilon d}}{\left(1-\beta\right)\alpha}\left(f\left(x_1\right)-f^*\right)$,
$\small C' = \frac{2\sqrt{G^2+\epsilon d}C_3}{\left(1-\beta\right)\alpha}$, and  $\small C_3 = \frac{1}{\sqrt{C_1}\left(1-\sqrt{\gamma}\right)}\left(\frac{L\left(2-\delta_g\right)G^2\alpha^2}{\epsilon\delta_g}+C_2\theta\right)$.
\end{theorem}

This theoretical result shows that with gradient quantization and error-feedback the proposed algorithm can converge to the first-order stationary point in the nonconvex setting. In addition,  the convergence rate is of the same order as the original Adam in the nonconvex setting \cite{zou2019sufficient}. Besides, paying attention to the constant $C_3$, it can be seen that the constant factor appearing in the convergence rate is related to the quantized level.

\begin{corollary}
\label{C1}
For given precision $\xi$, {when we choose hyperparameters $\small \alpha_t = 1/\sqrt{T}$ and $\theta_t = 1-1/T$,} to achieve $\small \mathbb{E}[\|\nabla f\left(x_\tau^T\right)\|^2] \leq \xi$, we have $\small T=O\left(\frac{1}{\xi^2}\right)$.
%For given iteration T, take $\alpha_t=\frac{\alpha}{\sqrt{T}}$. Then, the following estimation always holds:
%\[
%\mathbb{E}[\|\nabla f\left(x_\tau^T\right)\|^2] \leq \frac{C_3'}{\sqrt{T}} + \frac{2\sqrt{G^2+\epsilon %d}C_2\theta}{\left(1-\beta\right)\alpha\sqrt{C_1}\left(1-\sqrt{\gamma}\right)}\frac{1}{\sqrt{T}}\sum_{t=1}^T\frac{1}{t},
%\]
%where
%\[
%C_3' = \frac{2\sqrt{G^2+\epsilon d}}{\left(1-\beta\right)\alpha\sqrt{T}}\left(f\left(x_1\right) - f^* + %\frac{L\left(2-\delta_g\right)G^2\alpha^2}{\sqrt{C_1}\left(1-\sqrt{\gamma}\right) \epsilon \delta_g}\right).
%\]
\end{corollary}

\begin{remark}
This result shows Algorithm \ref{alg1} with gradient quantization and error feedback technique can convergence in the same order as some popular method such as stochastic gradient descent and vanilla Adam.
\end{remark}

\subsubsection{Weight Quantization}

In this subsection, we set $\small Q_g\left(g\right) = g$ in Algorithm \ref{alg1}. The proposed quantized generic Adam reduces to {generic} Adam with the weight quantization. In this situation, to establish the convergence rate of Algorithm \ref{alg1} are given below.

\begin{theorem}
\label{T2}
 Let $\{x_{t}\}$ be the point generated by Algorithm \ref{alg1} with $\small Q_g\left(g\right) = g$.
 In addition, let $x_\tau^T$ represent random variable $x_\tau$ with $\tau$ taking from $\small \{1,2,\ldots,T\}$ with the same probability. If Assumptions \ref{ass1}, \ref{ass3}, \ref{ass4} further hold, the convergence result of Algorithm \ref{alg1} holds as follows:
\[
\small
E\left[\|\nabla f\left(Q_x\left(x_\tau^T\right)\right)\|^2\right]\leq \frac{C_5+C_6\sum_{t=1}^T\frac{1}{t}}{\sqrt{T}} + C_7,
\]
where $\small C_5 = \frac{2\sqrt{G^2+\epsilon d}}{\left(1-\beta\right)\alpha}\left(f\left(x_1\right) - f^*\right)$,
     $\small C_6 = \frac{2\sqrt{G^2+\epsilon d}}{\left(1-\beta\right)\alpha \sqrt{C_1}\left(1-\sqrt{\gamma}\right)}\left(\frac{LG^2\alpha^2}{\epsilon}+C_2\theta\right)$,
     $\small C_7  = \frac{8\delta_x\sqrt{G^2+\epsilon d} LG}{\left(1-\beta\right)\sqrt{\epsilon C_1}\left(1-\sqrt{\gamma}\right)}$, and $\small C_2$ is defined in Theorem \ref{T1}.
\end{theorem}

This theoretical result shows that with weight quantization the algorithm will converge to the point related to the quantized level, and when we don't use quantization the proposed algorithm will converge to the first-order stationary point by setting $\delta_{x}=0$ directly. In addition, weight quantization on stochastic type gradient methods has already been considered in Khaled et al.\cite{khaled2019gradient}, in which the authors also showed that weight quantized SGD converges to a point near the global optimum. However, the analysis of weight quantized SGD in Khaled et al. \cite{khaled2019gradient} is merely restricted to the strongly convex setting.

\begin{corollary}
\label{C2}
%For given iteration T, take $\alpha_t=\frac{\alpha}{\sqrt{T}}$. Then, the following estimation always holds:
%\[
%\mathbb{E}[\|\nabla f\left(x_\tau^T\right)\|^2] \leq \frac{C_5' +c_6'\sum_{t=1}^T\frac{1}{t}}{\sqrt{T}} + C_7'
%\]
%where
%\[
%\begin{split}
%C_5'&=\frac{2\sqrt{G^2+\epsilon d}}{\left(1-\beta\right)\alpha\sqrt{T}}\left(f\left(x_1\right) - f^* + %\frac{LG^2\alpha^2}{\sqrt{C_1}\left(1-\sqrt{\gamma}\right)\epsilon}\right)\\
%C_6'&=\frac{2\sqrt{G^2+\epsilon d}C_2\theta}{\left(1-\beta\right)\alpha\sqrt{C_1}\left(1-\sqrt{\gamma}\right)\sqrt{T}}\\
%C_7'&=\frac{4\delta_x\sqrt{G^2+\epsilon d} LG}{\left(1-\beta\right)\sqrt{\epsilon C_1}\left(1-\sqrt{\gamma}\right)}.
%\end{split}
%\]
For given precision $\xi$, {when we choose hyperparameters $\small \alpha_t = 1/\sqrt{T}$ and $\theta_t = 1-1/T$,}  to achieve $\small \mathbb{E}[\|\nabla f\left(Q_x\left(x_\tau^T\right)\right)\|^2] \leq C_7'+\xi$, we have $T=O\left(\frac{1}{\xi^2}\right)$, where $\small C_7' = \frac{4\delta_x\sqrt{G^2+\epsilon d} LG}{\left(1-\beta\right)\sqrt{\epsilon C_1}\left(1-\sqrt{\gamma}\right)}$.
\end{corollary}

\begin{remark}
This result shows Algorithm \ref{alg1} with weight quantization can convergence to the point near the stationary point due to quantization, but the speed to near stationary is in the same order as some popular method such as stochastic gradient descent and vallina Adam.
\end{remark}

\subsection{Multi-Worker Analysis}

\begin{wrapfigure}{r}{0.50\textwidth}
\vspace{-0.4cm}
\begin{minipage}{0.48\columnwidth}
\small
\begin{algorithm}[H]
\caption{\ The Parameter Server}
\label{alg21}
{\bf Parameters:} Choose $x_1\in R^d$, and quantization function $Q_x\left( \cdot \right).$
\begin{algorithmic}[1]
\FOR{$t =1,2,\cdots,T$}
\STATE Broadcasting $Q_x\left(x_t\right)$;
\STATE Gathering all updates from workers $\hat{\delta}_t = \frac{1}{N}\sum_{i = 1}^N \delta_t^{\left(i\right)}$;
\STATE $x_{t+1} = x_{t}+\hat{\delta}_t$;
\ENDFOR
\STATE {\bf Output} $Q_x\left(x_t\right)$.
\end{algorithmic}
\end{algorithm}
\vspace{-0.6cm}
%\end{minipage}
%\end{wrapfigure}
%
%\begin{wrapfigure}{r}{0.50\textwidth}
%\begin{minipage}{0.48\columnwidth}
\begin{algorithm}[H]
\small
\caption{\ The $i$-th Worker}
\label{alg31}
{\bf Parameters:} Choose parameters $\{\alpha_t\},\{\beta\},\{\theta_t\}$, and quantization function $ Q_g\left(\cdot\right)$. Set initial values $m^{\left(i\right)}_0 = 0$ , $v^{\left(i\right)}_0 = 0$, $e^{\left(i\right)}_1 = 0$, and $\hat{x}_0 = 0$.
\vspace{-0.5cm}
\begin{algorithmic}[1]
\FOR{$t =1,2,\cdots,T$}
\STATE Receiving $\hat{x}_t$ from the server;
\STATE Sample a stochastic gradient of $f\left(\hat{x}_t\right)$ as $g^{\left(i\right)}_t$;
\STATE $v^{\left(i\right)}_t = \theta_t v^{\left(i\right)}_{t-1} + \left(1-\theta_t\right)\big[g^{\left(i\right)}_t\big]^2$;
\STATE $m^{\left(i\right)}_t = \beta m^{\left(i\right)}_{t-1} + \left(1-\beta\right)g^{\left(i\right)}_t$;
\STATE Sending $\delta_t^{\left(i\right)} = Q_g\left(\alpha_t\frac{m^{\left(i\right)}_t}{\sqrt{v^{\left(i\right)}_t+\epsilon}}+e^{\left(i\right)}_t\right)$;
\STATE $e^{\left(i\right)}_{t+1} = \alpha_t\frac{m^{\left(i\right)}_t}{\sqrt{v^{\left(i\right)}_t+\epsilon}}+e^{\left(i\right)}_t- \delta^{\left(i\right)}_t$;
\ENDFOR
\end{algorithmic}
\end{algorithm}
\end{minipage}
\end{wrapfigure}

In this subsection, we extend Algorithm \ref{alg1} to the multi-worker setting via the parameter server model. Below, we use Algorithm \ref{alg21} to represent the iteration schemes of the distributed quantized generic Adam algorithm in the parameter server, and Algorithm \ref{alg31} to represent the iteration schemes in all workers, respectively. Here, we assume that all the workers work independently.

Note that communicated information $\hat{x}_{t}$ and $\delta_{t}^{i}$ between the server and works is all quantized in order to improve the communication efficiency. The weight quantization procedure is performed on the server, while the gradient quantization and error-feedback procedures are performed on the workers. Below, we establish the convergence rates of distributed Adam with weight quantization, gradient quantization, and error-feedback in Algorithms \ref{alg21}-\ref{alg31} in the parameter server model.

\begin{theorem}
\label{T3}
 Let $\{x_{t}\}$ be the point generated by Algorithms \ref{alg21}-\ref{alg31}.
 In addition, let $\small \hat{x}^T_\tau$ be the random variable $\hat{x}_\tau$ with $\tau$ taking from $\small \{1,2,\ldots,T\}$ with the same probability. If Assumptions 1-4 hold and the iterates $\small \|x_{t}\| \le D$ are upper bounded, the convergence result of  Algorithms \ref{alg21}-\ref{alg31} holds as follows:
\[
\small
E\left[\|\nabla f\left(\hat{x}_\tau^T\right)\|^2\right]
\leq \frac{C_8+C_9\sum_{t=1}^T\frac{1}{t}}{\sqrt{T}} + C_{10},
\]
where $\small C_8=\frac{2\sqrt{G^2+\epsilon d}}{\left(1-\beta\right)\alpha}\left(f\left(x_1\right)-f^*\right)$, $
   \small  C_9=\frac{2\sqrt{G^2+\epsilon d}}{\left(1-\beta\right)\alpha\sqrt{C_1}\left(1-\sqrt{\gamma}\right)}\left(\frac{L\left(2-\delta_g\right)G^2\alpha^2}{\epsilon\delta_g}+C_2\theta\right)$, $
   \small  C_{10}=\frac{4\sqrt{G^2+\epsilon d}\delta_xLG}{\sqrt{C_1}\left(1-\sqrt{\gamma}\right)\sqrt{\epsilon}\left(1-\beta\right)}$,  and $\small C_2$ is defined in Theorem \ref{T1}.
\end{theorem}

In Algorithms \ref{alg21}-\ref{alg31}, both the gradient quantization and weight quantization schemes are applied. We also show that the proposed algorithms converge to a point near the saddle point of problem \eqref{minimize} up to a constant. It is noted that the constant is affected by both the gradient quantized level $\delta_{g}$ and the weight quantized level $\delta_{x}$. In addition, the limit point of the generate iterates will be influenced merely by the weight quantized level. Once gradient quantization and weight quantization reduce to identity mappings, Algorithms \ref{alg21}-\ref{alg31} reduce to the distributed Adam in the parameter server model and Theorem 3 provides their convergence rates.

\begin{corollary}
\label{C3}
%For given iteration T, take $\alpha_t=\frac{\alpha}{\sqrt{T}}$. Then, the following estimation always holds:
%\[
%\mathbb{E}[\|\nabla f\left(x_\tau^T\right)\|^2] \leq \frac{C_8' +c_9'\sum_{t=1}^T\frac{1}{t}}{\sqrt{T}} + C_{10}'
%\]
%where  $C_8'$, $C_9'$, and $C_{10}'$ are defined as follows
%\[
%\begin{split}
%    C_8'&=\frac{2\sqrt{G^2+\epsilon d}}{\left(1-\beta\right)\alpha}\left(f\left(x_1\right)-f^*  + %\frac{L\left(2-\delta_g\right)G^2\alpha^2}{\epsilon\delta_g\sqrt{C_1}\left(1-\sqrt{\gamma}\right)}\right)\\
%    C_9'&=\frac{2C_2\theta\sqrt{G^2+\epsilon d}}{\left(1-\beta\right)\alpha\sqrt{C_1}\left(1-\sqrt{\gamma}\right)}\\
%    C_{10}'&=\frac{2\sqrt{G^2+\epsilon d}\delta_xLG}{\sqrt{C_1}\left(1-\sqrt{\gamma}\right)\sqrt{\epsilon}\left(1-\beta\right)},
%\end{split}
%\]
%respectively.
For given precision $\xi$, {when we choose hyperparameters $\small \alpha_t = 1/\sqrt{T}$ and $\theta_t = 1-1/T$,} to achieve $\small \mathbb{E}[\|\nabla f\left(\hat{x}_\tau^T\right)\|^2] \leq C_{10}'+\xi$, we have $\small T=O\left(\frac{1}{\xi^2}\right)$, where $\small C_{10}' = \frac{4\sqrt{G^2+\epsilon d}\delta_xLG}{\sqrt{C_1}\left(1-\sqrt{\gamma}\right)\sqrt{\epsilon}\left(1-\beta\right)}$.
\end{corollary}

{
To close this section, we give several comments on the proposed Algorithms \ref{alg1}-\ref{alg31}. Different from distributed Adam where each worker transmits gradient to the parameter server and the parameter server calculates learning rate and update vector, we calculate the learning rates and update vector in local. Therefore, the error feedback technique can be applied to the adaptive algorithm. However, the proof will be complicated due to $N$ different learning rates being involved in the algorithm, and the following section will give a detailed proof of the above theorems.

}

\section{Proof Details}
\label{proof_details}
In this section, we provide the detailed proof procedures of Theorem \ref{T1}, Theorem \ref{T2}, Theorem \ref{T3} and the related corollaries of the main theorems.

\subsection{Proof of Theorem \ref{T1}}

Before providing the detailed proof of Theorem \ref{T1}, we first denote several useful notations. Then, we provide several lemmas that are used to split the main proof of Theorem \ref{T1} for better readability.

\begin{notation}\label{notation-thm1}
Denote $\small \sigma_t^2 = \mathbb{E}_t[g_t^2]$ where $\small \mathbb{E}_t[\cdot]$ is the conditional expectation on the random variables {$\small \{x_{t}, v_{t-1}, m_{t-1}\}$}.
Denote $\small \hat{v}_t = \theta_tv_{t-1} + \left(1-\theta_t\right)\sigma_t^2$,
$\small \hat{\eta}_t = \alpha_t/\sqrt{\hat{v}_t}$,
$\small \Delta_t = -\alpha_tm_t/\sqrt{v_t {+ \epsilon}}$,
$\small M_t = \mathbb{E}[\langle\nabla f\left(x_t\right), \Delta_t\rangle+L\left(2-\delta_g\right)\|\Delta_t\|^2 + L\left(2-\delta_g\right)\|e_t\|\|\Delta_t\|]$ and $\small \|x\|_{\hat{\eta}_t}^2 = \sum_{i=1}^d \hat{\eta}_t^{\left(i\right)}{x^{\left(i\right)}}^2$.
In addition, let $\small \tilde{x}_t = x_t - e_{t}$. Then, it holds that  $\small \tilde{x}_{t+1} = x_t - Q_g\left(-\Delta_t+e_t\right) - e_{t+1} = x_t + \Delta_t -e_t = \tilde{x}_t + \Delta_t$.
\end{notation}

{
\begin{lemma}
\label{arithmetic inequality}
Given two positive sequences $\{a_i\}_{i=1}^n$ and $\{b_i\}_{i=1}^n$, it holds that
\[ \small
\left(\sum_{i=1}^n {a_ib_i}\right)^2 \leq \left(\sum_{i=1}^n a_i^2\right)\left(\sum_{i=1}^n b_i^2\right).
\]
\end{lemma}
\begin{proof}
\[ \small
\begin{split}
\left(\sum_{i=1}^n {a_ib_i}\right)^2
&= \sum_{i=1}^n\sum_{j=1}^n a_ib_ia_jb_j\leq \sum_{i=1}^n\sum_{j=1}^n \frac{1}{2}(a_i^2b_j^2 + a_j^2b_i^2)
= \frac{1}{2}\sum_{i=1}^n a_i^2 \left(\sum_{j=1}^nb_j^2\right) + \frac{1}{2} \sum_{i=1}^n b_i^2 \left(\sum_{j=1}^n a_j^2\right)\\
& = \left(\sum_{i=1}^n a_i^2\right)\left(\sum_{i=1}^n b_i^2\right),
\end{split}
\]
where the second inequality is the arithmetic inequality with positive numbers $a_ib_j$ and $a_jb_i$.
\end{proof}
}

\begin{lemma}
By using Notation \ref{notation-thm1} and the iteration scheme of Theorem \ref{T1}, for all $t\ge 1$ the following inequality holds:
\[\small
  m_t^2\leq \frac{1}{C_1\left(1-\gamma\right)\left(1-\theta_t\right)}v_t.
\]
\end{lemma}
\begin{proof}
By using the definition of $m_t$ in Theorem \ref{T1}, it directly holds that
$\small m_t = \sum_{i=1}^t\beta^{t-i}\left(1-\beta\right)g_i$.
Let $\small \Theta\left(t,i\right) = \prod_{j=i+1}^t \theta_j$ for $i < t$, and $\small \Theta\left(i,i\right) = 1$. According to the definition of $v_t$, it holds that
$\small v_t = \sum_{i = 1}^t \left(\prod_{j=i+1}^t\theta_j\right)\left(1-\theta_i\right)g_i^2 = \sum_{i = 1}^t\Theta\left(t,i\right)\left(1-\theta_i\right)g_i^2$.

With {Lemma \ref{arithmetic inequality}}, it holds that
\[\small
\begin{split}
    m_t^2 &= \left(\sum_{i=1}^t\frac{\beta^{t-i}\left(1-\beta\right)}{\sqrt{\Theta\left(t,i\right)\left(1-\theta_i\right)}}\sqrt{\Theta\left(t,i\right)\left(1-\theta_i\right)}g_i \right)^2
    \leq \sum_{i=1}^t\frac{\beta^{2\left(t-i\right)}\left(1-\beta\right)^2}{\Theta\left(t,i\right)\left(1-\theta_i\right)}\sum_{i=1}^t\Theta\left(t,i\right)\left(1-\theta_i\right)g_i^2\\
    &= \sum_{i=1}^t\frac{\beta^{2\left(t-i\right)}\left(1-\beta\right)^2}{\Theta\left(t,i\right)\left(1-\theta_i\right)}v_t
    %\leq \sum_{i=1}^t\frac{\beta^{2\left(t-i\right)}}{\left(\prod_{j=i+1}^N\theta_j\right){\theta'}^{t-N}\left(1-\theta_i\right)}v_t
    \leq \sum_{i=1}^t\frac{\beta^{2\left(t-i\right)}}{\left(\prod_{j=i+1}^N\theta_j/\theta'\right)\left(\theta'\right)^{t-i}\left(1-\theta_i\right)}v_t\\
    &\leq \sum_{i=1}^t\frac{\beta^{2\left(t-i\right)}}{\left(\prod_{j=1}^N\theta_j/\theta'\right)\left(\theta'\right)^{t-i}\left(1-\theta_i\right)}v_t
    %= \sum_{i=1}^t\frac{\beta^{2\left(t-i\right)}}{C_1\left(\theta'\right)^{t-i}\left(1-\theta_i\right)}v_t
    \leq \frac{1}{C_1\left(1-\theta_t\right)}\sum_{i=1}^t \left(\frac{\beta^2}{\theta'}\right)^{t-i}v_t \leq \frac{1}{C_1\left(1-\gamma\right)\left(1-\theta_t\right)}v_t.
\end{split}
\]
Then, we obtain the targeted result.
\end{proof}

\begin{lemma}\label{ft}
  Let $\tau$ be randomly chosen from $\small \{1,2,\cdots,T\}$ with equal probabilities $\small p_\tau = \frac{1}{T}$. We have the following estimate:
  \[\small
  \mathbb{E}[\|\nabla f\left(x_\tau\right)\|^2]\leq \frac{\sqrt{G^2+\epsilon d}}{\alpha\sqrt{T}}\mathbb{E}\left[\sum_{t=1}^T\|\nabla f\left(\left(x_t\right)\right)\|_{\hat{\eta_t}}^2\right].
  \]

\end{lemma}

\begin{proof}
Note that $\small \|\hat{v}_t\|_1 = \theta_t \|v_{t-1}\|_1 +\left(1-\theta_t\right) \|\sigma_t\|^2$ and $\small \|g_t\|\leq G$. It is straightforward to prove $\small \|v_t\|_1 \leq G^2$. Hence, we have $\small \|\hat{v}_t + \epsilon\|_1 \leq G^2 + \epsilon d$.

Utilizing this inequality, we have
\[\small
\begin{split}
     \|\nabla f\left(x_t\right)\|^2 &= \frac{\|\nabla f\left(x_t\right)\|^2}{\sqrt{\|\hat{v}_t+\epsilon\|_1}}\sqrt{\|\hat{v}_t+\epsilon\|_1} = \sqrt{\|\hat{v}_t +\epsilon \|_1}\sum_{k=1}^d\frac{|\nabla_k f\left(x_t\right)|^2}{\sqrt{\sum_{l=1}^d \hat{v}_{t,l}+\epsilon}}\\
     &\leq \sqrt{\|\hat{v}_t+\epsilon\|_1}\alpha_t^{-1}\sum_{k=1}^d \frac{\alpha_t}{\sqrt{\hat{v}_{t,k}+\epsilon}}|\nabla_k f\left(x_t\right)|^2
     = \sqrt{\|\hat{v}_t+\epsilon\|_1}\alpha_t^{-1}\|\nabla f\left(x_t\right)\|_{\hat{\eta}_t}^2 \\
     &\leq \sqrt{G^2+\epsilon d}\alpha_t^{-1}\|\nabla f\left(x_t\right)\|_{\hat{\eta}_t}^2
       \leq \frac{\sqrt{G^2 + \epsilon d}}{\alpha_T}\|\nabla f\left(x_t\right)\|^2_{\hat{\eta}_t}.
\end{split}
\]
Then, by using the definition of $x_{\tau}$, we obtain
\[\small
     \mathbb{E}\left[\|\nabla f\left(x_\tau\right)\|^2\right] = \frac{1}{T}\sum_{t=1}^T\mathbb{E}\left[\|\nabla f\left(x_t\right)\|^2\right] \leq \frac{\sqrt{G^2 + \epsilon d}}{\alpha\sqrt{T}} \mathbb{E}\left[\sum_{t=1}^T \|\nabla f\left(x_t\right)\|^2_{\hat{\eta}_t}\right].
\]
Thus, the desired result is obtained.
\end{proof}

\begin{lemma}\label{sumdelta}
By using Notation \ref{notation-thm1}, the following inequality holds:
\[ \small
\sum_{t=1}^T\|\Delta_t\|^2 \leq \frac{G^2}{\epsilon} \sum_{t=1}^T \frac{\alpha^2}{t}.
\]
\end{lemma}
\begin{proof}
By using the definition of $m_{t}$, it holds $\small \|m_t\|^2\leq G^2$.

Then,  $\small \|\Delta_t\|^2 = \|\frac{\alpha_tm_t}{\sqrt{v_t+\epsilon}}\|^2\leq \frac{G^2}{\epsilon}\alpha_t^2$ by using the definition of $\small \Delta_{t}$.

Therefore, $\small \sum_{t=1}^T\|\Delta_t\|{^2} \leq \frac{G^2}{\epsilon}\sum_{t=1}^T\frac{\alpha^2}{t}$.
\end{proof}

\begin{lemma}\label{matrix}
By the iteration scheme of Algorithm 1, it holds that
\[
\small
\sum_{t = 1}^T \|e_t\|\|\Delta_t\| \leq \sum_{t=1}^T\frac{1-\delta_g}{\delta_g} \|\Delta_t\|^2.
\]
\end{lemma}
\begin{proof}
By the definition of noisy term $e_{t}$ and $\small \Delta_{t}$, it holds
\[\small
\begin{split}
  \|e_t\|  &= \|\Delta_t + e_{t-1} - Q_g\left(\Delta_t + e_{t-1}\right) \| \leq \left(1-\delta_g\right) \|\Delta_t + e _{t-1}\|\leq \left(1-\delta_g\right) \|\Delta_t\| + \left(1-\delta_g\right)\|e_{t-1}\|\\
  &\leq \sum_{i=1}^t \left(1-\delta_g\right)^{t-i+1}\|\Delta_i\|.
\end{split}
\]
Therefore, we have
\begin{align*}
\small
   \sum_{t = 1}^T \|e_t\|\|\Delta_t\| &\leq \sum_{t=1}^T\sum_{i=1}^t \left(1-\delta_g\right)^{t-i+1}\|\Delta_i\|\|\Delta_t\| \leq \frac{1}{2}\sum_{t=1}^T\sum_{i=1}^t \left(1-\delta_g\right)^{t-i+1}\left(\|\Delta_i\|^2+\|\Delta_t\|^2\right)\\
    %&\leq \frac{1}{2} \sum_{t=1}^T\|\Delta_t\|^2 \sum_{i=1}^t \left(1-\delta_g\right)^{t-i+1} + \frac{1}{2}\sum_{i=1}^T\|\Delta_i\|^2 \sum_{t=i}^T \left(1-\delta_g\right)^{t-i+1}\\
    &\leq \frac{1-\delta_g}{2\delta_g}\sum_{t=1}^T\|\Delta_t\|^2 + \frac{1-\delta_g}{2\delta_g}\sum_{i=1}^T\|\Delta_i\|^2 = \frac{1-\delta_g}{\delta_g}\sum_{t=1}^T\|\Delta_t\|^2.
\end{align*}
Hence, we obtain the desired result.
\end{proof}

\begin{lemma}\label{Mt}
By the definition of $\small M_{k}$, it holds that
\[
\small
    \sum_{t=1}^T M_t \leq \frac{1}{\sqrt{C_1}\left(1-\sqrt{\gamma}\right)}\left(\frac{L\left(2-\delta_g\right)G^2\alpha^2}{\epsilon\delta_g}+C_2\theta\right)\sum_{t=1}^T\frac{1}{t} - \frac{1-\beta}{2} \mathbb{E}\left[\|\sum_{t=1}^T\nabla f\left(x_t\right)\|^2_{\hat{\eta}_t}\right],
\]
where
\[\small
\begin{split}
    C_2 &= \frac{5\alpha G^3\left(1-\beta\right)}{2\epsilon\sqrt{\theta}}\left(\frac{\beta}{\left(1-\beta\right)\sqrt{\theta_1C_1\left(1-\gamma\right)}}+1\right)^2 + \frac{5\alpha G^3}{2\epsilon\sqrt{\theta}}+\frac{5\beta^2 \alpha d\sqrt{\epsilon}}{2\sqrt{\theta}\left(1-\beta\right)\theta_1C_1\left(1-\gamma\right)}\\&+\frac{5\alpha\sqrt{G^2+\epsilon}G^2\beta^2}{2\left(1-\beta\right)\sqrt{\theta}\theta_1C_1\left(1-\gamma\right)\epsilon}+\frac{5\alpha \sqrt{G^2+\epsilon} \beta^2 d}{2\left(1-\beta\right)\sqrt{\theta}\theta_1C_1\left(1-\gamma\right)}.
\end{split}
\]

\end{lemma}
\begin{proof}
To split $\small M_t$, first we introduce the following two equalities. Using the definitions of $v_{t}$ and $\hat{v}_{t}$, we obtain
\[\small
   \begin{split}
       \frac{\left(1-\beta\right)\alpha_tg_t}{\sqrt{v_t+\epsilon}} &= \frac{\left(1-\beta\right)\alpha_tg_t}{\sqrt{\hat{v}_t+\epsilon}} + \left(1-\beta\right)\alpha_tg_t\left(\frac{1}{\sqrt{v_t+\epsilon}} - \frac{1}{\sqrt{\hat{v}_t+\epsilon}}\right)\\
       & = \left(1-\beta\right)\hat{\eta}_tg_t + \left(1-\beta\right)\alpha_tg_t\frac{\left(1-\theta_t\right)\left(\sigma_t^2-g_t^2\right)}{\sqrt{v_t+\epsilon}\sqrt{\hat{v}_t+\epsilon}\left(\sqrt{v_t+\epsilon} + \sqrt{\hat{v}_t+\epsilon}\right)}\\
       & = \left(1-\beta\right)\hat{\eta}_tg_t + \hat{\eta}_t\sigma_t\frac{\left(1-\theta_t\right)g_t}{\sqrt{v_t+\epsilon}}\frac{\left(1-\beta\right)\sigma_t}{\sqrt{v_t+\epsilon}+\sqrt{\hat{v}_t+\epsilon}} - \hat{\eta}_tg_t\frac{\left(1-\theta_t\right)g_t}{\sqrt{v_t+\epsilon}}\frac{\left(1-\beta\right)g_t}{\sqrt{v_t+\epsilon}+\sqrt{\hat{v}_t+\epsilon}}.
   \end{split}
\]
In addition, it is not hard to check that the following equality holds:
\[\small
\begin{split}
& \beta\alpha_tm_{t-1}\left(\frac{1}{\sqrt{\theta_tv_{t-1}+\theta_t\epsilon}} - \frac{1}{{\sqrt{v_t+\epsilon}}}\right) = \beta\alpha_tm_{t-1}\frac{\left(1-\theta_t\right)\left(g_t^2 +\epsilon\right)}{\sqrt{v_t+\epsilon}\sqrt{\theta_tv_{t-1}+\theta_t\epsilon}\left(\sqrt{v_t+\epsilon} + \sqrt{\theta_tv_{t-1}+\theta_t\epsilon}\right)}\\
%& = \beta\alpha_tm_{t-1}\frac{\left(1-\theta_t\right)\left(g_t^2+\epsilon\right)}{\sqrt{v_t+\epsilon}\sqrt{\hat{v}_t+\epsilon}\left(\sqrt{v_t+\epsilon} + {\sqrt{\theta_tv_{t-1}+\theta_t\epsilon}}\right)}\\
%&\qquad + \beta\alpha_tm_{t-1}\frac{\left(1-\theta_t\right)\left(g_t^2+\epsilon\right)}{\sqrt{v_t+\epsilon}\left(\sqrt{v_t+\epsilon}+\sqrt{\theta_tv_{t-1}+\theta_t\epsilon}\right)}\left(\frac{1}{\sqrt{\theta_tv_{t-1}+\theta_t\epsilon}} - \frac{1}{\sqrt{\hat{v}_t+\epsilon}}\right)\\
%& = \hat{\eta}_tg_t\frac{\left(1-\theta_t\right)g_t}{\sqrt{v_t+\epsilon}}\frac{\beta m_{t-1}}{\sqrt{v_t+\epsilon}+\sqrt{\theta_tv_{t-1}+\theta_t\epsilon}}+\hat{\eta}_t \epsilon\frac{\left(1-\theta_t\right)}{\sqrt{v_t+\epsilon}}\frac{\beta m_{t-1}}{\sqrt{v_t+\epsilon}+\sqrt{\theta_tv_{t-1}+\theta_t\epsilon}}\\
%&\qquad +\frac{\beta\alpha_tm_{t-1}\left(1-\theta_t\right)^2\left(g_t^2+\epsilon\right)\left(\sigma_t^2+\epsilon\right)}{\sqrt{v_t+\epsilon}\sqrt{\hat{v}_t+\epsilon}\sqrt{\theta_tv_{t-1}+\theta_t\epsilon}\sqrt{v_t+\epsilon}+\sqrt{\theta_tv_{t-1}+\theta_t\epsilon}\right)\sqrt{\hat{v}_t+\epsilon} + \sqrt{\theta_tv_{t-1}+\theta_t\epsilon}\right)}\\
& = \hat{\eta}_tg_t\frac{\left(1-\theta_t\right)g_t}{\sqrt{v_t+\epsilon}}\frac{\beta m_{t-1}}{\sqrt{v_t+\epsilon}+\sqrt{\theta_tv_{t-1}+\theta_t\epsilon}}+\hat{\eta}_t \epsilon\frac{\left(1-\theta_t\right)}{\sqrt{v_t+\epsilon}}\frac{\beta m_{t-1}}{\sqrt{v_t+\epsilon}+\sqrt{\theta_tv_{t-1}+\theta_t\epsilon}}\\
&\qquad + \hat{\eta}_t\sqrt{\sigma_t^2+\epsilon}\frac{\left(1-\theta_t\right)g_t}{\sqrt{v_t+\epsilon}}\frac{\beta m_{t-1}}{\sqrt{\theta_tv_{t-1}+\theta_t\epsilon}}\frac{\sqrt{1-\theta_t}g_t}{\sqrt{v_t+\epsilon}+\sqrt{\theta_tv_{t-1}+\theta_t\epsilon}}\frac{\sqrt{1-\theta_t}\sqrt{\sigma_t^2+\epsilon}}{\sqrt{\hat{v}_t+\epsilon} +\sqrt{\theta_tv_{t-1}+\theta_t\epsilon}}\\
&\qquad +\hat{\eta}_t\sqrt{\sigma_t^2+\epsilon}\frac{\left(1-\theta_t\right)\sqrt{\epsilon}}{\sqrt{v_t+\epsilon}}\frac{\beta m_{t-1}}{\sqrt{\theta_tv_{t-1}+\theta_t\epsilon}}\frac{\sqrt{1-\theta_t}\sqrt{\epsilon}}{\sqrt{v_t+\epsilon}+\sqrt{\theta_tv_{t-1}+\theta_t\epsilon}}\frac{\sqrt{1-\theta_t}\sqrt{\sigma_t^2+\epsilon}}{\sqrt{\hat{v}_t+\epsilon} +\sqrt{\theta_tv_{t-1}+\theta_t\epsilon}}.
\end{split}
\]
For convenience, we denote
\begin{align*}
\small
A^1_t &= \hat{\eta}_tg_t\frac{\left(1-\theta_t\right)g_t}{\sqrt{v_t+\epsilon}}\left(\frac{\beta m_{t-1}}{\sqrt{v_t+\epsilon}+\sqrt{\theta_tv_{t-1}+\theta_t\epsilon}} + \frac{\left(1-\beta\right)g_t}{\sqrt{v_t+\epsilon}+\sqrt{\hat{v}_t}+\epsilon}\right),\\
A^2_t &= - \hat{\eta}_t\sigma_t\frac{\left(1-\theta_t\right)g_t}{\sqrt{v_t+\epsilon}}\frac{\left(1-\beta\right)\sigma_t}{\sqrt{v_t+\epsilon}+\sqrt{\hat{v}_t+\epsilon}},\ A^3_t = \hat{\eta}_t \epsilon\frac{\left(1-\theta_t\right)}{\sqrt{v_t+\epsilon}}\frac{\beta m_{t-1}}{\sqrt{v_t+\epsilon}+\sqrt{\theta_tv_{t-1}+\theta_t\epsilon}},\\
A^4_t &= \hat{\eta}_t\sqrt{\sigma_t^2+\epsilon}\frac{\left(1-\theta_t\right)g_t}{\sqrt{v_t+\epsilon}}\frac{\beta m_{t-1}}{\sqrt{\theta_tv_{t-1}+\theta_t\epsilon}}\frac{\sqrt{1-\theta_t}g_t}{\sqrt{v_t+\epsilon}+\sqrt{\theta_tv_{t-1}+\theta_t\epsilon}}\frac{\sqrt{1-\theta_t}\sqrt{\sigma_t^2+\epsilon}}{\sqrt{\hat{v}_t+\epsilon} +\sqrt{\theta_tv_{t-1}+\theta_t\epsilon}},\\
A^5_t &= \hat{\eta}_t\sqrt{\sigma_t^2+\epsilon}\frac{\left(1-\theta_t\right)\sqrt{\epsilon}}{\sqrt{v_t+\epsilon}}\frac{\beta m_{t-1}}{\sqrt{\theta_tv_{t-1}+\theta_t\epsilon}}\frac{\sqrt{1-\theta_t}\sqrt{\epsilon}}{\sqrt{v_t+\epsilon}+\sqrt{\theta_tv_{t-1}+\theta_t\epsilon}}\frac{\sqrt{1-\theta_t}\sqrt{\sigma_t^2+\epsilon}}{\sqrt{\hat{v}_t+\epsilon} +\sqrt{\theta_tv_{t-1}+\theta_t\epsilon}}.
\end{align*}
Then, we obtain
\[\small
\begin{split}
    \Delta_t - \frac{\beta\alpha_t}{\sqrt{\theta_t}\alpha_{t-1}}\Delta_{t-1}
    & = -\frac{\alpha_tm_t}{\sqrt{v_t+\epsilon}} + \frac{\beta\alpha_tm_{t-1}}{\sqrt{\theta_t\left(v_{t-1}+\epsilon\right)}} \\
  &  = -\frac{\left(1-\beta\right)\alpha_tg_t}{\sqrt{v_t+\epsilon}} + \beta\alpha_tm_{t-1}\left(\frac{1}{\sqrt{\theta_tv_{t-1}+\theta_t\epsilon}} - \frac{1}{\sqrt{v_t+\epsilon}}\right) \\
    &=  -\left(1-\beta\right)\hat{\eta}_tg_t +A^1_t + A^2_t +A^3_t + A^4_t +A^5_t.
\end{split}
\]
By using the above inequalities, we derive the upper bound for the term $\small \mathbb{E}\langle\nabla f\left(x_t\right),\Delta_t\rangle$:
\begin{equation}\label{cross-gradient-eq}
\small
\begin{split}
\mathbb{E}\langle\nabla f\left(x_t\right),\Delta_t\rangle &= \frac{\beta\alpha_t}{\sqrt{\theta_t}\alpha_{t-1}}\mathbb{E}\langle\nabla f\left(x_t\right),\Delta_{t-1}\rangle + \mathbb{E}\langle\nabla f\left(x_t\right),\Delta_t - \frac{\beta\alpha_t}{\sqrt{\theta_t}\alpha_t}\Delta_{t-1}\rangle\\
&= \frac{\beta\alpha_t}{\sqrt{\theta_t}\alpha_{t-1}}\left(\mathbb{E}\langle\nabla f\left(x_{t-1}\right),\Delta_{t-1}\rangle + \mathbb{E}\langle\nabla f\left(x_t\right) - \nabla f\left(x_{t-1}\right),\Delta_{t-1}\rangle\right)\\
&\qquad + \mathbb{E}\langle\nabla f\left(x_t\right),-\left(1-\beta\right)\hat{\eta}_tg_t\rangle + \mathbb{E}\langle\nabla f\left(x_t\right),A^1_t\rangle +
\mathbb{E}\langle\nabla f\left(x_t\right),A^2_t\rangle\\
&\qquad +\mathbb{E}\langle\nabla f\left(x_t\right),A^3_t\rangle+\mathbb{E}\langle\nabla f\left(x_t\right),A^4_t\rangle+\mathbb{E}\langle\nabla f\left(x_t\right),A^5_t\rangle.
\end{split}
\end{equation}
For the first term in Eq~\eqref{cross-gradient-eq}, we have
\[\small
\begin{aligned}
&\qquad \frac{\beta\alpha_t}{\sqrt{\theta_t}\alpha_{t-1}}\left(\mathbb{E}\langle\nabla f\left(x_{t-1}\right),\Delta_{t-1}\rangle + \mathbb{E}\langle\nabla f\left(x_t\right) - \nabla f\left(x_{t-1}\right),\Delta_{t-1}\rangle\right)\\
&\leq \frac{\beta\alpha_t}{\sqrt{\theta_t}\alpha_{t-1}}\left(\mathbb{E}\langle\nabla f\left(x_{t-1}\right),\Delta_{t-1}\rangle + \mathbb{E}[L\|x_t - x_{t-1}\|\|\Delta_{t-1}\|]\right)\\
& = \frac{\beta\alpha_t}{\sqrt{\theta_t}\alpha_{t-1}}\left(\mathbb{E}\langle\nabla f\left(x_{t-1}\right),\Delta_{t-1}\rangle + L\mathbb{E}[\|\Delta_{t-1}\|\|Q_g\left(\Delta_{t-1}+e_{t-1}\right)\|]\right)\\
&\leq \frac{\beta\alpha_t}{\sqrt{\theta_t}\alpha_{t-1}}\!\!\left(\mathbb{E}\langle\nabla f\!\left(x_{t-1}\right),\Delta_{t-1}\rangle \!\!+\!\! L\mathbb{E}[\|\Delta_{t-1}\|\left(\|Q_g\left(\Delta_{t-1}\!+\!e_{t-1}\right) \!-\! \left(\Delta_{t-1}\!+\!e_{t-1}\right)\!\!\| \!+\!\|\Delta_{t-1}\!+\!e_{t-1}\|\right)]\right)\\
& { \leq \frac{\beta\alpha_t}{\sqrt{\theta_t}\alpha_{t-1}}\left(\mathbb{E}\langle\nabla f\left(x_{t-1}\right),\Delta_{t-1}\rangle + L\mathbb{E}[\|\Delta_{t-1}\|\left((1-\delta_g)\|\Delta_{t-1}+e_{t-1}\| +\|\Delta_{t-1}+e_{t-1}\|\right)]\right) } \\
& { \leq \frac{\beta\alpha_t}{\sqrt{\theta_t}\alpha_{t-1}}\left(\mathbb{E}\langle\nabla f\left(x_{t-1}\right),\Delta_{t-1}\rangle + L\mathbb{E}[\|\Delta_{t-1}\|\left((2-\delta_g)\|\Delta_{t-1}\|+(2-\delta_g)\|e_{t-1}\|\right)]\right)} \\
& \leq \frac{\beta\alpha_t}{\sqrt{\theta_t}\alpha_{t-1}}\left(\mathbb{E}\langle\nabla f\left(x_{t-1}\right),\Delta_{t-1}\rangle + L\left(2-\delta_g\right)\mathbb{E}[\|\Delta_{t-1}\|^2]+ L\left(2-\delta_g\right)\mathbb{E}[ \|e_{t-1}\| \|\Delta_{t-1}\|]\right)\\
& = \frac{\beta\alpha_t}{\sqrt{\theta_t}\alpha_{t-1}}M_{t-1}.
\end{aligned}
\]
For the second term in Eq.~\eqref{cross-gradient-eq}, we have
\[\small
\begin{aligned}
    -(1-\beta)\mathbb{E}\langle\nabla f(x_t),\hat{\eta}_tg_t\rangle
    &= {-(1-\beta)\mathbb{E}\big[\mathbb{E}_t\langle\nabla f\left(x_t\right),\hat{\eta}_tg_t\rangle\big]
    = -\left(1-\beta\right)\mathbb{E}\langle\nabla{f}\left(x_t\right),\hat{\eta}_t\mathbb{E}_t[g_t]\rangle}\\
    &= -\left(1-\beta\right)\mathbb{E}\|\nabla f\left(x_t\right)\|_{\hat{\eta}_t}^2,
\end{aligned}
\]
{where the second equality is because given $\small \{x_t, v_{t-1}, m_{t-1}\}$, $\small \nabla f(x_t)$ and $\small \hat{\eta}_t$ are independent to $g_t$. Besides, $\small \mathbb{E}_t[\nabla f(x_t)] = \nabla f(x_t)$ and $\mathbb{E}_t [\hat{\eta}_t] = \hat{\eta}_t$.}

For the third term in Eq.~\eqref{cross-gradient-eq}, we have
\[\small
\begin{split}
    \mathbb{E} \langle \nabla f\left(x_t\right), A^1_t\rangle
    &\leq \mathbb{E}\left\langle \frac{\sqrt{\hat{\eta}_t}|\nabla f\left(x_t\right)||g_t|}{\sigma_t}, \frac{\sqrt{\hat{\eta}_t}\sigma_t\left(1-\theta_t\right)|g_t|\left|\frac{\beta m_{t-1}}{\sqrt{v_t+\epsilon}+\sqrt{\theta_tv_{t-1}+\theta_t\epsilon}} + \frac{\left(1-\beta\right)g_t}{\sqrt{v_t+\epsilon}}\right|}{\sqrt{v_t+\epsilon}} \right\rangle\\
    &\leq \mathbb{E}\left\langle \frac{\sqrt{\hat{\eta}_t}|\nabla f\left(x_t\right)||g_t|}{\sigma_t}, \sqrt{\frac{\alpha G}{\sqrt{\theta}}} \frac{\left(1-\theta_t\right)|g_t|\left(\left|\frac{\beta m_{t-1}}{\sqrt{v_t+\epsilon}+\sqrt{\theta_tv_{t-1}+\theta_t\epsilon}}\right| + \left|\frac{\left(1-\beta\right)g_t}{\sqrt{v_t+\epsilon}}\right|\right)}{\sqrt{v_t+\epsilon}} \right\rangle\\
    &\leq \mathbb{E}\left\langle \frac{\sqrt{\hat{\eta}_t}|\nabla f\left(x_t\right)||g_t|}{\sigma_t}, \sqrt{\frac{\alpha G}{\sqrt{\theta}}}\left(\frac{\beta}{\sqrt{\theta_1C_1\left(1-\gamma\right)}}+ \left(1-\beta\right)\right)\frac{\sqrt{1-\theta_t}|g_t|}{\sqrt{v_t+\epsilon}}\right\rangle\\
    %&\leq \frac{1-\beta}{10}\mathbb{E}\left \|\frac{\sqrt{\hat{\eta}_t}|\nabla f\left(x_t\right)|\mathbb{E}_t|g_t|}{\sigma_t}\right\|^2 + \frac{2\alpha G}{5\sqrt{\theta}}\left(\frac{\beta}{\left(1-\beta\right)\sqrt{\theta_1C_1\left(1-\gamma\right)}}+1\right)^2\mathbb{E}\left\|\frac{\sqrt{1-\theta_t}g_t}{\sqrt{v_t+\epsilon}}\right\|^2\\
    &\leq \frac{1-\beta}{10}\mathbb{E}\left\|\nabla f\left(x_t\right)\right\|_{\hat{\eta}_t}^2 + \frac{5\alpha G^3\left(1-\beta\right)\left(1-\theta_t\right)}{2\epsilon\sqrt{\theta}}\left(\frac{\beta}{\left(1-\beta\right)\sqrt{\theta_1C_1\left(1-\gamma\right)}}+1\right)^2,
\end{split}
\]
where the equalities hold according to the following inequities, respectively,
\begin{gather*}
\small
\sqrt{\hat{\eta}_t}\sigma_t = \sqrt{\frac{\alpha_t\sigma_t^2}{\sqrt{\hat{v}_t+\epsilon}}}
%\leq \sqrt{\frac{\alpha_t\sigma_t^2}{\sqrt{\left(1-\theta_t\right)\sigma_t^2}}}
%\leq \sqrt{\frac{G\alpha_t}{\sqrt{1-\theta_t}}}
\leq \sqrt{\frac{\alpha G}{\sqrt{\theta}}},\\ \left|\frac{\left(1-\beta\right)g_t}{\sqrt{v_t+\epsilon}}\right|\leq
\left|\frac{\left(1-\beta\right)g_t}{\sqrt{\left(1-\theta_t\right)g_t^2}}\right|=\frac{1-\beta}{\sqrt{1-\theta_t}},\\
\left|\frac{\beta m_{t-1}}{\sqrt{v_t+\epsilon}+\sqrt{\theta_tv_{t-1}+\theta_t\epsilon}}\right|
\leq \left|\frac{\beta m_{t-1}}{\sqrt{\theta_tv_{t-1}}}\right|
%\leq \frac{\beta}{\sqrt{\theta_t\left(1-\theta_t\right)C_1\left(1-\gamma\right)}}
\leq \frac{\beta}{\sqrt{\theta_1\left(1-\theta_t\right)C_1\left(1-\gamma\right)}}.\\
\end{gather*}
For the fourth term in Eq.~\eqref{cross-gradient-eq}, we have
\[\small
\begin{split}
&\mathbb{E}\langle \nabla f\left(x_t\right), A^2_t\rangle
\leq \mathbb{E}\left\langle |\nabla f\left(x_t\right)|, \left|\hat{\eta}_t\sigma_t\frac{\left(1-\theta_t\right)g_t}{\sqrt{v_t+\epsilon}}\frac{\left(1-\beta\right)\sigma_t}{\sqrt{v_t+\epsilon}+\sqrt{\hat{v}_t+\epsilon}}\right|\right\rangle \\
    &\leq \mathbb{E}\left\langle \sqrt{\hat{\eta}_t}|\nabla f\left(x_t\right)|, \sqrt{\hat{\eta}_t}\sigma_t\left|\frac{\left(1-\theta_t\right)g_t}{\sqrt{v_t+\epsilon}}\right| \left|\frac{\left(1-\beta\right)\sigma_t}{\sqrt{v_t+\epsilon}+\sqrt{\hat{v}_t+\epsilon}}\right| \right\rangle \\
    &\leq \frac{1-\beta}{10}\mathbb{E}\|\nabla f\left(x_t\right)\|_{\hat{\eta}_t}^2 + \frac{2\alpha G}{5\theta}\mathbb{E}\left\|\frac{\sqrt{1-\theta_t}g_t}{\sqrt{v_t+\epsilon}}\right\|^2 \leq \frac{1-\beta}{10}\mathbb{E}\|\nabla f\left(x_t\right)\|_{\hat{\eta}_t}^2 + \frac{5\alpha G^3\left(1-\theta_t\right)}{2\sqrt{\theta}\epsilon},\\
\end{split}
\]
where the equality holds according to
$\small
\left|\frac{\left(1-\beta\right)\sigma_t}{\sqrt{v_t+\epsilon}+\sqrt{\hat{v}_t+\epsilon}}\right| \leq \left|\frac{\left(1-\beta\right)\sigma_t}{\sqrt{\hat{v}_t+\epsilon}}\right|
%\leq \left|\frac{\left(1-\beta\right)\sigma_t}{\sqrt{\left(1-\theta_t\right)\sigma_t^2}}\right|
\leq \frac{1-\beta}{\sqrt{1-\theta_t}}.$

For the fifth term in Eq.~\eqref{cross-gradient-eq}, we have
\[\small
\begin{split}
    &\mathbb{E}\langle \nabla f\left(x_t\right),A^3_t\rangle
    \leq \mathbb{E}\left\langle |\nabla f\left(x_t\right)|, \left| \hat{\eta}_t \epsilon\frac{\left(1-\theta_t\right)}{\sqrt{v_t+\epsilon}}\frac{\beta m_{t-1}}{\sqrt{v_t+\epsilon}+\sqrt{\theta_tv_{t-1}+\theta_t\epsilon}}\right|\right\rangle\\
    &\leq \mathbb{E} \left\langle \sqrt{\hat{\eta}_t} |\nabla f\left(x_t\right)|, \sqrt{\hat{\eta}_t}\epsilon\frac{\left(1-\theta_t\right)\beta |m_{t-1}|}{\sqrt{\epsilon}\sqrt{\theta_tv_{t-1}}}\right\rangle
    \leq \frac{1-\beta}{10}\mathbb{E}\|\nabla f\left(x_t\right)\|_{\hat{\eta}_t}^2 + \frac{5\left(1-\theta_t\right)\beta^2 \alpha d\sqrt{\epsilon}}{2\sqrt{\theta}\left(1-\beta\right)\theta_1C_1\left(1-\gamma\right)},
\end{split}
\]
where the inequality holds according to
$\small
\sqrt{\hat{\eta}_t}\epsilon \leq \sqrt{\frac{\alpha_t\epsilon^2}{\sqrt{\hat{v}_t+\epsilon}}}
%\leq \sqrt{\frac{\alpha_t\epsilon^2}{\sqrt{\left(1-\theta_t\right)\epsilon}}}
\leq \sqrt{\frac{\alpha\epsilon^{3/2}}{\sqrt{\theta}}}.$

For the sixth term in Eq.~\eqref{cross-gradient-eq}, we have
\[\small
\begin{split}
    &\mathbb{E}\langle \nabla f\left(x_t\right),A^4_t\rangle\\
    & \leq \mathbb{E}\left\langle |\nabla f\left(x_t\right)|, \left|\hat{\eta}_t\sqrt{\sigma_t^2+\epsilon}\frac{\left(1\!-\!\theta_t\right)g_t}{\sqrt{v_t\!+\!\epsilon}}\frac{\beta m_{t-1}}{\sqrt{\theta_tv_{t-1}\!+\!\theta_t\epsilon}}\frac{\sqrt{1-\theta_t}g_t}{\sqrt{v_t\!+\!\epsilon}+\sqrt{\theta_tv_{t-1}\!+\!\theta_t\epsilon}}\frac{\sqrt{1-\theta_t}\sqrt{\sigma_t^2+\epsilon}}{\sqrt{\hat{v}_t\!+\!\epsilon} +\sqrt{\theta_tv_{t-1}\!+\!\theta_t\epsilon}}\right|\right\rangle \\
   & \leq \frac{1-\beta}{10}\mathbb{E}\|\nabla f\left(x_t\right)\|_{\hat{\eta}_t}^2 + \frac{5\alpha\sqrt{G^2+\epsilon}G^2\beta^2\left(1-\theta_t\right)}{2\left(1-\beta\right)\sqrt{\theta}\theta_1C_1\left(1-\gamma\right)\epsilon},
\end{split}
\]
where the equalities hold according to
$\small
\sqrt{\hat{\eta}_t}\sqrt{\sigma_t^2+\epsilon} = \sqrt{\frac{\alpha_t \left(\sigma_t^2 +\epsilon\right)}{\sqrt{\hat{v}_t+\epsilon}}}
%\leq \sqrt{\frac{\alpha_t \left(\sigma_t^2 +\epsilon\right)}{\sqrt{\left(1-\theta_t\right)\left(\sigma_t^2+\epsilon\right)}}}
\leq \sqrt{\frac{\alpha\sqrt{G^2+\epsilon}}{\sqrt{\theta}}},
$ and
\[\small
\begin{split}
&\frac{\beta m_{t-1}}{\sqrt{\theta_tv_{t-1}+\theta_t\epsilon}}\frac{\sqrt{1-\theta_t}g_t}{\sqrt{v_t+\epsilon}+\sqrt{\theta_tv_{t-1}+\theta_t\epsilon}}\frac{\sqrt{1-\theta_t}\sqrt{\sigma_t^2+\epsilon}}{\sqrt{\hat{v}_t+\epsilon} +\sqrt{\theta_tv_{t-1}+\theta_t\epsilon}} \\
&\leq \frac{\beta m_{t-1}}{\sqrt{\theta_tv_{t-1}+\theta_t\epsilon}}\frac{\sqrt{1-\theta_t}g_t}{\sqrt{v_t+\epsilon}}\frac{\sqrt{1-\theta_t}\sqrt{\sigma_t^2+\epsilon}}{\sqrt{\hat{v}_t+\epsilon}} \leq \frac{\beta}{\sqrt{\theta_1\left(1-\theta_t\right)C_1\left(1-\gamma\right)}}.
\end{split}
\]
For the seventh term in Eq.~\eqref{cross-gradient-eq}, we have
\[\small
\begin{split}
&\mathbb{E} \langle \nabla\!f\left(x_t\right), A^5_t\rangle\\
&\leq \mathbb{E}\left\langle |\nabla\!f\left(x_t\right)|, \left|\hat{\eta}_t\sqrt{\sigma_t^2\!+\!\epsilon}\frac{\left(1\!+\!\theta_t\right)\sqrt{\epsilon}}{\sqrt{v_t\!+\!\epsilon}}\frac{\beta m_{t-1}}{\sqrt{\theta_tv_{t-1}\!+\!\theta_t\epsilon}}\frac{\sqrt{1-\theta_t}\sqrt{\epsilon}}{\sqrt{v_t\!+\!\epsilon}+\sqrt{\theta_tv_{t-1}\!+\!\theta_t\epsilon}}\frac{\sqrt{1-\theta_t}\sqrt{\sigma_t^2\!+\!\epsilon}}{\sqrt{\hat{v}_t\!+\!\epsilon} +\sqrt{\theta_tv_{t-1}\!+\!\theta_t\epsilon}}\right|\right\rangle\\
&\leq \mathbb{E} \left\langle \sqrt{\hat{\eta}_t}|\nabla f\left(x_t\right)|, \sqrt{\frac{\alpha\sqrt{G^2+\epsilon}}{\sqrt{\theta}}}\frac{\beta\sqrt{1-\theta_t}\sqrt{\epsilon}}{\sqrt{\theta_1C_1\left(1-\gamma\right)}\sqrt{v_t+\epsilon}}\right\rangle
\leq \frac{1-\beta}{10}\|\nabla f\left(x_t\right)\|_{\hat{\eta}_t}^2 + \frac{5\alpha \sqrt{G^2+\epsilon} \beta^2 \left(1-\theta_t\right)d}{2\left(1-\beta\right)\sqrt{\theta}\theta_1C_1\left(1-\gamma\right)},
\end{split}
\]
where the inequality holds according to
\[\small
\begin{split}
&\frac{\beta m_{t-1}}{\sqrt{\theta_tv_{t-1}+\theta_t\epsilon}}\frac{\sqrt{1-\theta_t}\sqrt{\epsilon}}{\sqrt{v_t+\epsilon}+\sqrt{\theta_tv_{t-1}+\theta_t\epsilon}}\frac{\sqrt{1-\theta_t}\sqrt{\sigma_t^2+\epsilon}}{\sqrt{\hat{v}_t+\epsilon} +\sqrt{\theta_tv_{t-1}+\theta_t\epsilon}}\\
&\leq \frac{\beta m_{t-1}}{\sqrt{\theta_tv_{t-1}+\theta_t\epsilon}}\frac{\sqrt{1-\theta_t}\sqrt{\epsilon}}{\sqrt{\left(1-\theta_t\right)\epsilon}}\frac{\sqrt{1-\theta_t}\sqrt{\sigma_t^2+\epsilon}}{\sqrt{\hat{v}_t+\epsilon}}\leq \frac{\beta}{\sqrt{\theta_1\left(1-\theta_t\right)C_1\left(1-\gamma\right)}}.
\end{split}
\]
Therefore, by combining the above upper estimations for the seven terms in Eq. \eqref{cross-gradient-eq}, we obtain
\[\small
  \begin{split}
      \mathbb{E}\langle\nabla f\left(x_t\right), \Delta_t\rangle \leq \frac{\beta\alpha_t}{\sqrt{\theta_t}\alpha_{t-1}}M_{t-1} + C_2\left(1-\theta_t\right) - \frac{1-\beta}{2}\mathbb{E}[\|\nabla f\left(x_t\right)\|^2_{\hat{\eta}_t}],
  \end{split}
\]
where
$\small
C_2\!= \!\frac{5\alpha G^3\left(1\!-\!\beta\right)}{2\epsilon\sqrt{\theta}}\left(\frac{\beta}{\left(1\!-\!\beta\right)\sqrt{\theta_1C_1\left(1\!-\!\gamma\right)}}\!+\!1\right)^2 \!+\! \frac{5\alpha G^3}{2\epsilon\sqrt{\theta}}
\!+\!\frac{5\beta^2 \alpha d\sqrt{\epsilon}}{2\sqrt{\theta}\left(1\!-\!\beta\right)\theta_1C_1\left(1\!-\!\gamma\right)}\!+\!\frac{5\alpha\sqrt{G^2+\epsilon}G^2\beta^2}{2\left(1\!-\!\beta\right)\sqrt{\theta}\theta_1C_1\left(1\!-\!\gamma\right)\epsilon}\!+\!\frac{5\alpha \sqrt{G^2+\epsilon} \beta^2 d}{2\left(1-\beta\right)\sqrt{\theta}\theta_1C_1\left(1-\gamma\right)}.$
On the other hand, let $\small  N_t = L\left(2-\delta_g\right)\mathbb{E}[\|\Delta_t\|^2]+ L\left(2-\delta_g\right)\mathbb{E}[\|e_t\|\|\Delta_t\|]+ C_2\left(1-\theta_t\right)$.

Recalling the definition of $\small M_{t}$. For $\small M_1$, we have
\[\small
\begin{split}
   M_1 &= \mathbb{E}\left[-\left\langle\nabla f\left(x_1\right), \frac{\alpha_1m_1}{\sqrt{v_1+\epsilon}}\right\rangle + L\left(2-\delta_g\right)\|\Delta_1\|^2 + L\left(2-\delta_g\right)\|e_1\|\|\Delta_1\|\right]\\
   & = \mathbb{E}\left[-\left\langle \nabla f\left(x_1\right), \left(1-\beta_1\right)\hat{\eta}_1g_1 + \hat{\eta}_1\sigma_1\frac{\left(1-\theta_1\right)g_1}{\sqrt{v_1+\epsilon}}\frac{\left(1-\beta\right)\sigma_1}{\sqrt{v_1+\epsilon}+\sqrt{\hat{v}_1+\epsilon}} \!-\! \hat{\eta}_1g_1\frac{\left(1-\theta_1\right)g_1}{\sqrt{v_1+\epsilon}}\frac{\left(1-\beta\right)g_1}{\sqrt{v_1+\epsilon}+\sqrt{\hat{v}_1+\epsilon}}\right\rangle\right]\\
  &\quad +\mathbb{E}\left[L\left(2-\delta_g\right)\|\Delta_1\|^2\right] + L\left(2-\delta_g\right)\mathbb{E}[\|e_1\|\|\Delta_1\|]
    \leq N_1,
\end{split}
\]
where the last inequality holds by using $A^1_t$ and $A^2_t$.

Let $\small  M_0 = 0$. Then we have
$\small M_t\leq \frac{{\beta\alpha_t}}{\sqrt{\theta_t}\alpha_{t-1}}M_{t-1} + N_t -\frac{1-\beta}{2}\mathbb{E}[\|\nabla f\left(x_t\right)\|^2_{\hat{\eta}_t}] \leq \frac{{\beta\alpha_t}}{\sqrt{\theta_t}\alpha_{t-1}}M_{t-1} + N_t$.
By induction we have
\[\small
\begin{split}
     M_t &\leq \sum_{i=1}^t \frac{\alpha_t\beta^{t-i}}{\alpha_i\sqrt{\Theta\left(t,i\right)}}N_i - \frac{1-\beta}{2}\mathbb{E}\left[\|\nabla f\left(x_t\right)\|_{\hat{\eta}_t}^2\right]
     \leq \frac{1}{\sqrt{C_1}}\sum_{i=1}^t \left(\frac{\beta}{\sqrt{\theta'}}\right)^{t-i}N_i - \frac{1-\beta}{2}\mathbb{E}\left[\|\nabla f\left(x_t\right)\|_{\hat{\eta}_t}^2\right]\\
     &\leq \frac{1}{\sqrt{C_1}}\sum_{i=1}^t \sqrt{\gamma}^{t-i}N_i - \frac{1-\beta}{2}\mathbb{E}\left[\|\nabla f\left(x_t\right)\|_{\hat{\eta}_t}^2\right].
\end{split}
\]
Therefore,
\[\small
\begin{split}
    \sum_{t=1}^T M_t
    &\leq \frac{1}{\sqrt{C_1}}\sum_{t=1}^T\sum_{i=1}^t\sqrt{\gamma}^{t-i}N_i - \frac{1-\beta}{2}\mathbb{E}\left[\sum_{t=1}^T\|\nabla f\left(x_t\right)\|^2_{\hat{\eta}_t}\right]\\
    &\leq \frac{1}{\sqrt{C_1}\left(1-\sqrt{\gamma}\right)}\sum_{t=1}^TN_t - \frac{1-\beta}{2}\mathbb{E}\left[\sum_{t=1}^T\|\nabla f\left(x_t\right)\|^2_{\hat{\eta}_t}\right]\\
    & = \frac{1}{\sqrt{C_1}\left(1\!-\!\sqrt{\gamma}\right)}\sum_{t=1}^T[L\left(2\!-\!\delta_g\right)\mathbb{E}\|\Delta_t\|^2\!+\!L\left(2\!-\!\delta_g\right)\mathbb{E}\|e_t\|\|\Delta_t\|\!+\!\frac{C_2\theta}{t}] \!-\! \frac{1\!-\!\beta}{2}\mathbb{E}\left[\sum_{t=1}^T\|\nabla f\left(x_t\right)\|^2_{\hat{\eta}_t}\right]\\
    &\leq \frac{1}{\sqrt{C_1}\left(1\!-\!\sqrt{\gamma}\right)}\left(\frac{L\left(2\!-\!\delta_g\right)G^2\alpha^2}{\epsilon\delta_g}\!+\!C_2\theta\right)\sum_{t=1}^T\frac{1}{t} \!-\! \frac{1\!-\!\beta}{2} \mathbb{E}\left[\sum_{t=1}^T\|\nabla f\left(x_t\right)\|^2_{\hat{\eta}_t}\right].
\end{split}
\]
Hence, the targeted result holds.
\end{proof}

%\begin{theorem}
%\label{T1}
%The following estimation always holds:
%\[\mathbb{E}[\|\nabla f\left(\left(x_\tau^T\right)\right)\|^2] \leq \frac{C+C'\sum_{t=1}^T\frac{1}{t}}{\sqrt{T}}
%\]
%
%Where
%\[
%\begin{split}
%C_3 &= \frac{1}{\sqrt{C_1}\left(1-\sqrt{\gamma}\right)}\left(\frac{L\left(1+\delta_g'\right)G^2\alpha^2}{\epsilon\delta_g}+C_2\theta\right),\\
%    C &= \frac{2\sqrt{G^2+\epsilon d}}{\left(1-\beta\right)\alpha}\left(f\left(x_1\right)-f^*\right),\\
%    C' &= \frac{2\sqrt{G^2+\epsilon d}C_3}{\left(1-\beta\right)\alpha}
%\end{split}
%\]
%\end{theorem}

\begin{proof}[Proof of Theorem \ref{T1}]
Based on Notation \eqref{notation-thm1} and the above lemmas, then we can prove Theorem 1. First, according to the gradient Lipschitz condition of $f$, it holds
\[\small
\begin{split}
f\left(\tilde{x}_{t+1}\right) &\leq f\left(\tilde{x}_t\right) + \langle\nabla f\left(\tilde{x}_t\right), \Delta_t\rangle + \frac{L}{2} \|\Delta_t^2\|\\
& =  f\left(\tilde{x}_t\right) + \langle\nabla f\left(x_t\right), \Delta_t\rangle + \frac{L}{2} \|\Delta_t\|^2 +\langle\nabla f\left(\tilde{x}_t\right)  - \nabla f\left(x_t\right), \Delta_t\rangle\\
&\leq  f\left(\tilde{x}_t\right) + \langle\nabla f\left(x_t\right), \Delta_t\rangle + \frac{L}{2} \|\Delta_t\|^2 + \|\nabla f\left(\tilde{x}_t\right) - \nabla f\left(x_t\right)\|\|\Delta_t\|\\
&\leq f\left(\tilde{x}_t\right) + \langle\nabla f\left(x_t\right), \Delta_t\rangle +  L\left(2-\delta_g\right)\|\Delta_t\|^2 + L\left(2-\delta_g\right)\|e_t\|\|\Delta_t\|.
\end{split}
\]
Recall that $\small M_t = \mathbb{E}[\langle\nabla f\left(x_t\right), \Delta_t\rangle+L\left(2-\delta_g\right)\|\Delta_t\|^2 + L\left(2-\delta_g\right)\|e_t\|\|\Delta_t\|]$.
Then we have
\[\small
\begin{split}
    f^*&\leq \mathbb{E}[f\left(\tilde{x}_{T+1}\right)] \leq f\left(x_1\right) + \sum_{t=1}^T M_t\\
    & \leq f\left(x_1\right) + \frac{1}{\sqrt{C_1}\left(1-\sqrt{\gamma}\right)}\left(\frac{L\left(2-\delta_g\right)G^2\alpha^2}{\epsilon\delta_g}+C_2\theta\right)\sum_{t=1}^T\frac{1}{t} - \frac{1-\beta}{2} \mathbb{E}\left[\sum_{t=1}^T\|\nabla f\left(x_t\right)\|^2_{\hat{\eta}_t}\right].\\
\end{split}
\]
Using the above lemmas and arranging the corresponding terms, we have
\[\small
\begin{split}
    \left(\mathbb{E}[\|\nabla f\left(\left(x_\tau^T\right)\right)\|^2]\right) &\leq \frac{\sqrt{G^2+\epsilon d}}{\alpha\sqrt{T}}\mathbb{E}\left[\sum_{t=1}^T\|\nabla f\left(x_t\right)\|_{\hat{\eta}_t}^2\right]\leq \frac{2\sqrt{G^2+\epsilon d}}{\left(1-\beta\right)\alpha\sqrt{T}}\left(f\left(x_1\right) - f^* + C_3\sum_{t=1}^T\frac{1}{t}\right)\\
    & \leq \frac{C+C'\sum_{t=1}^T\frac{1}{t}}{\sqrt{T}},
\end{split}
\]
where $\small C_3 = \frac{1}{\sqrt{C_1}\left(1-\sqrt{\gamma}\right)}\left(\frac{L\left(2-\delta_g\right)G^2\alpha^2}{\epsilon\delta_g}+C_2\theta\right)$,
    $\small C = \frac{2\sqrt{G^2+\epsilon d}}{\left(1-\beta\right)\alpha}\left(f\left(x_1\right)-f^*\right)$,
    $\small C' = \frac{2\sqrt{G^2+\epsilon d}C_3}{\left(1-\beta\right)\alpha}$,
respectively. Hence, the proof is completed.
\end{proof}

%In the following, we set the base learning rate $\alpha_{t}$ as a constant, i.e.,  $\alpha_t=\frac{\alpha}{\sqrt{T}}$. In this situation, Theorem 1 reduces to the following corollary.

\begin{proof}[Proof of Corollary \ref{C1}]
For a fixed iteration $T$, let $\small \alpha_t=\frac{\alpha}{\sqrt{T}}$ and $\small \theta = 1-\frac{\theta}{T}$.
When $\small \alpha_t=\frac{\alpha}{\sqrt{T}}$, Lemma \ref{sumdelta} will update to $\small \sum_{t=1}^T \|\Delta_t\|^2 \leq \frac{G^2\alpha^2}{\epsilon}$.
By the same proof in Lemma \ref{Mt}, we have
\[\small
\begin{split}
\sum_{t=1}^T M_t
&\leq \frac{1}{\sqrt{C_1}\left(1\!-\!\sqrt{\gamma}\right)} \sum_{t=1}^T[L\left(2\!-\!\delta_g\right)\mathbb{E}\|\Delta_t\|^2 \!+\! L\left(2\!-\!\delta_g\right)\mathbb{E}\|e_t\|\|\Delta_t\|\!+\! \frac{C_2\theta}{T}] \!-\! \frac{1\!-\!\beta}{2}\mathbb{E}\left[\sum_{t=1}^T\|\nabla f\left(x_t\right)\|^2_{\hat{\eta}_t}\right]\\
& \leq \frac{L\left(2-\delta_g\right)G^2\alpha^2}{\sqrt{C_1}\left(1-\sqrt{\gamma}\right) \epsilon \delta_g} + \frac{C_2\theta}{\sqrt{C_1}\left(1-\sqrt{\gamma}\right)} - \frac{1-\beta}{2}\mathbb{E}\left[\sum_{t=1}^T\|\nabla f\left(x_t\right)\|^2_{\hat{\eta}_t}\right].
\end{split}
\]
Based on the proof of Theorem \ref{T1}, we have
\[\small
\begin{split}
    \left(\mathbb{E}[\|\nabla f\left(\left(x_\tau^T\right)\right)\|^2]\right) &\leq \frac{\sqrt{G^2+\epsilon d}}{\alpha\sqrt{T}}\mathbb{E}\left[\sum_{t=1}^T\|\nabla f\left(x_t\right)\|_{\hat{\eta}_t}^2\right]\\
    &\leq \frac{2\sqrt{G^2+\epsilon d}}{\left(1-\beta\right)\alpha\sqrt{T}}\left(f\left(x_1\right) - f^* + \frac{L\left(2-\delta_g\right)G^2\alpha^2}{\sqrt{C_1}\left(1-\sqrt{\gamma}\right) \epsilon \delta_g} + \frac{C_2\theta}{\sqrt{C_1}\left(1-\sqrt{\gamma}\right)}\right).\end{split}
\]
%where
%\[
%C_3' = \frac{2\sqrt{G^2+\epsilon d}}{\left(1-\beta\right)\alpha\sqrt{T}}\left(f\left(x_1\right) - f^* + %\frac{L\left(2-\delta_g\right)G^2\alpha^2}{\sqrt{C_1}\left(1-\sqrt{\gamma}\right) \epsilon \delta_g}\right).
%\]
By bounding $\small \frac{2\sqrt{G^2+\epsilon d}}{\left(1-\beta\right)\alpha\sqrt{T}}\left(f\left(x_1\right) - f^* + \frac{L\left(2-\delta_g\right)G^2\alpha^2}{\sqrt{C_1}\left(1-\sqrt{\gamma}\right) \epsilon \delta_g} + \frac{C_2\theta}{\sqrt{C_1}\left(1-\sqrt{\gamma}\right)}\right)$ by $\xi$, we obtain $\small T = O\left(\frac{1}{\xi^2}\right)$.
\end{proof}

\subsection{Proof of Theorem \ref{T2}}
To prove Theorem \ref{T2}, we first define some notations and provide several useful lemmas.

\begin{notation}\label{notation-thm2}
Let $\small \mathbb{E}_t[\cdot]$ be the conditional expectation conditioned on  {$\small \{x_{t}, v_{t-1}, m_{t-1}\}$}.
Denote $\small \sigma_t^2 = \mathbb{E}_t[g_t^2]$,  $\small \hat{v}_t = \theta_tv_{t-1} + \left(1\!-\!\theta_t\right)\sigma_t^2$, $\small \hat{\eta}_t = \alpha_t/\sqrt{\hat{v}_t\!+\!\epsilon}$, $\small \Delta_t = -\alpha_tm_t/\sqrt{v_t\!+\!\epsilon}$, and $\small M_t\!=\! \mathbb{E}[\langle\nabla f\left(Q_x\left(x_t\right)\right), \Delta_t\rangle + L\|\Delta_t\|^2 + 2L\|\Delta_t\|]$.
\end{notation}

\begin{lemma}
\label{delta_x}
By using Notation \ref{notation-thm2}, the following inequality holds:
\[
\small
\sum_{t=1}^T \|\Delta_t\| \leq \frac{2G\alpha\sqrt{T}}{\sqrt{\epsilon}}.
\]
\end{lemma}
\begin{proof}
     \[\small
     \begin{split}
         \sum_{t=1}^T\|\Delta_t\|  &= \sum_{t=1}^T \alpha_t \|\frac{m_t}{\sqrt{v_t+\epsilon}}\|\leq  \sum_{t=1}^T \frac{G\alpha}{\sqrt{t\epsilon}} \leq \frac{2G\alpha\sqrt{T}}{\sqrt{\epsilon}}.
     \end{split}
     \]
\end{proof}

\begin{lemma}
\label{Mt2}
Based on Notation \ref{notation-thm2}, we have the following upper estimation for $\small M_{t}$,
\[\small
    \sum_{t=1}^T M_t \!\leq\! \frac{1}{\sqrt{C_1}\left(1\!-\!\sqrt{\gamma}\right)}\left(\frac{LG^2\alpha^2}{\epsilon}+C_2\theta\right)\sum_{t=1}^T\frac{1}{t} + \frac{4L\delta_xG\alpha\sqrt{T}}{\sqrt{\epsilon C_1}\left(1-\sqrt{\gamma}\right)} - \frac{1-\beta}{2}\mathbb{E}[\sum_{t=1}^T\|\nabla f\left(Q_x\left(x_t\right)\right)\|_{\hat{\eta}_t}^2],
\]
where
\[\small
\begin{split}
    C_2 &= \frac{5\alpha G^3\left(1-\beta\right)}{2\epsilon\sqrt{\theta}}\left(\frac{\beta}{\left(1-\beta\right)\sqrt{\theta_1C_1\left(1-\gamma\right)}}+1\right)^2 + \frac{5\alpha G^3}{2\epsilon\sqrt{\theta}}+\frac{5\beta^2 \alpha d\sqrt{\epsilon}}{2\sqrt{\theta}\left(1-\beta\right)\theta_1C_1\left(1-\gamma\right)}\\&+\frac{5\alpha\sqrt{G^2+\epsilon}G^2\beta^2}{2\left(1-\beta\right)\sqrt{\theta}\theta_1C_1\left(1-\gamma\right)\epsilon}+\frac{5\alpha \sqrt{G^2+\epsilon} \beta^2 d}{2\left(1-\beta\right)\sqrt{\theta}\theta_1C_1\left(1-\gamma\right)}.
\end{split}
\]
\end{lemma}
\begin{proof}
Denoting the same notations $\small A^1_t,A^2_t,A^3_t,A^4_t,A^5_t$ in Lemma \ref{Mt}, we have
\begin{equation}\label{cross-gradient-thm2}
\small
\begin{split}
\mathbb{E}\langle\nabla f\left(Q_x\left(x_t\right)\right),\Delta_t\rangle
&= \frac{\beta\alpha_t}{\sqrt{\theta_t}\alpha_{t-1}}\mathbb{E}\langle\nabla f\left(Q_x\left(x_t\right)\right),\Delta_{t-1}\rangle + \mathbb{E}\langle\nabla f\left(Q_x\left(x_t\right)\right),\Delta_t - \frac{\beta\alpha_t}{\sqrt{\theta_t}\alpha_t}\Delta_{t-1}\rangle\\
&= \frac{\beta\alpha_t}{\sqrt{\theta_t}\alpha_{t-1}}\left(\mathbb{E}\langle\nabla f\left(Q_x\left(x_{t-1}\right)\right),\Delta_{t-1}\rangle + \mathbb{E}\langle\nabla f\left(Q_x\left(x_t\right)\right) - \nabla f\left(Q_x\left(x_{t-1}\right)\right),\Delta_{t-1}\rangle\right)\\ &+ \mathbb{E}\langle\nabla f\left(Q_x\left(x_t\right)\right),-\left(1-\beta\right)\hat{\eta}_tg_t\rangle + \mathbb{E}\langle\nabla f\left(Q_x\left(x_t\right)\right),A^1_t\rangle +
\mathbb{E}\langle\nabla f\left(Q_x\left(x_t\right)\right),A^2_t\rangle\\&+\mathbb{E}\langle\nabla f\left(Q_x\left(x_t\right)\right),A^3_t\rangle+\mathbb{E}\langle\nabla f\left(Q_x\left(x_t\right)\right),A^4_t\rangle+\mathbb{E}\langle\nabla f\left(Q_x\left(x_t\right)\right),A^5_t\rangle.
\end{split}
\end{equation}
For the first term in Eq.~\eqref{cross-gradient-thm2}, we have
\[\small
\begin{aligned}
&\frac{\beta\alpha_t}{\sqrt{\theta_t}\alpha_{t-1}}\left(\mathbb{E}\langle\nabla f\left(Q_x\left(x_{t-1}\right)\right),\Delta_{t-1}\rangle + \mathbb{E}\langle\nabla f\left(Q_x\left(x_t\right)\right) - \nabla f\left(Q_x\left(x_{t-1}\right)\right),\Delta_{t-1}\rangle\right)\\
&\leq \frac{\beta\alpha_t}{\sqrt{\theta_t}\alpha_{t-1}}\left(\mathbb{E}\langle\nabla f\left(Q_x\left(x_{t-1}\right)\right),\Delta_{t-1}\rangle + \mathbb{E}[L\|Q_x\left(x_t\right) - Q_x\left(x_{t-1}\right)\|\|\Delta_{t-1}\|]\right)\\
&\leq \!\frac{\beta\alpha_t}{\sqrt{\theta_t}\alpha_{t-1}}\!\left(\!\mathbb{E}\langle\nabla\!f\!\left(Q_x\left(x_{t-1}\right)\right),\!\Delta_{t-1}\rangle \!+\! \mathbb{E}[\!L\!\left(\!\|Q_x\!\left(x_t\right) \! -\! x_t\| \!+\! \|x_t \!-\! x_{t-1}\|\!+\!\|Q_x\!\left(x_{t-1}\right)\!-\! x_{t-1}\|\right)\|\Delta_{t-1}\|]\right)\\
& {
\leq \frac{\beta\alpha_t}{\sqrt{\theta_t}\alpha_{t-1}}\left(\mathbb{E}\langle\nabla f\left(Q_x\left(x_{t-1}\right)\right),\Delta_{t-1}\rangle + \mathbb{E}[L\left(\delta_x + \|\Delta_{t-1}\| + \delta_x \right)\|\Delta_{t-1}\|]\right)
}\\
&\leq \frac{\beta\alpha_t}{\sqrt{\theta_t}\alpha_{t-1}}\left(\mathbb{E}\langle\nabla f\left(Q_x\left(x_{t-1}\right)\right),\Delta_{t-1}\rangle + L\mathbb{E}[\|\Delta_{t-1}\|^2] + 2L\delta_x\mathbb{E}[\|\Delta_{t-1}\|]\right) = \frac{\beta\alpha_t}{\sqrt{\theta_t}\alpha_{t-1}}M_{t-1}.
\end{aligned}
\]
The rest six parts remain the same as Lemma \ref{Mt}. Then we have
\[\small
  \begin{split}
      \mathbb{E}\langle\nabla f\left(Q_x\left(x_t\right)\right), \Delta_t\rangle \leq \frac{\beta\alpha_t}{\sqrt{\theta_t}\alpha_{t-1}}M_{t-1} + C_2\left(1-\theta_t\right) - \frac{1-\beta}{2}\mathbb{E}[\|\nabla f\left(Q_x\left(x_t\right)\right)\|^2_{\hat{\eta}_t}],
  \end{split}
\]
where $\small C_2 \!=\! \frac{5\alpha G^3\left(1\!-\!\beta\right)}{2\epsilon\sqrt{\theta}}\left(\frac{\beta}{\left(1\!-\!\beta\right)\sqrt{\theta_1C_1\left(1\!-\!\gamma\right)}}\!+\!1\right)^2 \!+\! \frac{5\alpha G^3}{2\epsilon\sqrt{\theta}}\!+\!\frac{5\beta^2 \alpha d\sqrt{\epsilon}}{2\sqrt{\theta}\left(1\!-\!\beta\right)\theta_1C_1\left(1\!-\!\gamma\right)}\!+\!\frac{5\alpha\sqrt{G^2+\epsilon}G^2\beta^2}{2\left(1\!-\!\beta\right)\sqrt{\theta}\theta_1C_1\left(1\!-\!\gamma\right)\epsilon}\!+\!\frac{5\alpha \sqrt{G^2+\epsilon} \beta^2 d}{2\left(1\!-\!\beta\right)\sqrt{\theta}\theta_1C_1\left(1\!-\!\gamma\right)}$.

Define $\small N_t=L\mathbb{E}[\|\Delta_{t-1}\|^2] + 2L\delta_x\mathbb{E}[\|\Delta_{t-1}\|] + C_2\left(1-\theta_t\right)$.
By applying the same induction in Lemma \ref{Mt}, we have
\[\small
\begin{split}
\sum_{t=1}^TM_t &\leq \frac{1}{\sqrt{C_1}\left(1-\sqrt{\gamma}\right)}\sum_{t=1}^TN_t - \frac{1-\beta}{2}\mathbb{E}[\sum_{t=1}^T\|\nabla f\left(Q_x\left(x_t\right)\right)\|_{\hat{\eta}_t}^2]\\
&=\frac{1}{\sqrt{C_1}\left(1-\sqrt{\gamma}\right)}\left(\frac{LG^2\alpha^2}{\epsilon}+C_2\theta\right)\sum_{t=1}^T\frac{1}{t} + \frac{4L\delta_xG\alpha\sqrt{T}}{\sqrt{\epsilon C_1}\left(1-\sqrt{\gamma}\right)} - \frac{1-\beta}{2}\mathbb{E}[\sum_{t=1}^T\|\nabla f\left(Q_x\left(x_t\right)\right)\|_{\hat{\eta}_t}^2].
\end{split}
\]
Hence, we obtain the desired result.
\end{proof}

\begin{proof}[Proof of Theorem \ref{T2}]
By the Lipschitz continuity of the gradient, we have
\[\small
\begin{split}
f\left(x_{t+1}\right) &\leq f\left(x_t\right) + \langle\nabla f\left(x_t\right),\Delta_t\rangle + \frac{L}{2}\|\Delta_t\|^2\\
& = f\left(x_t\right) + \langle\nabla f\left(Q_x\left(x_t\right)\right),\Delta_t\rangle + \frac{L}{2}\|\Delta_t\|^2 + \langle\nabla f\left(x_t\right) - \nabla f\left(Q_x\left(x_t\right)\right),\Delta_t\rangle \\
& \leq  f\left(x_t\right) + \langle\nabla f\left(Q_x\left(x_t\right)\right),\Delta_t\rangle + \frac{L}{2}\|\Delta_t\|^2 + L\delta_x\|\Delta_t\|.
\end{split}
\]
Let $\small M_t = \mathbb{E}[\langle\nabla f\left(Q_t\left(x_t\right)\right), \Delta_t\rangle+L\|\Delta_t\|^2 + 2L\delta_x\|\Delta_t\|]$. Then we have
\[\small
\begin{split}
f^* &\leq E[f\left(x_{T+1}\right)] \leq f\left(x_1\right) + \sum_{t=1}^T M_t\\
& \leq f\left(x_1\right) \!+\! \frac{1}{\sqrt{C_1}\left(1\!-\!\sqrt{\gamma}\right)}\left(\frac{LG^2\alpha^2}{\epsilon}\!+\!C_2\theta\right)\sum_{t=1}^T\frac{1}{t} \!+\! \frac{4L\delta_xG\alpha\sqrt{T}}{\sqrt{\epsilon C_1}\left(1\!-\!\sqrt{\gamma}\right)} \!-\! \frac{1\!-\!\beta}{2}\mathbb{E}[\sum_{t=1}^T\|\nabla f\left(Q_x\left(x_t\right)\right)\|_{\hat{\eta}_t}^2].
\end{split}
\]
Arranging the corresponding terms suitably, we obtain
\[
\small
\begin{split}
    &\mathbb{E}\left[\left\|\nabla f\left(Q_x\left(x_\tau^T\right)\right)\right\|^2\right] \leq \frac{\sqrt{G^2+\epsilon d}}{\alpha\sqrt{T}}\mathbb{E}\left[\sum_{t=1}^T\|\nabla f\left(Q\left(x_t\right)\right)\|_{\hat{\eta}_t}^2\right]\\
    &\leq \frac{2\sqrt{G^2+\epsilon d}}{\left(1-\beta\right)\alpha\sqrt{T}}\left(f\left(x_1\right) - f^* + \frac{1}{\sqrt{C_1}\left(1-\sqrt{\gamma}\right)}\left(\frac{LG^2\alpha^2}{\epsilon}+C_2\theta\right)\sum_{t=1}^T\frac{1}{t} + \frac{4L\delta_xG\alpha\sqrt{T}}{\sqrt{\epsilon C_1}\left(1-\sqrt{\gamma}\right)}\right)\\
    &= \frac{C_5+C_6\sum_{t=1}^T\frac{1}{t}}{\sqrt{T}} + C_7,
\end{split}
\]
where $\small C_5$, $\small C_6$, and $\small C_7$ are defined as
$\small
C_5 = \frac{2\sqrt{G^2+\epsilon d}}{\left(1-\beta\right)\alpha}\left(f\left(x_1\right) - f^*\right),
     C_6 = \frac{2\sqrt{G^2+\epsilon d}}{\left(1-\beta\right)\alpha \sqrt{C_1}\left(1-\sqrt{\gamma}\right)}\left(\frac{LG^2\alpha^2}{\epsilon}+C_2\theta\right),$ and
$\small C_7  = \frac{8\delta_x\sqrt{G^2+\epsilon d} LG}{\left(1-\beta\right)\sqrt{\epsilon C_1}\left(1-\sqrt{\gamma}\right)}$,
respectively.

Hence, the proof is completed.
\end{proof}

%Bellow, we set the base learning rate as a constant. Then, we obtain the following corollary.

\begin{proof}[proof of corollary \ref{C2}]
For a fixed iteration $T$, let $\small \alpha_t = \frac{\alpha}{\sqrt{T}}$ and $\theta_t = 1-\frac{\theta}{T}$.
It is not hard to check that  Lemma \ref{delta_x} updates to  $\small \sum_{t=1}^T \|\Delta_t\|\leq \frac{\alpha G\sqrt{T}}{\sqrt{\epsilon}}$,
and Lemma \ref{Mt2} updates to
\[
\small
\sum_{t=1}^T M_t \leq \frac{LG^2\alpha^2}{\sqrt{C_1}\left(1-\sqrt{\gamma}\right)\epsilon} +\frac{C_2\theta}{\sqrt{C_1}\left(1-\sqrt{\gamma}\right)}+ \frac{2L\delta_xG\alpha\sqrt{T}}{\sqrt{\epsilon C_1}\left(1-\sqrt{\gamma}\right)} - \frac{1-\beta}{2}\mathbb{E}[\sum_{t=1}^T\|\nabla f\left(Q_x\left(x_t\right)\right)\|_{\hat{\eta}_t}^2].
\]
Then, we have
\[\small
\begin{split}
&\mathbb{E}\left[\left\|\nabla f\left(Q_x\left(x_\tau^T\right)\right)\right\|^2\right]\leq \frac{\sqrt{G^2+\epsilon d}}{\alpha\sqrt{T}}\mathbb{E}\left[\sum_{t=1}^T\|\nabla f\left(Q\left(x_t\right)\right)\|_{\hat{\eta}_t}^2\right]%\\
 %   &\leq \frac{2\sqrt{G^2+\epsilon d}}{\left(1-\beta\right)\alpha\sqrt{T}}\left(f\left(x_1\right) - f^* +\frac{LG^2\alpha^2}{\sqrt{C_1}\left(1-\sqrt{\gamma}\right)\epsilon} +\frac{C_2\theta}{\sqrt{C_1}\left(1-\sqrt{\gamma}\right)} + \frac{2L\delta_xG\alpha\sqrt{T}}{\sqrt{\epsilon C_1}\left(1-\sqrt{\gamma}\right)} \right) =
 \leq \frac{C_5'}{\sqrt{T}} + C_7',
\end{split}
\]
where $\small  C_5'=\frac{2\sqrt{G^2+\epsilon d}}{\left(1-\beta\right)\alpha}\left(f\left(x_1\right) - f^* +\frac{LG^2\alpha^2}{\sqrt{C_1}\left(1-\sqrt{\gamma}\right)\epsilon} +\frac{C_2\theta}{\sqrt{C_1}\left(1-\sqrt{\gamma}\right)}\right)$, $\small C_7'=\frac{4\delta_x\sqrt{G^2+\epsilon d} LG}{\left(1-\beta\right)\sqrt{\epsilon C_1}\left(1-\sqrt{\gamma}\right)}$.

Hence, by bounding $\small \frac{C_5'}{\sqrt{T}} + C_7' \leq C_7' + \xi$, we obtain the desired result.
\end{proof}

\subsection{Proof of Theorem \ref{T3}}
To prove Theorem 3, we first denote a few notations and provide several useful lemmas.
\begin{notation}\label{notation-thm3}
Let $\small \mathbb{E}_t[\cdot]$ be the conditional expectation conditioned on {$\small \{x_{t}, v_{t-1}, m_{t-1}\}$}.
Denote $\small {\sigma_t^{\left(i\right)}}^2 = \mathbb{E}_t[{g_t^{\left(i\right)}}^2]$,  $\hat{v}^{\left(i\right)}_t = \theta_tv^{\left(i\right)}_{t-1} + \left(1-\theta_t\right){\sigma^{\left(i\right)}_t}^2$,
$\small \hat{\eta}^{\left(i\right)}_t = \alpha_t/\sqrt{\hat{v}^{\left(i\right)}_t+\epsilon}$,
$\small \Delta_t^{\left(i\right)} = -\alpha_tm_t^{\left(i\right)}/\sqrt{v_t^{\left(i\right)}+\epsilon}$,
$\small \hat{\Delta}_i=\frac{1}{N} \sum_{i=1}^N\Delta_t^{\left(i\right)}$,
and
$\small M_t = \frac{1}{N}\sum_{i=1}^N\mathbb{E}\langle\nabla f\left(\hat{x}_{t}\right),\Delta_{t}^{\left(i\right)}\rangle + \frac{1}{N^2}L\left(2\!-\!\delta_g\right)\sum_{i=1}^N\sum_{j=1}^N\mathbb{E} \|\Delta_{t}^{\left(i\right)}\|\|\Delta_{t}^{\left(j\right)}\|\!+\!\frac{1}{N^2}L\left(2\!-\!\delta_g\right)\sum_{i=1}^N\sum_{j=1}^N\mathbb{E}\|\Delta_{t}^{\left(i\right)}\|\|e_{t}^{\left(j\right)}\|\! +\! \frac{1}{N}2L\sum_{i=1}^N \mathbb{E}\|\Delta_t^{\left(i\right)}\|$. In addition, let $\small \tilde{x}_t \!=\! x_t \!-\! \frac{1}{N}\sum_{i=1}^N e_{t}^{\left(i\right)}$. Then
\begin{align*}
\small
\tilde{x}_{t+1} &= x_t - \frac{1}{N}\sum_{i=1}^N Q_g\left(-\Delta_t^{\left(i\right)}+e_t^{\left(i\right)}\right) - \frac{1}{N}\sum_{i=1}^N e_{t+1}^{\left(i\right)}
= x_t + \frac{1}{N}\sum_{i=1}^N \Delta_t^{\left(i\right)} - \frac{1}{N}\sum_{i=1}^N e_t = \tilde{x}_t \!+\!  \frac{1}{N}\sum_{i=1}^N \Delta_t^{\left(i\right)}.
\end{align*}
\end{notation}

\begin{lemma}
\label{multi_sum}
With Notation \ref{notation-thm3}, we derive an upper estimation for $\small \hat{\Delta}_t$ as:
$\small
   \sum_{t=1}^T\|\hat{\Delta}_t\|^2 \leq \frac{G^2}{\epsilon} \sum_{t=1}^T \frac{\alpha^2}{t}.$
\end{lemma}
\begin{proof}
By using the definition of $\small \hat{\Delta}_t$, it is not hard to check that the following equations hold:
\[\small
\sum_{t=1}^T\|\hat{\Delta}_t\|^2 = \sum_{t=1}^T\|\frac{1}{N}\sum_{i=1}^N\Delta_t^{\left(i\right)}\|^2\leq \sum_{t=1}^T\frac{1}{N}\sum_{i=1}^N\|\Delta_t^{\left(i\right)}\|^2
= \frac{1}{N}\sum_{i=1}^N\sum_{t=1}^T\|\Delta_t^{\left(i\right)}\|^2 \leq \frac{G^2}{\epsilon} \sum_{t=1}^T \frac{\alpha^2}{t}.
\]
\end{proof}

\begin{lemma}
Let $e_{i}$ be the noisy term in Algorithm2 and $\small \hat{\Delta}_t$ be the term defined in Notation \ref{notation-thm3}. Then it holds that
\[\small
\mathbb{E}\left[\frac{1}{N^2}\sum_{t=1}^T\sum_{i=1}^N\sum_{j=1}^N \|e_t^{\left(i\right)}\|\|\Delta_t^{\left(j\right)}\|\right] \leq \frac{1-\delta_g}{\delta_gN}\sum_{i=1}^N \mathbb{E}\left[\sum_{t=1}^T\|\Delta_t^{\left(i\right)}\|^2\right].
\]
\end{lemma}
\begin{proof}
 {
Using the definition of the noisy term  $e_{t}$, the following holds:
\[\small
    \|e^{\left(j\right)}_t\| \!=\! \|\Delta_t^{\left(j\right)} + e_{t-1}^{\left(j\right)} - Q_g\left(\Delta_t^{\left(j\right)} + e_{t-1}^{\left(j\right)}\right)\| \leq \left(1-\delta_g\right) \left(\|\Delta_t^{\left(j\right)}\| \!+\! \|e_{t-1}^{\left(j\right)}\|\right)
    \!\leq\! \sum_{k = 1}^t \left(1-\delta_g\right)^{t-k+1} \|\Delta_k^{\left(j\right)}\|.
\]
Then, we have
\[\small
\begin{split}
&\mathbb{E}\left[\frac{1}{N^2}\sum_{t=1}^T\sum_{i=1}^N\sum_{j=1}^N \|e_t^{\left(i\right)}\|\|\Delta_t^{\left(j\right)}\|\right] \leq
\mathbb{E}\left[\frac{1}{N^2}\sum_{t=1}^T\sum_{i=1}^N\sum_{j=1}^N\sum_{k=1}^t \left(1-\delta_g\right)^{t-k+1}\|\Delta_k^{\left(i\right)}\|\|\Delta_t^{\left(j\right)}\|\right]\\
&\leq \mathbb{E}\left[\frac{1}{N^2}\sum_{t=1}^T\sum_{i=1}^N\sum_{j=1}^N\sum_{k=1}^t  \frac{\left(1-\delta_g\right)^{t-k+1}}{2}\left(\|\Delta_k^{\left(i\right)}\|^2 +\|\Delta_t^{\left(j\right)}\|^2\right)\right]\\
& = \mathbb{E}\left[\frac{1}{N}\sum_{t=1}^T\sum_{i=1}^N \sum_{k=1}^t  \frac{\left(1-\delta_g\right)^{t-k+1}}{2}\left(\|\Delta_k^{\left(i\right)}\|^2\right) + \frac{1}{N}\sum_{t=1}^T\sum_{j=1}^N\sum_{k=1}^t \frac{\left(1-\delta_g\right)^{t-k+1}}{2}\left(\|\Delta_t^{\left(j\right)}\|^2\right)\right] \\
&\leq \mathbb{E}\left[\frac{1-\delta_g}{2\delta_gN} \sum_{k=1}^T\sum_{i=1}^N \|\Delta_k^{\left(i\right)}\|^2 + \frac{1-\delta_g}{2\delta_gN} \sum_{t=1}^T\sum_{j=1}^N \|\Delta_t^{\left(j\right)}\|^2\right]
= \mathbb{E}\left[\frac{1-\delta_g}{\delta_gN}\sum_{t=1}^T\sum_{i=1}^N \|\Delta_t^{\left(i\right)}\|^2\right].
\end{split}
\]
}
\end{proof}

\begin{lemma}
 Let $\tau$ be randomly chosen from $\small \{1,2,\cdots,T\}$ with equal probabilities $\small p_\tau = \frac{1}{T}$. We have the following estimate:
  \[\small
  \mathbb{E}[\|\nabla f\left(\hat{x}_\tau\right)\|^2]\leq \frac{\sqrt{G^2+\epsilon d}}{\alpha\sqrt{T}N}\sum_{i=1}^N\mathbb{E}\left[\sum_{t=1}^T\|\nabla f\left(\left(\hat{x}_t\right)\right)\|_{\hat{\eta_t}^{\left(i\right)}}^2\right].
  \]
\end{lemma}
\begin{proof}
  By Lemma \ref{ft}, $\small \mathbb{E}[\|\nabla f\left(\hat{x}_\tau\right)\|^2]\leq \frac{\sqrt{G^2+\epsilon d}}{\alpha\sqrt{T}}\mathbb{E}[\sum_{t=1}^T\|\nabla f\left(\hat{x}_t\right)\|_{\hat{\eta_t}^{\left(i\right)}}^2]$ holds for any $i$. Hence, the proof is finished.
\end{proof}

\begin{lemma}
\label{multi_Mt}
By the definition of $\small M_{t}$, we obtain its upper-estimation as follows:
\[\small
\sum_{t=1}^T M_t
\!\leq\! \frac{1}{\sqrt{C_1}\left(1\!-\!\sqrt{\gamma}\right)}\left(\left(\frac{L\left(2\!-\!\delta_g\right)G^2\alpha^2}{\epsilon\delta_g}\!+\!C_2\theta\right)\sum_{t=1}^T\frac{1}{t}\!+\!\frac{4\delta_xLG\alpha\sqrt{T}}{\sqrt{\epsilon}}\right)\!-\! \frac{1\!-\!\beta}{2N}\sum_{i=1}^N\mathbb{E}\left[\sum_{t=1}^T\|\nabla \!f\left(\hat{x}_t\right)\|^2_{\hat{\eta}_t^{\left(i\right)}}\right].
\]
\end{lemma}
\begin{proof}
Based on the similar proof of $\small \Delta_t - \frac{\beta\alpha_t}{\sqrt{\theta_t}\alpha_t}\Delta_{t-1}$ in Lemma \ref{Mt}, we define
\[\small
\left(A^1_t\right)^{\left(i\right)} = \hat{\eta}_t^{\left(i\right)}g_t^{\left(i\right)}\frac{\left(1-\theta_t\right)g_t^{\left(i\right)}}{\sqrt{v_t^{\left(i\right)}+\epsilon}}\left(\frac{\beta m_{t-1}^{\left(i\right)}}{\sqrt{v_t^{\left(i\right)}+\epsilon}+\sqrt{\theta_tv_{t-1}^{\left(i\right)}+\theta_t\epsilon}} + \frac{\left(1-\beta\right)g_t^{\left(i\right)}}{\sqrt{v_t^{\left(i\right)}+\epsilon}+\sqrt{\hat{v}_t^{\left(i\right)}}+\epsilon}\right),\\\]
\[\small
\left(A^2_t\right)^{\left(i\right)} \!=\! - \hat{\eta}_t^{\left(i\right)}\sigma_t^{\left(i\right)}\frac{\left(1\!-\!\theta_t\right)g_t^{\left(i\right)}}{\sqrt{v_t^{\left(i\right)}\!+\!\epsilon}}\frac{\left(1\!-\!\beta\right)\sigma_t^{\left(i\right)}}{\sqrt{v_t^{\left(i\right)}\!+\!\epsilon}\!+\!\sqrt{\hat{v}^{\left(i\right)}_t\!+\!\epsilon}},\ \left(A^3_t\right)^{\left(i\right)}
\!=\! \hat{\eta}_t^{\left(i\right)} \epsilon\frac{\left(1\!-\!\theta_t\right)}{\sqrt{v_t^{\left(i\right)}\!+\!\epsilon}}\frac{\beta m_{t-1}^{\left(i\right)}}{\sqrt{v_t^{\left(i\right)}\!+\!\epsilon}\!+\!\sqrt{\theta_tv_{t-1}^{\left(i\right)}\!+\!\theta_t\epsilon}},\\\]
\[\small
\left(A^4_t\right)^{\left(i\right)} \!=\! \hat{\eta}_t^{\left(i\right)}\sqrt{{\sigma_t^{\left(i\right)}}^2\!+\!\epsilon}\frac{\left(1\!-\!\theta_t\right)g_t^{\left(i\right)}}{\sqrt{v_t^{\left(i\right)}\!+\!\epsilon}}\frac{\beta m_{t-1}^{\left(i\right)}}{\sqrt{\theta_tv_{t-1}^{\left(i\right)}\!+\!\theta_t\epsilon}}\frac{\sqrt{1\!-\!\theta_t}g_t^{\left(i\right)}}{\sqrt{v_t^{\left(i\right)}\!+\!\epsilon}\!+\!\sqrt{\theta_tv_{t-1}^{\left(i\right)}\!+\!\theta_t\epsilon}}\frac{\sqrt{1\!-\!\theta_t}\sqrt{{\sigma_t^{\left(i\right)}}^2\!+\!\epsilon}}{\sqrt{\hat{v}_t^{\left(i\right)}\!+\!\epsilon} \!+\!\sqrt{\theta_tv_{t-1}^{\left(i\right)}\!+\!\theta_t\epsilon}},\\\]
\[\small
\left(A^5_t\right)^{\left(i\right)} \!=\! \hat{\eta}_t^{\left(i\right)}\sqrt{{\sigma_t^{\left(i\right)}}^2\!+\!\epsilon}\frac{\left(1\!-\!\theta_t\right)\sqrt{\epsilon}}{\sqrt{v_t^{\left(i\right)}\!+\!\epsilon}}\frac{\beta m_{t-1}}{\sqrt{\theta_tv_{t-1}^{\left(i\right)}\!+\!\theta_t\epsilon}}\frac{\sqrt{1\!-\!\theta_t}\sqrt{\epsilon}}{\sqrt{v_t^{\left(i\right)}\!+\!\epsilon}\!+\!\sqrt{\theta_tv_{t-1}^{\left(i\right)}\!+\!\theta_t\epsilon}}\frac{\sqrt{1\!-\!\theta_t}\sqrt{{\sigma_t^{\left(i\right)}}^2\!+\!\epsilon}}{\sqrt{\hat{v}_t^{\left(i\right)}\!+\!\epsilon} \!+\!\sqrt{\theta_tv_{t-1}^{\left(i\right)}\!+\!\theta_t\epsilon}}.
\]
Then, via the same proof as Lemma \ref{Mt}, the following equation holds
\begin{equation}\label{cross-gradient-them2}
\small
\begin{split}
\mathbb{E}\langle\nabla f\left(\hat{x}_t\right),\hat{\Delta}_t\rangle
&=\! \frac{1}{N}\!\sum_{i=1}^N \mathbb{E}\langle\nabla f\left(\hat{x}_t\right),\Delta_t^{\left(i\right)} \rangle \!
=\! \frac{1}{N}\!\sum_{i=1}^N \frac{\beta\alpha_t}{\sqrt{\theta_t}\alpha_{t-1}}\mathbb{E}\langle\nabla f\left(\hat{x}_t\right),\Delta_{t-1}^{\left(i\right)}\rangle\! +\! \mathbb{E}\langle\nabla f\left(\hat{x}_t\right),\Delta_t^{\left(i\right)} \!-\! \frac{\beta\alpha_t}{\sqrt{\theta_t}\alpha_t}\Delta_{t-1}^{\left(i\right)}\rangle\\
&= \frac{1}{N}\sum_{i=1}^N \frac{\beta\alpha_t}{\sqrt{\theta_t}\alpha_{t-1}}\left(\mathbb{E}\langle\nabla f\left(\hat{x}_{t-1}\right),\Delta_{t-1}^{\left(i\right)}\rangle + \mathbb{E}\langle\nabla f\left(\hat{x}_t\right) - \nabla f\left(\hat{x}_{t-1}\right),\Delta_{t-1}^{\left(i\right)}\rangle\right)\\
&\qquad + \mathbb{E}\langle\nabla f\left(\hat{x}_t\right),-\left(1-\beta\right)\hat{\eta}_t^{\left(i\right)}g_t^{\left(i\right)}\rangle + \mathbb{E}\langle\nabla f\left(\hat{x}_t^{\left(i\right)}\right),\left(A^1_t\right)^{\left(i\right)}\rangle +
\mathbb{E}\langle\nabla f\left(\hat{x}_t^{\left(i\right)}\right),\left(A^2_t\right)^{\left(i\right)}\rangle\\
&\qquad +\mathbb{E}\langle\nabla f\left(\hat{x}_t^{\left(i\right)}\right),\left(A^3_t\right)^{\left(i\right)}\rangle+\mathbb{E}\langle\nabla f\left(\hat{x}_t^{\left(i\right)}\right),\left(A^4_t\right)^{\left(i\right)}\rangle+\mathbb{E}\langle\nabla f\left(\hat{x}_t^{\left(i\right)}\right),\left(A^5_t\right)^{\left(i\right)}\rangle.
\end{split}
\end{equation}
For the first term in Eq.~\eqref{cross-gradient-them2}, it holds that
\begin{align*}
\small
&\frac{1}{N}\sum_{i=1}^N\frac{\beta\alpha_t}{\sqrt{\theta_t}\alpha_{t-1}}\left(\mathbb{E}\langle\nabla f\left(\hat{x}_{t-1}\right),\Delta_{t-1}^{\left(i\right)}\rangle + \mathbb{E}\langle\nabla f\left(\hat{x}_t\right) - \nabla f\left(\hat{x}_{t-1}\right),\Delta_{t-1}^{\left(i\right)}\rangle\right)\\
&\leq \frac{1}{N}\sum_{i=1}^N\frac{\beta\alpha_t}{\sqrt{\theta_t}\alpha_{t-1}}\left(\mathbb{E}\langle\nabla f\left(\hat{x}_{t-1}\right),\Delta_{t-1}^{\left(i\right)}\rangle + L\mathbb{E} \|\hat{x}_t -\hat{x}_{t-1}\| \| \Delta_{t-1}^{\left(i\right)}\|\right)
%&\leq \frac{1}{N}\sum_{i=1}^N\frac{\beta\alpha_t}{\sqrt{\theta_t}\alpha_{t-1}}\left(\mathbb{E}\langle\nabla f\left(\hat{x}_{t-1}\right),\Delta_{t-1}^{\left(i\right)}\rangle + L\mathbb{E} \|x_t -x_{t-1}\| \| \Delta_{t-1}^{\left(i\right)}\| + 2\delta_xL\mathbb{E}\|\Delta_{t-1}^{\left(i\right)}\|\right)\\
%&\leq \frac{1}{N}\sum_{i=1}^N\frac{\beta\alpha_t}{\sqrt{\theta_t}\alpha_{t-1}}\left(\mathbb{E}\langle\nabla f\left(\hat{x}_{t-1}\right),\Delta_{t-1}^{\left(i\right)}\rangle + L\mathbb{E} \|\frac{1}{N}\sum_{j=1}^N Q_g\left(\Delta_t^{\left(j\right)}+e_t^{\left(j\right)}\right)\| \| \Delta_{t-1}^{\left(i\right)}\|+ 2\delta_xL\mathbb{E}\|\Delta_{t-1}^{\left(i\right)}\|\right)\\
%&\leq \frac{\beta\alpha_t}{\sqrt{\theta_t}\alpha_{t-1}}\left(\frac{1}{N}\sum_{i=1}^N\mathbb{E}\langle\nabla f\left(\hat{x}_{t-1}\right),\Delta_{t-1}^{\left(i\right)}\rangle + \frac{1}{N^2}L\left(2-\delta_g\right)\sum_{i=1}^N\sum_{j=1}^N\mathbb{E} \|\Delta_{t-1}^{\left(i\right)}\|\|\Delta_{t-1}^{\left(j\right)}\|\\
%&\qquad +\frac{1}{N^2}L\left(2-\delta_g\right)\sum_{i=1}^N\sum_{j=1}^N\mathbb{E}\|\Delta_{t-1}^{\left(i\right)}\|\|e_{t-1}^{\left(j\right)}\| +\frac{1}{N}2\delta_xL\sum_{i=1}^N\mathbb{E}\|\Delta_{t-1}^{\left(i\right)}\| \right)\\
=\frac{\beta\alpha_t}{\sqrt{\theta_t}\alpha_{t-1}}M_{t-1}.
\end{align*}
The remain 6 terms in Eq.~\eqref{cross-gradient-them2} are the same as Lemma \ref{Mt}. Thus, we have
\[\small
  \begin{split}
      \mathbb{E}\langle\nabla f\left(\hat{x}_t\right), \hat{\Delta_t}\rangle \leq \frac{\beta\alpha_t}{\sqrt{\theta_t}\alpha_{t-1}}M_{t-1} + C_2\left(1-\theta_t\right) - \frac{1}{N}\sum_{i=1}^N \frac{1-\beta}{2}\mathbb{E}[\|\nabla f\left(\hat{x}_t\right)\|^2_{\hat{\eta}_t^{\left(i\right)}}].
  \end{split}
\]
Let
$\small N_t \!=\! \frac{L\left(2-\delta_g\right)}{N^2}\sum\limits_{i=1}^N\sum\limits_{j=1}^N\mathbb{E} \|\Delta_{t}^{\left(i\right)}\|\|\Delta_{t}^{\left(j\right)}\|
\!+\!\frac{L\left(2-\delta_g\right)}{N^2}\sum\limits_{i=1}^N\sum\limits_{j=1}^N\mathbb{E}\|\Delta_{t}^{\left(i\right)}\|\|e_{t}^{\left(j\right)}\|\!+\!\frac{2L}{N}\sum\limits_{i=1}^N\mathbb{E}\|\Delta_t^{\left(i\right)}\| \!+\!C_2\left(1\!-\!\theta_t\right)$.
Similarly, we can acquire
\[\small
\begin{split}
\sum_{t=1}^T M_t
&\leq \frac{1}{\sqrt{C_1}\left(1-\sqrt{\gamma}\right)}\sum_{t=1}^TN_t - \frac{1-\beta}{2N}\mathbb{E}[\sum_{t=1}^T\sum_{i=1}^N\|\nabla f\left(\hat{x}_t\right)\|^2_{\hat{\eta}_t^{\left(i\right)}}]\\
%& = \frac{1}{\sqrt{C_1}\left(1-\sqrt{\gamma}\right)}\left\{ \frac{1}{N^2}L\left(2-\delta_g\right)\sum_{i=1}^N\sum_{j=1}^N\mathbb{E} \|\Delta_{t}^{\left(i\right)}\|\|\Delta_{t}^{\left(j\right)}\|+\frac{1}{N^2}L\left(2-\delta_g\right)\sum_{i=1}^N\sum_{j=1}^N\mathbb{E}\|\Delta_{t}^{\left(i\right)}\|\|e_{t}^{\left(j\right)}\|\right.\\
%&\qquad \left. +\frac{1}{N}2\delta_xL\sum_{i=1}^N\mathbb{E}\|\Delta_t^{\left(i\right)}\| +C_2\left(1-\theta_t\right)\right\} - \frac{1-\beta}{2N}\sum_{i=1}^N\mathbb{E}[\sum_{t=1}^T\|\nabla f\left(\hat{x}_t\right)\|^2_{\hat{\eta}_t^{\left(i\right)}}]\\
%&\leq \frac{1}{\sqrt{C_1}\left(1-\sqrt{\gamma}\right)}\left\{\sum_{t=1}^T\frac{1}{N}L\left(2-\delta_g\right)\sum_{i=1}^N \mathbb{E} \|\Delta_{t}^{\left(i\right)}\|^2+\frac{1-\delta_g}{\delta_gN}L\left(2-\delta_g\right)\sum_{i=1}^N\mathbb{E}\|\Delta_{t}^{\left(i\right)}\|^2\right.\\
%&\qquad\left. \frac{2\delta_xLD}{N}\sum_{i=1}^N\|\Delta_t^{\left(i\right)}\|+C_2\left(1-\theta_t\right)\right\}  - \frac{1-\beta}{2N}\sum_{i=1}^N\mathbb{E}[\sum_{t=1}^T\|\nabla f\left(\hat{x}_t\right)\|^2_{\hat{\eta}_t^{\left(i\right)}}]\\
&\small \!\leq\! \frac{1}{\sqrt{C_1}\left(1\!-\!\sqrt{\gamma}\right)}\left(\!\left(\!\frac{L\left(2\!-\!\delta_g\right)G^2\alpha^2}{\epsilon\delta_g}\!+\!C_2\theta\!\right)\!\sum_{t=1}^T\frac{1}{t}\!+\!\frac{4\delta_xLG\alpha\sqrt{T}}{\sqrt{\epsilon}}\!\right)\!\!-\! \frac{1\!-\!\beta}{2N}\sum_{i=1}^N\mathbb{E}\!\left[\sum_{t=1}^T\|\nabla f\left(\hat{x}_t\right)\|^2_{\hat{\eta}_t^{\left(i\right)}}\right].
\end{split}
\]
Hence, the proof is completed.
\end{proof}

\begin{proof}[Proof of Theorem \ref{T3}]
By using the gradient Lipschitz continuity of $f$, it holds that
\[\small
\begin{split}
&\mathbb{E} f\left(\tilde{x}_{t+1}\right) \leq \mathbb{E} \left(f\left(\tilde{x}_t\right) + \langle\nabla f\left(\tilde{x}_t\right),\hat{\Delta}_t\rangle + \frac{L}{2}\|\hat{\Delta}_t\|^2\right)\\
& = \mathbb{E} \left[ f\left(\tilde{x}_t\right) + \langle\nabla f\left(\hat{x}_t\right),\hat{\Delta}_t\rangle + \frac{L}{2}\|\hat{\Delta}_t\|^2 + \langle\nabla f\left(x_t\right) - \nabla f\left(\hat{x}_t\right),\hat{\Delta}_t\rangle + \langle \nabla f\left(\tilde{x}_t\right) - \nabla f\left(x_t\right),\hat{\Delta}_t\rangle \right]\\
%&\leq \mathbb{E} \left[f\left(\tilde{x}_t\right) + \frac{1}{N}\sum_{i=1}^N \langle \nabla f\left(\hat{x}_t\right),\Delta_t^{\left(i\right)}\rangle +\frac{L}{2N^2}\sum_{i=1}^N\sum_{j=1}^N \|\Delta_t^{\left(i\right)}\|\|\Delta_t^{\left(j\right)}\|\right] \\
%&\qquad + \mathbb{E}\left[\frac{1}{N}\delta_xL\sum_{i=1}^N\|\Delta_t^{\left(i\right)}\| + \frac{L}{N^2}\sum_{i=1}^N\|e_t^{\left(i\right)}\|\|\Delta_t^{\left(j\right)}\|\right]\\
&\leq \mathbb{E} f\left(\tilde{x}_t\right)+ M_t.
\end{split}
\]
Taking summation over both sides of the above inequality, it holds that
\[\small
\begin{split}
    &f^* \leq \mathbb{E}f\left(\tilde{x}_{T+1}\right) \leq f\left(x_1\right)+ \sum_{t=1}^T M_t \\
    &\leq \frac{1}{\sqrt{C_1}\left(1\!-\!\sqrt{\gamma}\right)}\left(\left(\frac{L\left(2\!-\!\delta_g\right)G^2\alpha^2}{\epsilon\delta_g}\!+\!C_2\theta\right)\sum_{t=1}^T\frac{1}{t}\!+\!\frac{4\delta_xLG\alpha\sqrt{T}}{\sqrt{\epsilon}}\right)\!-\! \frac{1\!-\!\beta}{2N}\sum_{i=1}^N\mathbb{E}\left[\sum_{t=1}^T\|\nabla\! f\left(\hat{x}_t\right)\|^2_{\hat{\eta}_t^{\left(i\right)}}\right].
\end{split}
\]
Arranging the terms, it holds that
\[\small
\begin{split}
&\mathbb{E}[\|\nabla f\left(\hat{x}_\tau\right)\|^2] \leq \frac{\sqrt{G^2+\epsilon d}}{\alpha\sqrt{T}N}\sum_{i=1}^N\mathbb{E}\left[\sum_{t=1}^T\|\nabla f\left(\left(\hat{x}_t\right)\right)\|_{\hat{\eta_t}^{\left(i\right)}}^2\right],\\
&\leq \frac{2\sqrt{G^2\!+\!\epsilon d}}{\left(1\!-\!\beta\right)\alpha\sqrt{T}}\left\{f\left(x_1\right)\!-\!f^* \!+\! \frac{1}{\sqrt{C_1}\left(1\!-\!\sqrt{\gamma}\right)}\left(\left(\frac{L\left(2\!-\!\delta_g\right)G^2\alpha^2}{\epsilon\delta_g}\!+\!C_2\theta\right)\sum_{t=1}^T\frac{1}{t}\!+\!\frac{4\delta_xLG\alpha\sqrt{T}}{\sqrt{\epsilon}}\right)\right\}\\
& = \frac{C_8+C_9\sum_{t=1}^T\frac{1}{t}}{\sqrt{T}} + C_{10},
\end{split}
\]
where $\small C_8\!=\!\frac{2\sqrt{G^2+\epsilon d}}{\left(1\!-\!\beta\right)\alpha}\left(f\left(x_1\right)\!-\!f^*\right)$,$\small C_9\!=\!\frac{2\sqrt{G^2+\epsilon d}}{\left(1\!-\!\beta\right)\alpha\sqrt{C_1}\left(1\!-\!\sqrt{\gamma}\right)}\left(\frac{L\left(2\!-\!\delta_g\right)G^2\alpha^2}{\epsilon\delta_g}\!+\!C_2\theta\right)$,
$\small C_{10}\!=\!\frac{8\sqrt{G^2+\epsilon d}\delta_xLG}{\sqrt{C_1}\left(1\!-\!\sqrt{\gamma}\right)\sqrt{\epsilon}\left(1\!-\!\beta\right)}$,
respectively.

Hence, the proof is completed.
\end{proof}
%Bellow, we set the base learning rate $\alpha_t$ as a constant. Hence, Theorem 3 reduces to the following corollary.

\begin{proof}[proof of corollary \ref{C3}]
Given iteration $T$, let $\small \alpha_t \!=\! \frac{\alpha}{\sqrt{T}}$ and $\small \theta_t \!=\! 1 \!-\! \frac{\theta}{T}$.
Lemma \ref{multi_Mt} updates to
\[\small
\sum_{t=1}^T M_t \leq  \frac{1}{\sqrt{C_1}\left(1-\sqrt{\gamma}\right)}\left(\frac{L\left(2-\delta_g\right)G^2\alpha^2}{\epsilon\delta_g}+C_2\theta +\frac{2\delta_xLG\alpha\sqrt{T}}{\sqrt{\epsilon}}\right)- \frac{1-\beta}{2N}\sum_{i=1}^N\mathbb{E}\left[\sum_{t=1}^T\|\nabla f\left(\hat{x}_t\right)\|^2_{\hat{\eta}_t^{\left(i\right)}}\right].
\]
Then, we have
\[\small
\begin{split}
&\mathbb{E}[\|\nabla f\left(\hat{x}_\tau\right)\|^2] \leq \frac{\sqrt{G^2+\epsilon d}}{\alpha\sqrt{T}N}\sum_{i=1}^N\mathbb{E}\left[\sum_{t=1}^T\|\nabla f\left(\left(\hat{x}_t\right)\right)\|_{\hat{\eta_t}^{\left(i\right)}}^2\right]%\\
%&\leq \frac{2\sqrt{G^2+\epsilon d}}{\left(1-\beta\right)\alpha\sqrt{T}}\left\{f\left(x_1\right)-f^* +\frac{1}{\sqrt{C_1}\left(1-\sqrt{\gamma}\right)}\left(\frac{L\left(2-\delta_g\right)G^2\alpha^2}{\epsilon\delta_g}+C_2\theta +\frac{2\delta_xLG\alpha\sqrt{T}}{\sqrt{\epsilon}}\right)\right\}\\
 \leq \frac{C_8'}{\sqrt{T}} + C_{10}',
\end{split}
\]
where $\small C_8'=\frac{2\sqrt{G^2+\epsilon d}}{\left(1-\beta\right)\alpha}\left(f\left(x_1\right)-f^* +\frac{1}{\sqrt{C_1}\left(1-\sqrt{\gamma}\right)}\left(\frac{L\left(2-\delta_g\right)G^2\alpha^2}{\epsilon\delta_g}+C_2\theta\right)\right)$,
$\small C_{10}'=\frac{4\sqrt{G^2+\epsilon d}\delta_xLG}{\sqrt{C_1}\left(1-\sqrt{\gamma}\right)\sqrt{\epsilon}\left(1-\beta\right)}$.

Hence, by bounding $\small \frac{C_8'}{\sqrt{T}} + C_{10}' \leq C_{10}' + \xi$, we obtain the desired result.
\end{proof}

\section{Experiments}
\label{experiments}

In this section, we apply the proposed algorithms to train deep neural networks including VGG16 network \cite{simonyan2014very} and ResNet-101 \cite{he2016deep} on CIFAR10 \cite{krizhevsky2009learning}  and CIFAR100 \cite{krizhevsky2009learning} datasets, respectively. Below, we first describe the implementation details and then show the experimental results.

\subsection{Implementation Details}

We evaluate the effectiveness of the proposed algorithms via training ResNet-101 \cite{he2016deep} on the CIFAR100 dataset. The CIFAR100 dataset contains 100 classes. In each class, there are $600$ color images with the size of $32\times 32$, among which 500 images serve as training images and the rest are testing images. We train ResNet-101 with 8 workers, each worker will take a batch gradient with batch size 16, so the batch size for one optimization step is 128. We set the total number of training epochs as $200$. In addition, we add $\ell_2$  regularization with a scalar $5e-4$ as the weight decay term. Input images are all down-scaled to 1/8 of their original sizes after 100 convolutional layers and then fed into a fully-connected layer for the 100-class classifications. The output channel numbers of 1-11 conv layers, 12-23 conv layers, 24-92 conv layers, and 93-100 conv layers are 64, 128, 256, and 512, respectively.
Also, we train VGG16 network \cite{simonyan2014very} on the CIFAR10 dataset. CIFAR10 contains 10 classes of $32\times 32$ images, including $50,000$ training images and $10,000$ testing images. We train VGG with 8 workers and each worker takes 16 samples to estimate gradient for $78,200$ iterations with $\ell_2$ regularization which is the same as the previous regularization term. The VGG16 contains 13 convolutional layers, and 3 fully-connected layers with 4096,4096 and 10 neurons per layer. The output channel numbers of 1-2 conv layers, 3-4 conv layers, 5-7 layers, 8-13 layers are 64, 128, 256, 512, respectively.

\begin{wrapfigure}{r}{0.5\textwidth}
\vspace{-1.2cm}
\begin{minipage}{1\linewidth}
\small
\begin{table}[H]
    \centering
    \begin{tabular}{c|>{ }c|c|c}
        Method&Test Acc&Comm&Size\\
        \hline
         QADAM& $77.94\pm0.44\%$ & 162.9 & 162.9\\
         QADAM& $77.44\pm 0.40 \%$ & 15.27 & 162.9\\
         QADAM& $78.36 \pm 0.37\% $ & 10.18 & 162.9\\
         TernGrad\cite{wen2017terngrad} & $76.69 \pm 0.16 \%$ & 162.9 & 162.9\\
         TernGrad\cite{wen2017terngrad} & $76.62\pm 0.38\%$ & 15.27 & 162.9\\
         TernGrad\cite{wen2017terngrad} & $76.00\pm 0.15\%$ & 10.18 & 162.9\\
         \cite{zheng2019communication} & $76.28 \pm 0.24 \%$ & 162.9 & 162.9\\
         \cite{zheng2019communication} & $76.20 \pm 0.47\%$ & 15.27 & 162.9\\
         \cite{zheng2019communication} & $76.06 \pm 0.38\%$ & 10.18 & 162.9\\

         QADAM& $78.10 \pm 0.33\%$ & 162.9 & 81.44\\
         QADAM& $78.08 \pm 0.45\%$ & 162.9 & 40.72\\
         WQuan& $77.00 \pm 0.57 \%$ & 162.9 & 81.44\\
         WQuan& $76.68 \pm 0.42\%$ & 162.9 & 40.72\\

         QADAM& $78.18 \pm 0.43\%$ & 15.27 & 81.44\\
         QADAM& $78.32 \pm 0.47\%$ & 10.18 & 81.44\\
         QADAM& $78.19 \pm 0.28\%$ & 15.27 & 40.72\\
         QADAM& $78.04 \pm 0.39\%$ & 10.18 & 40.72\\

    \end{tabular}
    \caption{Test Accuracy for CIFAR100. "Comm" stands for communication cost and unit is MB per iteration. "Size" stands for model size and unit is MB. WQuan means quantizing weight after training.}
    \label{table1}
\end{table}
\end{minipage}
\vspace{-0.8cm}
\end{wrapfigure}
Besides, in all experiments we choose $\beta$ as 0.99, $\theta$ as 0.999, and $\epsilon$ as $1e-5$. Since it is natural to choose the exponential decay strategy on learning rate, we choose to reduce $\alpha$ by half every 50 epochs instead of using ${\alpha}/{\sqrt{t}}$ as \cite{reddi2019convergence}. We choose the starting learning rate as 0.001. The value is chosen based on grid search on the set $\{0.01, 0.001, 0.0001\}$ based on the accuracy of the full precision setting.
For gradient quantization, we use function $\small Q_g(g) = \|g\|_\infty \arg\min_{\hat{g}\in \mathcal{G}^d} \|g/\|g\|_\infty-\hat{g}\|$, where $\small \mathcal{G} = \{-1, \cdots, -2^{-k_g}, 0, 2^{-k_g}, 2^{-k_g+1},\cdots, 1\}$. For weight quantization, we use function $\small Q_x(x) = 0.5\times\arg\min_{\hat{x}\in \mathcal{X}} \|2x-\hat{x}\|$, where $\small \mathcal{X} = \left\{-1,\cdots, -\frac{1}{2^{k_x}},0,\frac{1}{2^{k_x}}, \frac{2}{2^{k_x}},\cdots, 1\right\}$.

We compared our method with TernGrad \cite{wen2017terngrad} and Zheng et. al\cite{zheng2019communication} for gradient quantization, where the learning rate in these two methods is 0.1  {which generated by grid search in $\{0.1,0.05,0.01\}$}. For weight quantization, we compare with the result which quantizes the final model directly {named WQuan in tables}.

\subsection{Results Illustration}

 {In this section, we will illustrate the results of training ResNet-101 on the CIFAR100 dataset and training VGG16 on the CIFAR10 dataset, respectively. Two tables show the test accuracy after 200 epochs of training, where the first column represents training methods, the second column represents the test accuracy, the third column represents bits required for gradient communication, and the last column represents the bits to save a model. As for the same method, we can set different $k_x$ and $k_g$, and we can get a different number of bits needed for gradients and weights.}

\subsubsection{Results of Training ResNet-101 on the CIFAR100}

 {
Figure \ref{res1}  and Table \ref{table1} show the result of training ResNet-101 on the CIFAR100 dataset. When the algorithm involves gradient quantization, we compare our algorithm with Zheng et al.\cite{zheng2019communication} and TernGrad \cite{wen2017terngrad} with a different number of communication bits, which has been shown in the left figure of Figure \ref{res1} and the first 9 lines in Table \ref{table1}. It can be shown that even with gradient quantization, our algorithm can outperform TernGrad\cite{wen2017terngrad} and Zheng et al. \cite{zheng2019communication}. The middle figure in Figure \ref{res1} shows the result of using weight quantization only. Row 10-12 in Table \ref{table1} shows the comparison results between quantizing weight during training and after training, which shows when quantizing weight during the training process can achieve higher test accuracy. The right figure in Figure \ref{res1} and the last 4 rows demonstrate the results of combining gradient quantization and weight quantization. It shows even though we shrink model size into 1/4 of its original size and gradient size into 1/16 of its original size, respectively, it can still give comparable results.
}

\medskip

\subsubsection{Results of Training VGG16 on the CIFAR10}

\begin{wrapfigure}{r}{0.5\textwidth}
\vspace{-1.2cm}
\begin{minipage}{1\linewidth}
\small
\begin{table}[H]
    \centering
    \begin{tabular}{c|>{ }c|c|c}
        Method&Test Acc&Comm&Size\\
        \hline
         QADAM& $91.21 \pm 0.24\%$ & 512.3 & 512.3\\
         QADAM& $91.29 \pm 0.26\%$ & 48.03 & 512.3\\
         QADAM& $90.64 \pm 0.18\%$ & 32.02 & 512.3\\
         TernGrad\cite{wen2017terngrad} & $90.60 \pm 0.13$\% & 512.3 & 512.3\\
         TernGrad\cite{wen2017terngrad} & $90.60 \pm 0.32\%$ & 48.03 & 512.3\\
         TernGrad\cite{wen2017terngrad} & $89.91 \pm 0.28\%$ & 32.02 & 512.3\\
         \cite{zheng2019communication} & $91.51 \pm 0.22\%$ & 512.3 & 512.3\\
         \cite{zheng2019communication} & $91.38 \pm 0.36\%$ & 48.03 & 512.3\\
         \cite{zheng2019communication} & $90.93 \pm 0.10\%$ & 32.02 & 512.3\\

         QADAM& $91.06 \pm 0.31 \% $& 512.3 & 256.2\\
         QADAM& $91.21 \pm 0.41\%$ & 512.3 & 128.1\\
         WQuan& $91.21 \pm 0.24\%$ & 512.3 & 256.2\\
         WQuan& $90.18 \pm 0.39\%$ & 512.3 & 128.1\\

         QADAM& $90.87 \pm 0.27\%$ & 48.03 & 256.2\\
         QADAM& $91.20 \pm 0.22\%$ & 32.02 & 256.2\\
         QADAM& $91.16 \pm 0.26\%$ & 48.03 & 128.1\\
         QADAM& $90.48 \pm 0.36\%$ & 32.02 & 128.1\\

    \end{tabular}
    \caption{Test Accuracy for CIFAR10. "Comm" stands for communication cost and unit is MB per iteration. "Size" stands for model size and unit is MB.WQuan means quantizing weight with the first model.}
    \label{table2}
\end{table}
\end{minipage}
\vspace{-0.3cm}
\end{wrapfigure}
Figure \ref{res2} and Table \ref{table2} show the result of training VGG16 on the CIFAR10 dataset. The left figure in Figure \ref{res2} and the first 9 rows in Table ref{table2} show results of gradient quantization comparison among our algorithm, TernGrad\cite{wen2017terngrad} and Zheng et al.\cite{zheng2019communication}. Our results and Zheng et al.\cite{zheng2019communication} can achieve similar performance, but TernGrad \cite{wen2017terngrad} gets worse performance due to noise to ensure unbiasedness. When considering weight quantization, shown in the middle figure of Figure \ref{res2} and row 10-13 in Table \ref{table2}, performance is similar between quantizing during training or after training. Further, as it has shown in the last 4 rows with different sizes of the model and different gradient quantization, our method can still achieve high accuracy compared to the full precision version (first row in Table \ref{table2}).

\begin{figure}[htb!]
\small
\begin{center}
\begin{tabular}{ccc}
\!\!\!\!\!\!\includegraphics[width=4.8cm]{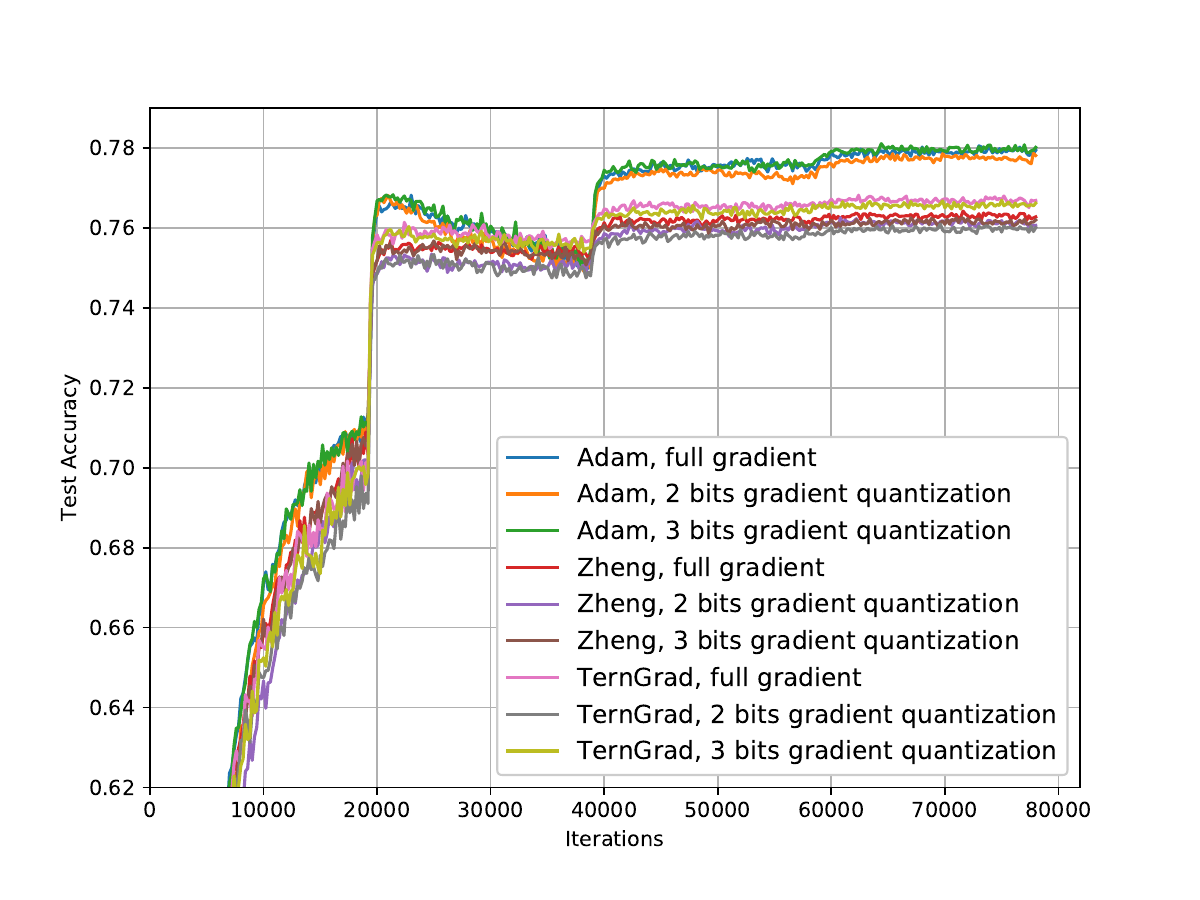} \!\!\!\!\!\!\!\!\!\!\!\!\!\!\!\!
& \includegraphics[width=4.8cm]{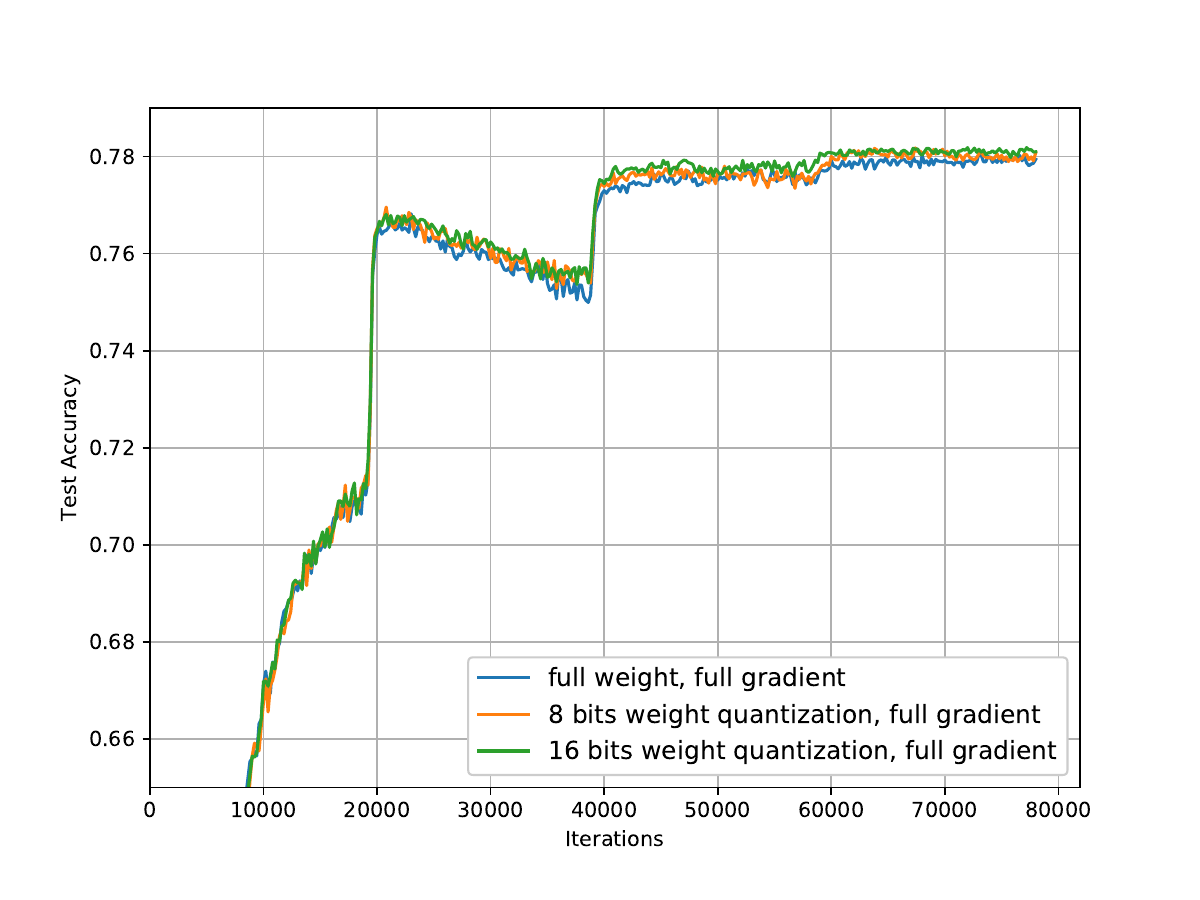} \!\!\!\!\!\!\!\!\!\!\!\!\!\!\!\!
&\includegraphics[width=4.8cm]{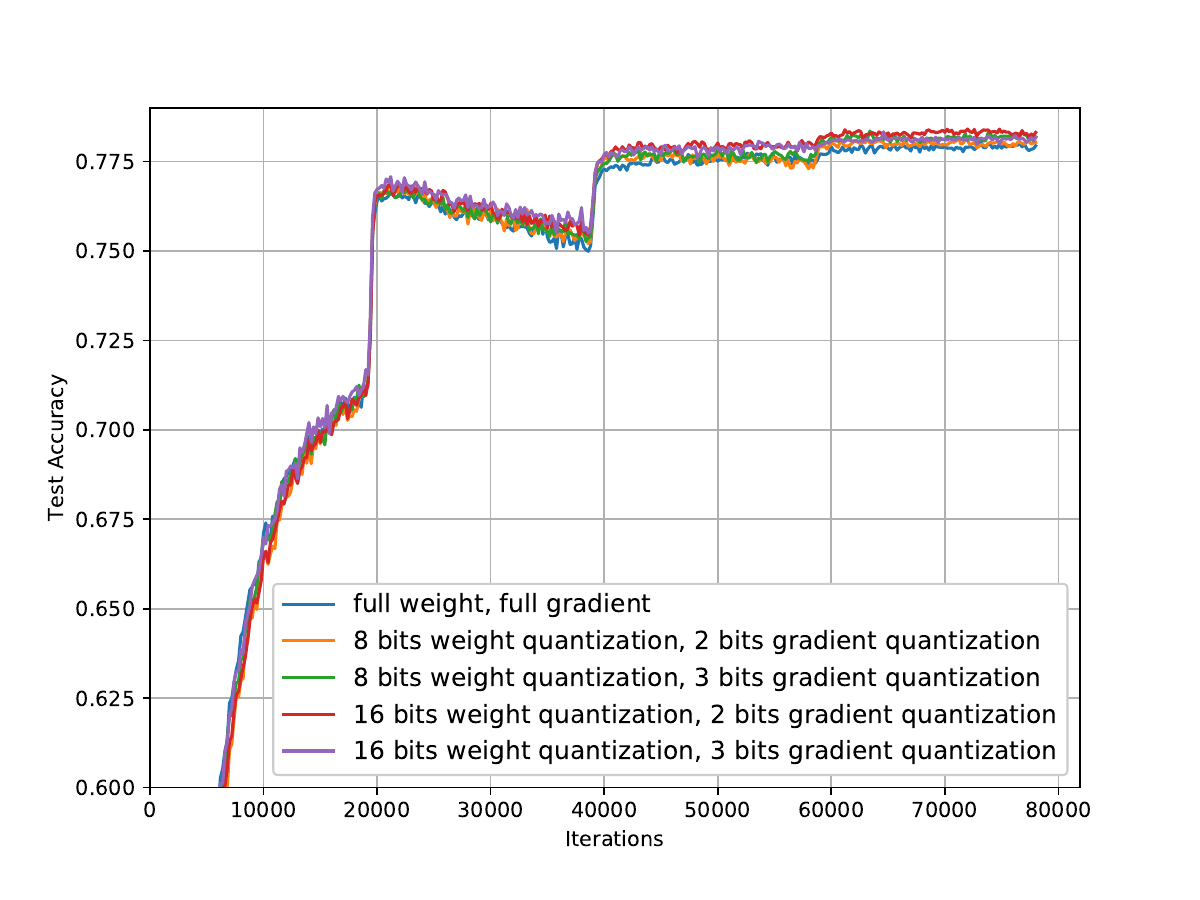}
\end{tabular}
\vspace{-0.5cm}
\caption{ {Results for Training ResNet-101 on CIFAR100.}}
\label{res1}
\end{center}
\vspace{-0.4cm}
\begin{center}
\begin{tabular}{ccc}
\!\!\!\!\!\!\includegraphics[width=4.8cm]{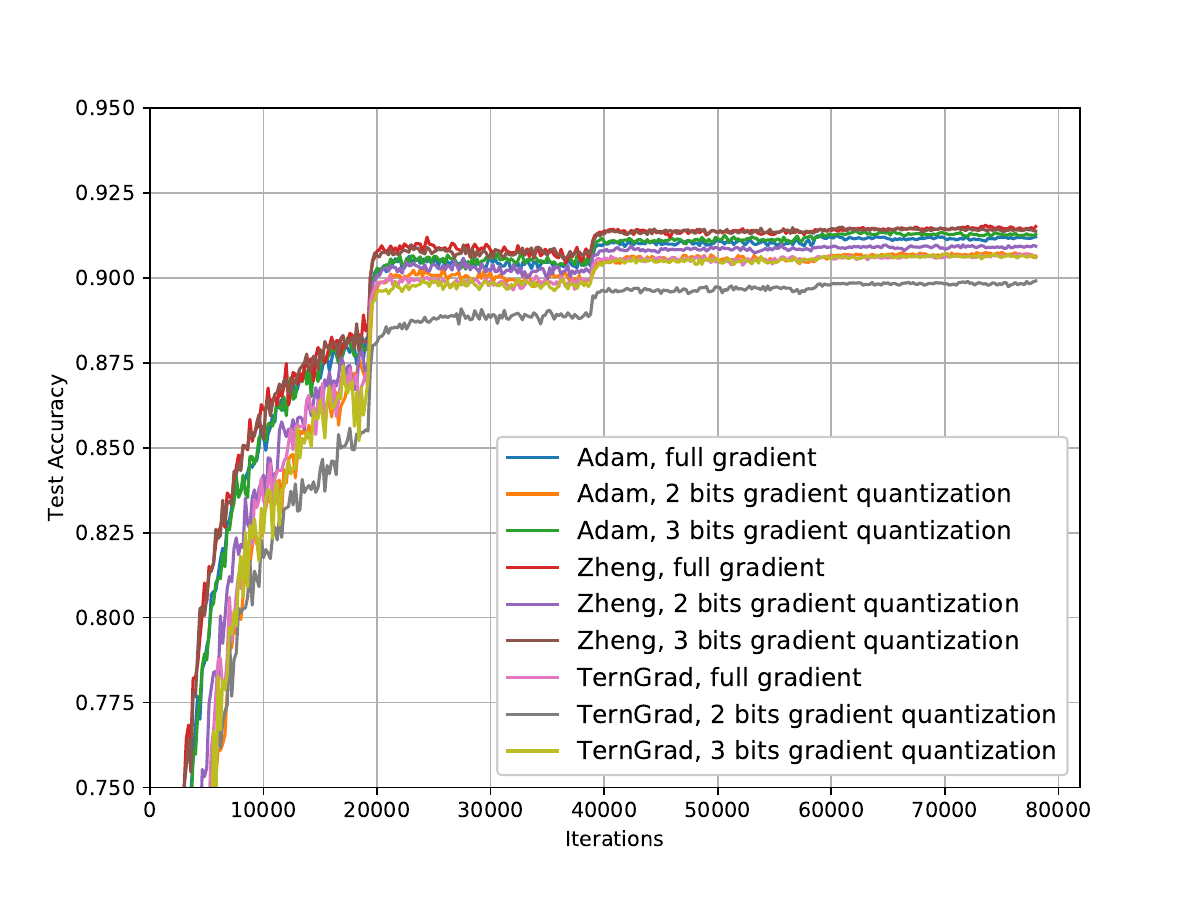}\!\!\!\!\!\!\!\!\!\!\!
& \includegraphics[width=4.8cm]{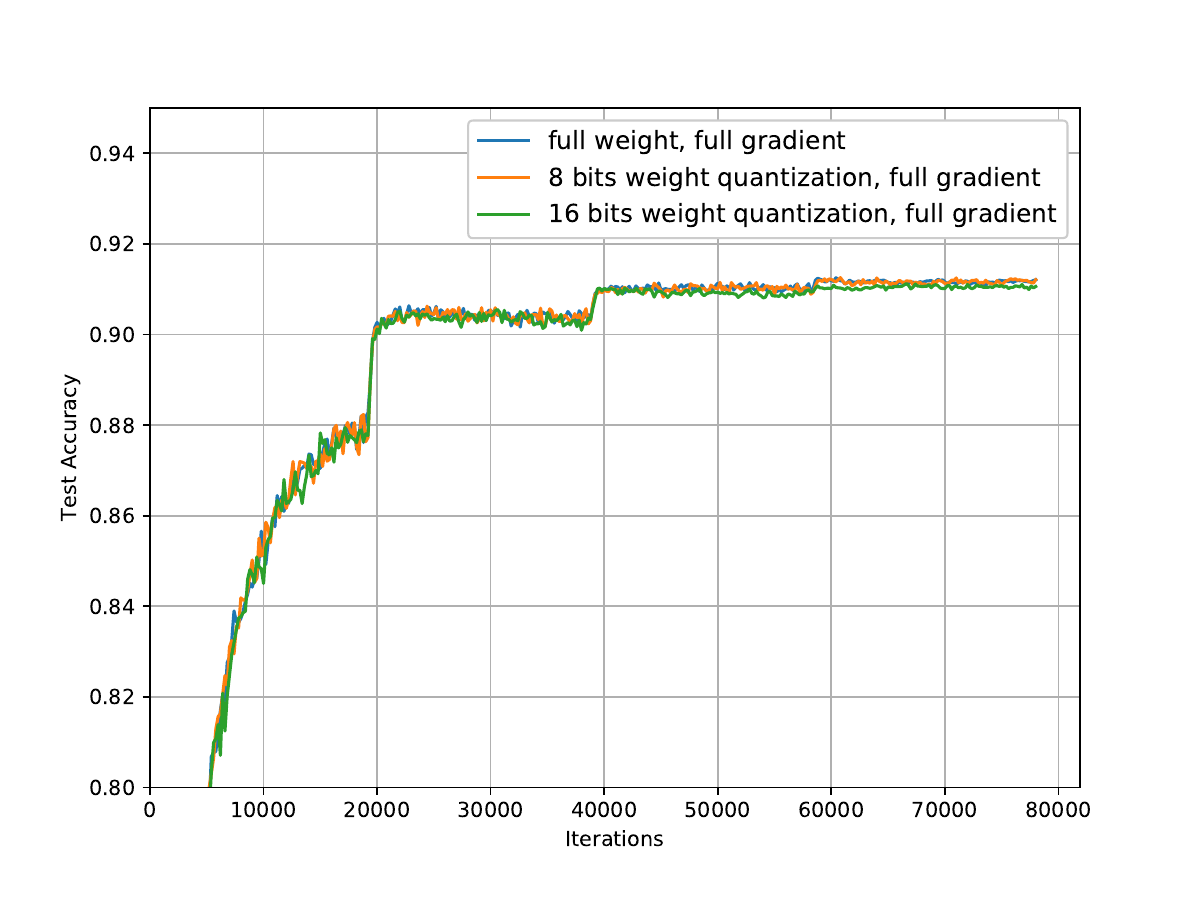} \!\!\!\!\!\!\!\!\!\!\!\!
& \includegraphics[width=4.8cm]{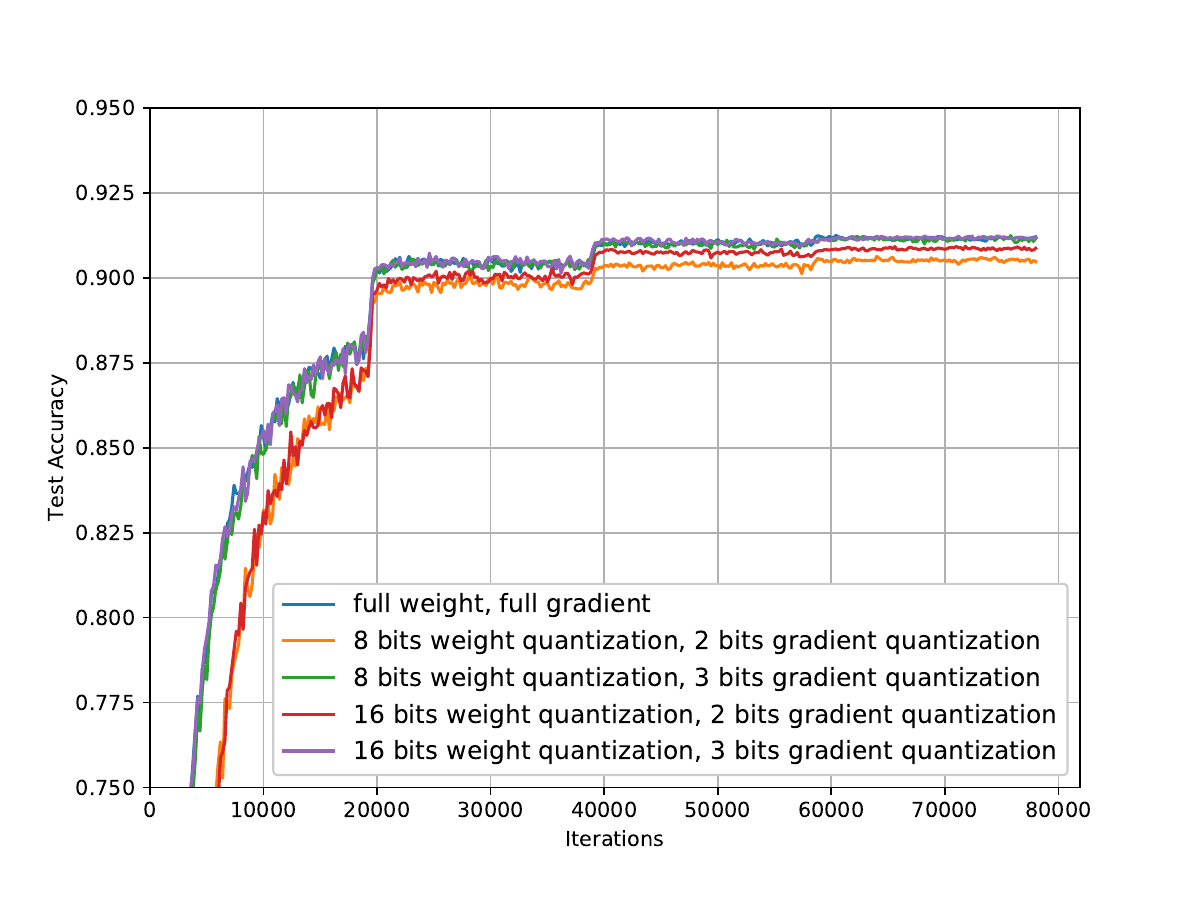}\\
\end{tabular}
\vspace{-0.5cm}
\caption{ {Results for Training VGG16 on CIFAR10.}}
\label{res2}
\end{center}
\vspace{-0.5cm}
\end{figure}

\section{Conclusions}

To accelerate the training process of deep learning models, we proposed distributed Adam with weight quantization, gradient quantization, and the error-feedback technique in the parameter-server model. Through capitalizing on the two schemes of weight quantization and gradient quantization, the communication cost between the server and works can be significantly alleviated. In addition, the proposed error-feedback technique can suppress the bias caused by the gradient quantization step, thereby making the proposed algorithms more efficient. We further established the convergence rates of the proposed algorithms in the nonconvex stochastic setting and showed that quantized Adam with the error-feedback technique converges to the neighborhood of a stationary point under both the single-worker and multi-worker modes. Moreover, we applied the proposed algorithms to train VGG16 on the CIFAR10 dataset and ResNet-101 on the CIFAR100 dataset, respectively. The experiments demonstrate the efficacy of the proposed algorithms.

\bibliographystyle{ACM-Reference-Format}
\bibliography{sample-base}


\begin{thebibliography}{48}

%%% the Web address.

\ifx \showCODEN    \undefined \def \showCODEN     #1{\unskip}     \fi
\ifx \showDOI      \undefined \def \showDOI       #1{#1}\fi
\ifx \showISBNx    \undefined \def \showISBNx     #1{\unskip}     \fi
\ifx \showISBNxiii \undefined \def \showISBNxiii  #1{\unskip}     \fi
\ifx \showISSN     \undefined \def \showISSN      #1{\unskip}     \fi
\ifx \showLCCN     \undefined \def \showLCCN      #1{\unskip}     \fi
\ifx \shownote     \undefined \def \shownote      #1{#1}          \fi
\ifx \showarticletitle \undefined \def \showarticletitle #1{#1}   \fi
\ifx \showURL      \undefined \def \showURL       {\relax}        \fi
% invisible to TeX
\providecommand\bibfield[2]{#2}
\providecommand\bibinfo[2]{#2}
\providecommand\natexlab[1]{#1}
\providecommand\showeprint[2][]{arXiv:#2}

\bibitem[\protect\citeauthoryear{Amodei, Ananthanarayanan, Anubhai, Bai,
  Battenberg, Case, Casper, Catanzaro, Cheng, Chen, et~al\mbox{.}}{Amodei
  et~al\mbox{.}}{2016}]%
        {amodei2016deep}
\bibfield{author}{\bibinfo{person}{Dario Amodei}, \bibinfo{person}{Sundaram
  Ananthanarayanan}, \bibinfo{person}{Rishita Anubhai},
  \bibinfo{person}{Jingliang Bai}, \bibinfo{person}{Eric Battenberg},
  \bibinfo{person}{Carl Case}, \bibinfo{person}{Jared Casper},
  \bibinfo{person}{Bryan Catanzaro}, \bibinfo{person}{Qiang Cheng},
  \bibinfo{person}{Guoliang Chen}, {et~al\mbox{.}}}
  \bibinfo{year}{2016}\natexlab{}.
\newblock \showarticletitle{Deep speech 2: End-to-end speech recognition in
  english and mandarin}. In \bibinfo{booktitle}{\emph{International conference
  on machine learning}}. \bibinfo{pages}{173--182}.
\newblock


\bibitem[\protect\citeauthoryear{Basu, De, Mukherjee, and Ullah}{Basu
  et~al\mbox{.}}{2018}]%
        {basu2018convergence}
\bibfield{author}{\bibinfo{person}{Amitabh Basu}, \bibinfo{person}{Soham De},
  \bibinfo{person}{Anirbit Mukherjee}, {and} \bibinfo{person}{Enayat Ullah}.}
  \bibinfo{year}{2018}\natexlab{}.
\newblock \showarticletitle{Convergence guarantees for rmsprop and adam in
  non-convex optimization and their comparison to nesterov acceleration on
  autoencoders}.
\newblock \bibinfo{journal}{\emph{arXiv preprint arXiv:1807.06766}}
  (\bibinfo{year}{2018}).
\newblock


\bibitem[\protect\citeauthoryear{Chen, Liu, Sun, and Hong}{Chen
  et~al\mbox{.}}{2018}]%
        {chen2018convergence}
\bibfield{author}{\bibinfo{person}{Xiangyi Chen}, \bibinfo{person}{Sijia Liu},
  \bibinfo{person}{Ruoyu Sun}, {and} \bibinfo{person}{Mingyi Hong}.}
  \bibinfo{year}{2018}\natexlab{}.
\newblock \showarticletitle{On the convergence of a class of adam-type
  algorithms for non-convex optimization}.
\newblock \bibinfo{journal}{\emph{arXiv preprint arXiv:1808.02941}}
  (\bibinfo{year}{2018}).
\newblock


\bibitem[\protect\citeauthoryear{Dean, Corrado, Monga, Chen, Devin, Mao,
  Ranzato, Senior, Tucker, Yang, et~al\mbox{.}}{Dean et~al\mbox{.}}{2012}]%
        {dean2012large}
\bibfield{author}{\bibinfo{person}{Jeffrey Dean}, \bibinfo{person}{Greg
  Corrado}, \bibinfo{person}{Rajat Monga}, \bibinfo{person}{Kai Chen},
  \bibinfo{person}{Matthieu Devin}, \bibinfo{person}{Mark Mao},
  \bibinfo{person}{Marc'aurelio Ranzato}, \bibinfo{person}{Andrew Senior},
  \bibinfo{person}{Paul Tucker}, \bibinfo{person}{Ke Yang}, {et~al\mbox{.}}}
  \bibinfo{year}{2012}\natexlab{}.
\newblock \showarticletitle{Large scale distributed deep networks}. In
  \bibinfo{booktitle}{\emph{Advances in neural information processing
  systems}}. \bibinfo{pages}{1223--1231}.
\newblock


\bibitem[\protect\citeauthoryear{Deng, Dong, Socher, Li, Li, and Fei-Fei}{Deng
  et~al\mbox{.}}{2009}]%
        {deng2009imagenet}
\bibfield{author}{\bibinfo{person}{Jia Deng}, \bibinfo{person}{Wei Dong},
  \bibinfo{person}{Richard Socher}, \bibinfo{person}{Li-Jia Li},
  \bibinfo{person}{Kai Li}, {and} \bibinfo{person}{Li Fei-Fei}.}
  \bibinfo{year}{2009}\natexlab{}.
\newblock \showarticletitle{Imagenet: A large-scale hierarchical image
  database}. In \bibinfo{booktitle}{\emph{2009 IEEE conference on computer
  vision and pattern recognition}}. Ieee, \bibinfo{pages}{248--255}.
\newblock


\bibitem[\protect\citeauthoryear{Devlin, Chang, Lee, and Toutanova}{Devlin
  et~al\mbox{.}}{2018}]%
        {devlin2018bert}
\bibfield{author}{\bibinfo{person}{Jacob Devlin}, \bibinfo{person}{Ming-Wei
  Chang}, \bibinfo{person}{Kenton Lee}, {and} \bibinfo{person}{Kristina
  Toutanova}.} \bibinfo{year}{2018}\natexlab{}.
\newblock \showarticletitle{Bert: Pre-training of deep bidirectional
  transformers for language understanding}.
\newblock \bibinfo{journal}{\emph{arXiv preprint arXiv:1810.04805}}
  (\bibinfo{year}{2018}).
\newblock


\bibitem[\protect\citeauthoryear{Duchi, Hazan, and Singer}{Duchi
  et~al\mbox{.}}{2011}]%
        {duchi2011adaptive}
\bibfield{author}{\bibinfo{person}{John Duchi}, \bibinfo{person}{Elad Hazan},
  {and} \bibinfo{person}{Yoram Singer}.} \bibinfo{year}{2011}\natexlab{}.
\newblock \showarticletitle{Adaptive subgradient methods for online learning
  and stochastic optimization}.
\newblock \bibinfo{journal}{\emph{Journal of machine learning research}}
  \bibinfo{volume}{12}, \bibinfo{number}{Jul} (\bibinfo{year}{2011}),
  \bibinfo{pages}{2121--2159}.
\newblock


\bibitem[\protect\citeauthoryear{Goodfellow, Bengio, and Courville}{Goodfellow
  et~al\mbox{.}}{2016}]%
        {goodfellow2016deep}
\bibfield{author}{\bibinfo{person}{Ian Goodfellow}, \bibinfo{person}{Yoshua
  Bengio}, {and} \bibinfo{person}{Aaron Courville}.}
  \bibinfo{year}{2016}\natexlab{}.
\newblock \bibinfo{booktitle}{\emph{Deep learning}}.
\newblock \bibinfo{publisher}{MIT press}.
\newblock


\bibitem[\protect\citeauthoryear{Han, Mao, and Dally}{Han
  et~al\mbox{.}}{2015}]%
        {han2015deep}
\bibfield{author}{\bibinfo{person}{Song Han}, \bibinfo{person}{Huizi Mao},
  {and} \bibinfo{person}{William~J Dally}.} \bibinfo{year}{2015}\natexlab{}.
\newblock \showarticletitle{Deep compression: Compressing deep neural networks
  with pruning, trained quantization and huffman coding}.
\newblock \bibinfo{journal}{\emph{arXiv preprint arXiv:1510.00149}}
  (\bibinfo{year}{2015}).
\newblock


\bibitem[\protect\citeauthoryear{He, Zhang, Ren, and Sun}{He
  et~al\mbox{.}}{2016}]%
        {he2016deep}
\bibfield{author}{\bibinfo{person}{Kaiming He}, \bibinfo{person}{Xiangyu
  Zhang}, \bibinfo{person}{Shaoqing Ren}, {and} \bibinfo{person}{Jian Sun}.}
  \bibinfo{year}{2016}\natexlab{}.
\newblock \showarticletitle{Deep residual learning for image recognition}. In
  \bibinfo{booktitle}{\emph{Proceedings of the IEEE conference on computer
  vision and pattern recognition}}. \bibinfo{pages}{770--778}.
\newblock


\bibitem[\protect\citeauthoryear{Hinton, Srivastava, and Swersky}{Hinton
  et~al\mbox{.}}{2012}]%
        {hinton2012neural}
\bibfield{author}{\bibinfo{person}{Geoffrey Hinton}, \bibinfo{person}{Nitish
  Srivastava}, {and} \bibinfo{person}{Kevin Swersky}.}
  \bibinfo{year}{2012}\natexlab{}.
\newblock \showarticletitle{Neural networks for machine learning lecture 6a
  overview of mini-batch gradient descent}.
\newblock  (\bibinfo{year}{2012}).
\newblock


\bibitem[\protect\citeauthoryear{Hou, Zhang, and Kwok}{Hou
  et~al\mbox{.}}{2018}]%
        {hou2018analysis}
\bibfield{author}{\bibinfo{person}{Lu Hou}, \bibinfo{person}{Ruiliang Zhang},
  {and} \bibinfo{person}{James~T Kwok}.} \bibinfo{year}{2018}\natexlab{}.
\newblock \showarticletitle{Analysis of quantized models}.
\newblock  (\bibinfo{year}{2018}).
\newblock


\bibitem[\protect\citeauthoryear{Jiang and Agrawal}{Jiang and Agrawal}{2018}]%
        {jiang2018linear}
\bibfield{author}{\bibinfo{person}{Peng Jiang} {and} \bibinfo{person}{Gagan
  Agrawal}.} \bibinfo{year}{2018}\natexlab{}.
\newblock \showarticletitle{A linear speedup analysis of distributed deep
  learning with sparse and quantized communication}. In
  \bibinfo{booktitle}{\emph{Proceedings of the 32nd International Conference on
  Neural Information Processing Systems}}. \bibinfo{pages}{2530--2541}.
\newblock


\bibitem[\protect\citeauthoryear{Karimireddy, Rebjock, Stich, and
  Jaggi}{Karimireddy et~al\mbox{.}}{2019}]%
        {KarimireddyError}
\bibfield{author}{\bibinfo{person}{Sai~Praneeth Karimireddy},
  \bibinfo{person}{Quentin Rebjock}, \bibinfo{person}{Sebastian Stich}, {and}
  \bibinfo{person}{Martin Jaggi}.} \bibinfo{year}{2019}\natexlab{}.
\newblock \showarticletitle{Error feedback fixes signsgd and other gradient
  compression schemes}. In \bibinfo{booktitle}{\emph{International Conference
  on Machine Learning}}. PMLR, \bibinfo{pages}{3252--3261}.
\newblock


\bibitem[\protect\citeauthoryear{Khaled and Richt{\'a}rik}{Khaled and
  Richt{\'a}rik}{2019}]%
        {khaled2019gradient}
\bibfield{author}{\bibinfo{person}{Ahmed Khaled} {and} \bibinfo{person}{Peter
  Richt{\'a}rik}.} \bibinfo{year}{2019}\natexlab{}.
\newblock \showarticletitle{Gradient descent with compressed iterates}.
\newblock \bibinfo{journal}{\emph{arXiv preprint arXiv:1909.04716}}
  (\bibinfo{year}{2019}).
\newblock


\bibitem[\protect\citeauthoryear{Kingma and Ba}{Kingma and Ba}{2014}]%
        {kingma2014adam}
\bibfield{author}{\bibinfo{person}{Diederik~P Kingma} {and}
  \bibinfo{person}{Jimmy Ba}.} \bibinfo{year}{2014}\natexlab{}.
\newblock \showarticletitle{Adam: A method for stochastic optimization}.
\newblock \bibinfo{journal}{\emph{arXiv preprint arXiv:1412.6980}}
  (\bibinfo{year}{2014}).
\newblock


\bibitem[\protect\citeauthoryear{Koloskova, Lin, Stich, and Jaggi}{Koloskova
  et~al\mbox{.}}{2019}]%
        {koloskova2019decentralized}
\bibfield{author}{\bibinfo{person}{Anastasia Koloskova}, \bibinfo{person}{Tao
  Lin}, \bibinfo{person}{Sebastian~U Stich}, {and} \bibinfo{person}{Martin
  Jaggi}.} \bibinfo{year}{2019}\natexlab{}.
\newblock \showarticletitle{Decentralized Deep Learning with Arbitrary
  Communication Compression}. In \bibinfo{booktitle}{\emph{International
  Conference on Learning Representations}}.
\newblock


\bibitem[\protect\citeauthoryear{Kraska, Talwalkar, Duchi, Griffith, Franklin,
  and Jordan}{Kraska et~al\mbox{.}}{2013}]%
        {kraska2013mlbase}
\bibfield{author}{\bibinfo{person}{Tim Kraska}, \bibinfo{person}{Ameet~S.
  Talwalkar}, \bibinfo{person}{John~C. Duchi}, \bibinfo{person}{R. Griffith},
  \bibinfo{person}{M. Franklin}, {and} \bibinfo{person}{Michael~I. Jordan}.}
  \bibinfo{year}{2013}\natexlab{}.
\newblock \showarticletitle{MLbase: A Distributed Machine-learning System}. In
  \bibinfo{booktitle}{\emph{CIDR}}.
\newblock


\bibitem[\protect\citeauthoryear{Krizhevsky, Hinton, et~al\mbox{.}}{Krizhevsky
  et~al\mbox{.}}{2009}]%
        {krizhevsky2009learning}
\bibfield{author}{\bibinfo{person}{Alex Krizhevsky}, \bibinfo{person}{Geoffrey
  Hinton}, {et~al\mbox{.}}} \bibinfo{year}{2009}\natexlab{}.
\newblock \showarticletitle{Learning multiple layers of features from tiny
  images}.
\newblock  (\bibinfo{year}{2009}).
\newblock


\bibitem[\protect\citeauthoryear{Krizhevsky, Sutskever, and Hinton}{Krizhevsky
  et~al\mbox{.}}{2012}]%
        {krizhevsky2012imagenet}
\bibfield{author}{\bibinfo{person}{Alex Krizhevsky}, \bibinfo{person}{Ilya
  Sutskever}, {and} \bibinfo{person}{Geoffrey~E Hinton}.}
  \bibinfo{year}{2012}\natexlab{}.
\newblock \showarticletitle{Imagenet classification with deep convolutional
  neural networks}. In \bibinfo{booktitle}{\emph{Advances in neural information
  processing systems}}. \bibinfo{pages}{1097--1105}.
\newblock


\bibitem[\protect\citeauthoryear{LeCun, Bengio, and Hinton}{LeCun
  et~al\mbox{.}}{2015}]%
        {lecun2015deep}
\bibfield{author}{\bibinfo{person}{Yann LeCun}, \bibinfo{person}{Yoshua
  Bengio}, {and} \bibinfo{person}{Geoffrey Hinton}.}
  \bibinfo{year}{2015}\natexlab{}.
\newblock \showarticletitle{Deep learning}.
\newblock \bibinfo{journal}{\emph{nature}} \bibinfo{volume}{521},
  \bibinfo{number}{7553} (\bibinfo{year}{2015}), \bibinfo{pages}{436--444}.
\newblock


\bibitem[\protect\citeauthoryear{LeCun, Bottou, Bengio, and Haffner}{LeCun
  et~al\mbox{.}}{1998}]%
        {lecun1998gradient}
\bibfield{author}{\bibinfo{person}{Yann LeCun}, \bibinfo{person}{L{\'e}on
  Bottou}, \bibinfo{person}{Yoshua Bengio}, {and} \bibinfo{person}{Patrick
  Haffner}.} \bibinfo{year}{1998}\natexlab{}.
\newblock \showarticletitle{Gradient-based learning applied to document
  recognition}.
\newblock \bibinfo{journal}{\emph{Proc. IEEE}} \bibinfo{volume}{86},
  \bibinfo{number}{11} (\bibinfo{year}{1998}), \bibinfo{pages}{2278--2324}.
\newblock


\bibitem[\protect\citeauthoryear{Li, Andersen, Park, Smola, Ahmed, Josifovski,
  Long, Shekita, and Su}{Li et~al\mbox{.}}{2014}]%
        {li2014scaling}
\bibfield{author}{\bibinfo{person}{Mu Li}, \bibinfo{person}{David~G Andersen},
  \bibinfo{person}{Jun~Woo Park}, \bibinfo{person}{Alexander~J Smola},
  \bibinfo{person}{Amr Ahmed}, \bibinfo{person}{Vanja Josifovski},
  \bibinfo{person}{James Long}, \bibinfo{person}{Eugene~J Shekita}, {and}
  \bibinfo{person}{Bor-Yiing Su}.} \bibinfo{year}{2014}\natexlab{}.
\newblock \showarticletitle{Scaling distributed machine learning with the
  parameter server}. In \bibinfo{booktitle}{\emph{11th $\{$USENIX$\}$ Symposium
  on Operating Systems Design and Implementation ($\{$OSDI$\}$ 14)}}.
  \bibinfo{pages}{583--598}.
\newblock


\bibitem[\protect\citeauthoryear{Li and Orabona}{Li and Orabona}{2019}]%
        {li2018convergence}
\bibfield{author}{\bibinfo{person}{Xiaoyu Li} {and} \bibinfo{person}{Francesco
  Orabona}.} \bibinfo{year}{2019}\natexlab{}.
\newblock \showarticletitle{On the convergence of stochastic gradient descent
  with adaptive stepsizes}. In \bibinfo{booktitle}{\emph{The 22nd International
  Conference on Artificial Intelligence and Statistics}}. PMLR,
  \bibinfo{pages}{983--992}.
\newblock


\bibitem[\protect\citeauthoryear{Lian, Zhang, Zhang, Hsieh, Zhang, and
  Liu}{Lian et~al\mbox{.}}{2017}]%
        {lian2017can}
\bibfield{author}{\bibinfo{person}{Xiangru Lian}, \bibinfo{person}{Ce Zhang},
  \bibinfo{person}{Huan Zhang}, \bibinfo{person}{Cho-Jui Hsieh},
  \bibinfo{person}{Wei Zhang}, {and} \bibinfo{person}{Ji Liu}.}
  \bibinfo{year}{2017}\natexlab{}.
\newblock \showarticletitle{Can decentralized algorithms outperform centralized
  algorithms? a case study for decentralized parallel stochastic gradient
  descent}. In \bibinfo{booktitle}{\emph{Advances in Neural Information
  Processing Systems}}. \bibinfo{pages}{5330--5340}.
\newblock


\bibitem[\protect\citeauthoryear{Liu, Chen, and Wang}{Liu
  et~al\mbox{.}}{2017}]%
        {liu2017distributed}
\bibfield{author}{\bibinfo{person}{Tie-Yan Liu}, \bibinfo{person}{Wei Chen},
  {and} \bibinfo{person}{Taifeng Wang}.} \bibinfo{year}{2017}\natexlab{}.
\newblock \showarticletitle{Distributed machine learning: Foundations, trends,
  and practices}. In \bibinfo{booktitle}{\emph{Proceedings of the 26th
  International Conference on World Wide Web Companion}}.
  \bibinfo{pages}{913--915}.
\newblock


\bibitem[\protect\citeauthoryear{McMahan et~al\mbox{.}}{McMahan
  et~al\mbox{.}}{2021}]%
        {kairouz2019advances}
\bibfield{author}{\bibinfo{person}{H~Brendan McMahan} {et~al\mbox{.}}}
  \bibinfo{year}{2021}\natexlab{}.
\newblock \showarticletitle{Advances and open problems in federated learning}.
\newblock \bibinfo{journal}{\emph{Foundations and Trends{\textregistered} in
  Machine Learning}} \bibinfo{volume}{14}, \bibinfo{number}{1}
  (\bibinfo{year}{2021}).
\newblock


\bibitem[\protect\citeauthoryear{McMahan and Streeter}{McMahan and
  Streeter}{2010}]%
        {mcmahan2010adaptive}
\bibfield{author}{\bibinfo{person}{H~Brendan McMahan} {and}
  \bibinfo{person}{Matthew Streeter}.} \bibinfo{year}{2010}\natexlab{}.
\newblock \showarticletitle{Adaptive bound optimization for online convex
  optimization}.
\newblock \bibinfo{journal}{\emph{arXiv preprint arXiv:1002.4908}}
  (\bibinfo{year}{2010}).
\newblock


\bibitem[\protect\citeauthoryear{Mnih, Kavukcuoglu, Silver, Rusu, Veness,
  Bellemare, Graves, Riedmiller, Fidjeland, Ostrovski, et~al\mbox{.}}{Mnih
  et~al\mbox{.}}{2015}]%
        {mnih2015human}
\bibfield{author}{\bibinfo{person}{Volodymyr Mnih}, \bibinfo{person}{Koray
  Kavukcuoglu}, \bibinfo{person}{David Silver}, \bibinfo{person}{Andrei~A
  Rusu}, \bibinfo{person}{Joel Veness}, \bibinfo{person}{Marc~G Bellemare},
  \bibinfo{person}{Alex Graves}, \bibinfo{person}{Martin Riedmiller},
  \bibinfo{person}{Andreas~K Fidjeland}, \bibinfo{person}{Georg Ostrovski},
  {et~al\mbox{.}}} \bibinfo{year}{2015}\natexlab{}.
\newblock \showarticletitle{Human-level control through deep reinforcement
  learning}.
\newblock \bibinfo{journal}{\emph{Nature}} \bibinfo{volume}{518},
  \bibinfo{number}{7540} (\bibinfo{year}{2015}), \bibinfo{pages}{529--533}.
\newblock


\bibitem[\protect\citeauthoryear{Nazari, Tarzanagh, and Michailidis}{Nazari
  et~al\mbox{.}}{2019}]%
        {nazari2019dadam}
\bibfield{author}{\bibinfo{person}{Parvin Nazari},
  \bibinfo{person}{Davoud~Ataee Tarzanagh}, {and} \bibinfo{person}{George
  Michailidis}.} \bibinfo{year}{2019}\natexlab{}.
\newblock \showarticletitle{Dadam: A consensus-based distributed adaptive
  gradient method for online optimization}.
\newblock \bibinfo{journal}{\emph{arXiv preprint arXiv:1901.09109}}
  (\bibinfo{year}{2019}).
\newblock


\bibitem[\protect\citeauthoryear{Rastegari, Ordonez, Redmon, and
  Farhadi}{Rastegari et~al\mbox{.}}{2016}]%
        {rastegari2016xnor}
\bibfield{author}{\bibinfo{person}{Mohammad Rastegari},
  \bibinfo{person}{Vicente Ordonez}, \bibinfo{person}{Joseph Redmon}, {and}
  \bibinfo{person}{Ali Farhadi}.} \bibinfo{year}{2016}\natexlab{}.
\newblock \showarticletitle{Xnor-net: Imagenet classification using binary
  convolutional neural networks}. In \bibinfo{booktitle}{\emph{European
  conference on computer vision}}. Springer, \bibinfo{pages}{525--542}.
\newblock


\bibitem[\protect\citeauthoryear{Reddi, Charles, Zaheer, Garrett, Rush,
  Kone{\v{c}}n{\`y}, Kumar, and McMahan}{Reddi et~al\mbox{.}}{2020}]%
        {reddi2020adaptive}
\bibfield{author}{\bibinfo{person}{Sashank Reddi}, \bibinfo{person}{Zachary
  Charles}, \bibinfo{person}{Manzil Zaheer}, \bibinfo{person}{Zachary Garrett},
  \bibinfo{person}{Keith Rush}, \bibinfo{person}{Jakub Kone{\v{c}}n{\`y}},
  \bibinfo{person}{Sanjiv Kumar}, {and} \bibinfo{person}{H~Brendan McMahan}.}
  \bibinfo{year}{2020}\natexlab{}.
\newblock \showarticletitle{Adaptive Federated Optimization}.
\newblock \bibinfo{journal}{\emph{arXiv preprint arXiv:2003.00295}}
  (\bibinfo{year}{2020}).
\newblock


\bibitem[\protect\citeauthoryear{Reddi, Kale, and Kumar}{Reddi
  et~al\mbox{.}}{2019}]%
        {reddi2019convergence}
\bibfield{author}{\bibinfo{person}{Sashank~J Reddi}, \bibinfo{person}{Satyen
  Kale}, {and} \bibinfo{person}{Sanjiv Kumar}.}
  \bibinfo{year}{2019}\natexlab{}.
\newblock \showarticletitle{On the convergence of adam and beyond}.
\newblock \bibinfo{journal}{\emph{arXiv preprint arXiv:1904.09237}}
  (\bibinfo{year}{2019}).
\newblock


\bibitem[\protect\citeauthoryear{Silver, Huang, Maddison, Guez, Sifre, Van
  Den~Driessche, Schrittwieser, Antonoglou, Panneershelvam, Lanctot,
  et~al\mbox{.}}{Silver et~al\mbox{.}}{2016}]%
        {silver2016mastering}
\bibfield{author}{\bibinfo{person}{David Silver}, \bibinfo{person}{Aja Huang},
  \bibinfo{person}{Chris~J Maddison}, \bibinfo{person}{Arthur Guez},
  \bibinfo{person}{Laurent Sifre}, \bibinfo{person}{George Van Den~Driessche},
  \bibinfo{person}{Julian Schrittwieser}, \bibinfo{person}{Ioannis Antonoglou},
  \bibinfo{person}{Veda Panneershelvam}, \bibinfo{person}{Marc Lanctot},
  {et~al\mbox{.}}} \bibinfo{year}{2016}\natexlab{}.
\newblock \showarticletitle{Mastering the game of Go with deep neural networks
  and tree search}.
\newblock \bibinfo{journal}{\emph{nature}} \bibinfo{volume}{529},
  \bibinfo{number}{7587} (\bibinfo{year}{2016}), \bibinfo{pages}{484}.
\newblock


\bibitem[\protect\citeauthoryear{Simonyan and Zisserman}{Simonyan and
  Zisserman}{2014}]%
        {simonyan2014very}
\bibfield{author}{\bibinfo{person}{Karen Simonyan} {and}
  \bibinfo{person}{Andrew Zisserman}.} \bibinfo{year}{2014}\natexlab{}.
\newblock \showarticletitle{Very deep convolutional networks for large-scale
  image recognition}.
\newblock \bibinfo{journal}{\emph{arXiv preprint arXiv:1409.1556}}
  (\bibinfo{year}{2014}).
\newblock


\bibitem[\protect\citeauthoryear{Smola and Narayanamurthy}{Smola and
  Narayanamurthy}{2010}]%
        {smola2010architecture}
\bibfield{author}{\bibinfo{person}{Alexander Smola} {and}
  \bibinfo{person}{Shravan Narayanamurthy}.} \bibinfo{year}{2010}\natexlab{}.
\newblock \showarticletitle{An architecture for parallel topic models}.
\newblock \bibinfo{journal}{\emph{Proceedings of the VLDB Endowment}}
  \bibinfo{volume}{3}, \bibinfo{number}{1-2} (\bibinfo{year}{2010}),
  \bibinfo{pages}{703--710}.
\newblock


\bibitem[\protect\citeauthoryear{Tang, Lian, Qiu, Yuan, Zhang, Zhang, and
  Liu}{Tang et~al\mbox{.}}{2019}]%
        {tang2019deepsqueeze}
\bibfield{author}{\bibinfo{person}{H Tang}, \bibinfo{person}{X Lian},
  \bibinfo{person}{S Qiu}, \bibinfo{person}{L Yuan}, \bibinfo{person}{C Zhang},
  \bibinfo{person}{T Zhang}, {and} \bibinfo{person}{J Liu}.}
  \bibinfo{year}{2019}\natexlab{}.
\newblock \showarticletitle{DeepSqueeze: Decentralized meets error-compensated
  compression}.
\newblock \bibinfo{journal}{\emph{arXiv preprint arXiv:1907.07346}}
  (\bibinfo{year}{2019}).
\newblock


\bibitem[\protect\citeauthoryear{Ward, Wu, and Bottou}{Ward
  et~al\mbox{.}}{2019}]%
        {ward2018adagrad}
\bibfield{author}{\bibinfo{person}{Rachel Ward}, \bibinfo{person}{Xiaoxia Wu},
  {and} \bibinfo{person}{Leon Bottou}.} \bibinfo{year}{2019}\natexlab{}.
\newblock \showarticletitle{AdaGrad stepsizes: Sharp convergence over nonconvex
  landscapes}. In \bibinfo{booktitle}{\emph{International Conference on Machine
  Learning}}. PMLR, \bibinfo{pages}{6677--6686}.
\newblock


\bibitem[\protect\citeauthoryear{Wen, Xu, Yan, Wu, Wang, Chen, and Li}{Wen
  et~al\mbox{.}}{2017}]%
        {wen2017terngrad}
\bibfield{author}{\bibinfo{person}{Wei Wen}, \bibinfo{person}{Cong Xu},
  \bibinfo{person}{Feng Yan}, \bibinfo{person}{Chunpeng Wu},
  \bibinfo{person}{Yandan Wang}, \bibinfo{person}{Yiran Chen}, {and}
  \bibinfo{person}{Hai Li}.} \bibinfo{year}{2017}\natexlab{}.
\newblock \showarticletitle{Terngrad: Ternary gradients to reduce communication
  in distributed deep learning}. In \bibinfo{booktitle}{\emph{Advances in
  neural information processing systems}}. \bibinfo{pages}{1509--1519}.
\newblock


\bibitem[\protect\citeauthoryear{Wu, Chen, Fan, Zhang, Hou, Liu, and Zhang}{Wu
  et~al\mbox{.}}{2019}]%
        {wu2019tencent}
\bibfield{author}{\bibinfo{person}{Baoyuan Wu}, \bibinfo{person}{Weidong Chen},
  \bibinfo{person}{Yanbo Fan}, \bibinfo{person}{Yong Zhang},
  \bibinfo{person}{Jinlong Hou}, \bibinfo{person}{Jie Liu}, {and}
  \bibinfo{person}{Tong Zhang}.} \bibinfo{year}{2019}\natexlab{}.
\newblock \showarticletitle{Tencent ml-images: A large-scale multi-label image
  database for visual representation learning}.
\newblock \bibinfo{journal}{\emph{IEEE Access}}  \bibinfo{volume}{7}
  (\bibinfo{year}{2019}), \bibinfo{pages}{172683--172693}.
\newblock


\bibitem[\protect\citeauthoryear{Wu, Li, Chen, and Shi}{Wu
  et~al\mbox{.}}{2018}]%
        {wu2018training}
\bibfield{author}{\bibinfo{person}{Shuang Wu}, \bibinfo{person}{Guoqi Li},
  \bibinfo{person}{Feng Chen}, {and} \bibinfo{person}{Luping Shi}.}
  \bibinfo{year}{2018}\natexlab{}.
\newblock \showarticletitle{Training and inference with integers in deep neural
  networks}.
\newblock \bibinfo{journal}{\emph{arXiv preprint arXiv:1802.04680}}
  (\bibinfo{year}{2018}).
\newblock


\bibitem[\protect\citeauthoryear{Xing, Ho, Dai, Kim, Wei, Lee, Zheng, Xie,
  Kumar, and Yu}{Xing et~al\mbox{.}}{2015}]%
        {xing2015petuum}
\bibfield{author}{\bibinfo{person}{Eric~P Xing}, \bibinfo{person}{Qirong Ho},
  \bibinfo{person}{Wei Dai}, \bibinfo{person}{Jin~Kyu Kim},
  \bibinfo{person}{Jinliang Wei}, \bibinfo{person}{Seunghak Lee},
  \bibinfo{person}{Xun Zheng}, \bibinfo{person}{Pengtao Xie},
  \bibinfo{person}{Abhimanu Kumar}, {and} \bibinfo{person}{Yaoliang Yu}.}
  \bibinfo{year}{2015}\natexlab{}.
\newblock \showarticletitle{Petuum: A new platform for distributed machine
  learning on big data}.
\newblock \bibinfo{journal}{\emph{IEEE Transactions on Big Data}}
  \bibinfo{volume}{1}, \bibinfo{number}{2} (\bibinfo{year}{2015}),
  \bibinfo{pages}{49--67}.
\newblock


\bibitem[\protect\citeauthoryear{Yang, Liu, Chen, and Tong}{Yang
  et~al\mbox{.}}{2019}]%
        {yang2019federated}
\bibfield{author}{\bibinfo{person}{Qiang Yang}, \bibinfo{person}{Yang Liu},
  \bibinfo{person}{Tianjian Chen}, {and} \bibinfo{person}{Yongxin Tong}.}
  \bibinfo{year}{2019}\natexlab{}.
\newblock \showarticletitle{Federated machine learning: Concept and
  applications}.
\newblock \bibinfo{journal}{\emph{ACM Transactions on Intelligent Systems and
  Technology (TIST)}} \bibinfo{volume}{10}, \bibinfo{number}{2}
  (\bibinfo{year}{2019}), \bibinfo{pages}{1--19}.
\newblock


\bibitem[\protect\citeauthoryear{Zheng, Huang, and Kwok}{Zheng
  et~al\mbox{.}}{2019}]%
        {zheng2019communication}
\bibfield{author}{\bibinfo{person}{Shuai Zheng}, \bibinfo{person}{Ziyue Huang},
  {and} \bibinfo{person}{James Kwok}.} \bibinfo{year}{2019}\natexlab{}.
\newblock \showarticletitle{Communication-efficient distributed blockwise
  momentum SGD with error-feedback}. In \bibinfo{booktitle}{\emph{Advances in
  Neural Information Processing Systems}}. \bibinfo{pages}{11450--11460}.
\newblock


\bibitem[\protect\citeauthoryear{Zhou, Ni, Zhou, Wen, Wu, and Zou}{Zhou
  et~al\mbox{.}}{2016}]%
        {DBLP:journals/corr/ZhouNZWWZ16}
\bibfield{author}{\bibinfo{person}{Shuchang Zhou}, \bibinfo{person}{Zekun Ni},
  \bibinfo{person}{Xinyu Zhou}, \bibinfo{person}{He Wen},
  \bibinfo{person}{Yuxin Wu}, {and} \bibinfo{person}{Yuheng Zou}.}
  \bibinfo{year}{2016}\natexlab{}.
\newblock \showarticletitle{DoReFa-Net: Training Low Bitwidth Convolutional
  Neural Networks with Low Bitwidth Gradients}.
\newblock \bibinfo{journal}{\emph{CoRR}}  \bibinfo{volume}{abs/1606.06160}
  (\bibinfo{year}{2016}).
\newblock
\showeprint[arxiv]{1606.06160}
\urldef\tempurl%
\url{http://arxiv.org/abs/1606.06160}
\showURL{%
\tempurl}


\bibitem[\protect\citeauthoryear{Zhou, Zhang, Lu, Wang, Zhang, and Yu}{Zhou
  et~al\mbox{.}}{2018}]%
        {zhou2018adashift}
\bibfield{author}{\bibinfo{person}{Zhiming Zhou}, \bibinfo{person}{Qingru
  Zhang}, \bibinfo{person}{Guansong Lu}, \bibinfo{person}{Hongwei Wang},
  \bibinfo{person}{Weinan Zhang}, {and} \bibinfo{person}{Yong Yu}.}
  \bibinfo{year}{2018}\natexlab{}.
\newblock \showarticletitle{AdaShift: Decorrelation and Convergence of Adaptive
  Learning Rate Methods}. In \bibinfo{booktitle}{\emph{International Conference
  on Learning Representations}}.
\newblock


\bibitem[\protect\citeauthoryear{Zou, Shen, Jie, Sun, and Liu}{Zou
  et~al\mbox{.}}{2018}]%
        {zou2018weighted}
\bibfield{author}{\bibinfo{person}{Fangyu Zou}, \bibinfo{person}{Li Shen},
  \bibinfo{person}{Zequn Jie}, \bibinfo{person}{Ju Sun}, {and}
  \bibinfo{person}{Wei Liu}.} \bibinfo{year}{2018}\natexlab{}.
\newblock \showarticletitle{Weighted AdaGrad with unified momentum}.
\newblock \bibinfo{journal}{\emph{arXiv preprint arXiv:1808.03408}}
  (\bibinfo{year}{2018}).
\newblock


\bibitem[\protect\citeauthoryear{Zou, Shen, Jie, Zhang, and Liu}{Zou
  et~al\mbox{.}}{2019}]%
        {zou2019sufficient}
\bibfield{author}{\bibinfo{person}{Fangyu Zou}, \bibinfo{person}{Li Shen},
  \bibinfo{person}{Zequn Jie}, \bibinfo{person}{Weizhong Zhang}, {and}
  \bibinfo{person}{Wei Liu}.} \bibinfo{year}{2019}\natexlab{}.
\newblock \showarticletitle{A sufficient condition for convergences of adam and
  rmsprop}. In \bibinfo{booktitle}{\emph{Proceedings of the IEEE Conference on
  Computer Vision and Pattern Recognition}}. \bibinfo{pages}{11127--11135}.
\newblock


\end{thebibliography}

\end{document}